\documentclass[pdflatex,sn-mathphys-ay]{sn-jnl}

\usepackage{graphicx}
\usepackage{multirow}
\usepackage{amsmath,amssymb,amsfonts}
\usepackage{amsthm}
\usepackage{mathrsfs}
\usepackage[title]{appendix}
\usepackage{xcolor}
\usepackage{textcomp}
\usepackage{manyfoot}
\usepackage{booktabs}
\usepackage{algorithm}
\usepackage{algorithmicx}
\usepackage{algpseudocode}
\usepackage{listings}
\usepackage{setspace}
\usepackage{makecell}

\theoremstyle{thmstyleone}
\newtheorem{theorem}{Theorem}

\theoremstyle{thmstyletwo}

\theoremstyle{thmstylethree}
\newtheorem{definition}{Definition}

\begin{document}

\title[Article Title]{Neural Approximation and Its Applications}

\author[1]{\fnm{Wei-Hao} \sur{Wu}}\email{weihaowu99@163.com}

\author*[1]{\fnm{Ting-Zhu} \sur{Huang}}\email{tingzhuhuang@126.com}

\author*[1]{\fnm{Xi-Le} \sur{Zhao}}\email{xlzhao122003@163.com}

\author[2]{\fnm{Yisi} \sur{Luo}}\email{yisiluo1221@foxmail.com}

\author[2,3]{\fnm{Deyu} \sur{Meng}}\email{dymeng@mail.xjtu.edu.cn}

\affil[1]{\orgdiv{School of Mathematical Sciences}, \orgname{University of Electronic Science and Technology of China}, \city{Chengdu}, \postcode{611731}, \state{Sichuan}, \country{China}}

\affil[2]{\orgdiv{School of Mathematical and Statistics}, \orgname{Xi'an Jiaotong University}, \city{Xi'an}, \postcode{710049}, \state{Shaanxi}, \country{China}}

\affil[3]{\orgdiv{Key Laboratory of Computational Intelligence and Chinese Information Processing of Ministry of Education}, \orgname{Shanxi University}, \city{Taiyuan}, \postcode{030006}, \state{Shanxi}, \country{China}}

\abstract{Multivariate function approximation is a fundamental problem in machine learning. Classic multivariate function approximations rely on hand-crafted basis functions (e.g., polynomial basis and Fourier basis), which limits their approximation ability and data adaptation ability, resulting in unsatisfactory performance. To address these challenges, we introduce the neural basis function by leveraging an untrained neural network as the basis function. Equipped with the proposed neural basis function, we suggest the neural approximation (NeuApprox) paradigm for multivariate function approximation. Specifically, the underlying multivariate function behind the multi-dimensional data is decomposed into a sum of block terms. The clear physically-interpreted block term is the product of expressive neural basis functions and their corresponding learnable coefficients, which allows us to faithfully capture distinct components of the underlying data and also flexibly adapt to new data by readily fine-tuning the neural basis functions. Attributed to the elaborately designed block terms, the suggested NeuApprox enjoys strong approximation ability and flexible data adaptation ability over the hand-crafted basis function-based methods. We also theoretically prove that NeuApprox can approximate any multivariate continuous function to arbitrary accuracy. Extensive experiments on diverse multi-dimensional datasets (including multispectral images, light field data, videos, traffic data, and point cloud data) demonstrate the promising performance of NeuApprox in terms of both approximation capability and adaptability.}


\maketitle

\section{Introduction}
\label{sec:introduction}

Multivariate function approximation is a fundamental problem that arises in a variety of fields, including signal processing \cite{Zhang2023,fa_sp2,fa_sp3,Kim2011}, machine learning \cite{Zhao2024,fa_ml2,fa_ml3}, and data analysis \cite{Gu2024,fa_da2,fa_da3}. Traditional methods usually leverage the linear combination of hand-crafted basis functions, such as the Fourier basis \cite{TSP_CP}, Chebyshev basis \cite{SIAM_chebfun}, and Hermite basis \cite{Hermite}, to approximate the target functions behind the multi-dimensional data. These methods have demonstrated strong performance and are supported by solid theoretical foundations \cite{theory}. For instance, a recent study \cite{fsa_tsp} introduced a method combining Canonical Polyadic Decomposition (CPD) with finite Fourier series to approximate multivariate functions. This approach offers advantages such as improved computational efficiency and enhanced identifiability of model parameters.

\begin{figure}[h]
	\scriptsize
	\setlength{\tabcolsep}{0.9pt}
	\begin{center}
		\begin{tabular}{c}
			\includegraphics[width=0.9\textwidth]{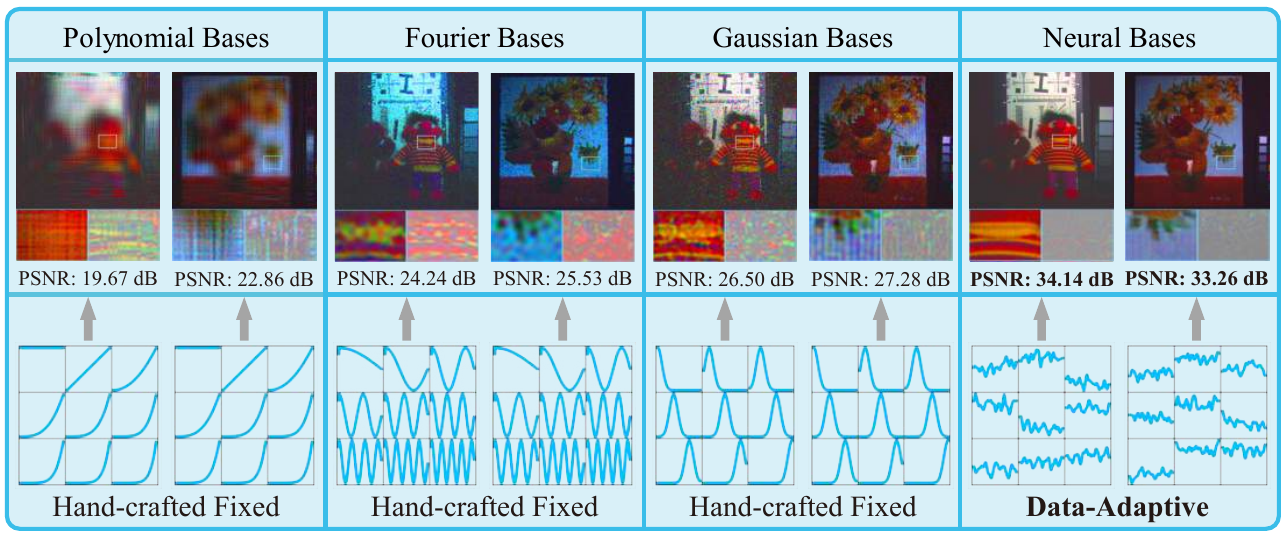}
		\end{tabular}
	\end{center}
	\caption{Comparison of multivariate function approximation behind multi-dimensional data (i.e.,  MSI \textit{Painting} and MSI \textit{Toy} with SR = 0.1) based on different basis function. For different data, the hand-crafted basis functions (e.g., polynomial, Fourier, and Gaussian basis functions) are predefined, while the proposed neural basis functions are data-adaptive. The enlarged patch and corresponding residuals are displayed under the recovered results.\label{fig_motivation}}
	
\end{figure}

While hand-crafted basis function-based methods have proven effective across a wide range of problems \cite{8673620,1039207,10920477}, they are limited by predefined hand-crafted basis functions. This constraint hampers their approximation ability, particularly for complex tasks or target functions that do not align well with the chosen basis functions. These methods may also prove inadequate for capturing complex structures in multi-dimensional datasets \cite{5462953,9380520}, such as multispectral images (MSIs) and point cloud data, where an undesired trade-off between the model complexity and the approximation accuracy is required. For example, Fig. \ref{fig_motivation} displays the recovered results on different MSIs by functional approximation with different basis functions. We can observe that hand-crafted basis function-based functional approximations provide limited performance, because of the limited representation ability of predefined hand-crafted basis functions. As the complexity of approximation tasks involving multi-dimensional data increases, the limitations of these hand-crafted basis function-based methods become more apparent, highlighting the need for more flexible and robust multivariate function approximation techniques.

To overcome the challenges of basis function selection and the limitations of fixed basis sets, researchers have increasingly turned to neural network-based approaches for multivariate function approximation. Neural networks are capable of learning complex functions directly from observed data and have achieved remarkable success across various domains. The universal approximation theorem \cite{ua} establishes that feedforward neural networks with sufficient hidden units can approximate any continuous function in a compact set, providing a solid theoretical foundation for their use in function approximation without the need for predefined basis functions.

While directly approximating a complex multivariate function using neural networks requires a large number of learnable parameters, it is more tractable to first decompose the multivariate function into a series of simpler univariate functions and then use neural networks to approximate these univariate functions~\cite{10.5555/3305381.3305546,liang2022coordx}.
Hence, recent advancements have led to methods that combine neural networks with tensor decomposition, whose factors are treated as univariate functions and coefficients. For example, a recent study \cite{lrtfr_tpami} proposed using multilayer perceptrons (MLPs) in conjunction with tensor Tucker decomposition, which regards each factor function as vector-value function and treats the core tensor as learnable coefficients, achieving improved performance compared to standard MLPs and demonstrating the potential of integrating neural networks with tensor-based approaches. Specifically, the approach \cite{lrtfr_tpami} employs a single neural network to approximate the each entire univariate factor function along each dimension and treats a learnable core tensor as the coefficients. However, these designs typically regard multi-dimensional data as a single block term, thereby neglecting the inherent structural complexities and the interplay between different principal components of the multi-dimensional data. Song et al. \cite{Song2023} decomposed the third-order off-meshgrid data into the sum of several terms, where each term is approximated by the product of three univariate function parameterized by MLPs. However, the basis function used in each block term is scale-value function, which limits the ability of neural networks to fully exploit their representation power and capture high-order interactions within the data. These oversimplifications ultimately result in missed opportunities for performance enhancement that could be achieved through more refined decomposition methods and localized neural approximations.

\begin{figure}[t]
	\scriptsize
	\setlength{\tabcolsep}{2pt}
	\begin{center}
		\begin{tabular}{ccc}
			(a) & (b) & (c)\\[2pt]
			\includegraphics[width=0.3\textwidth]{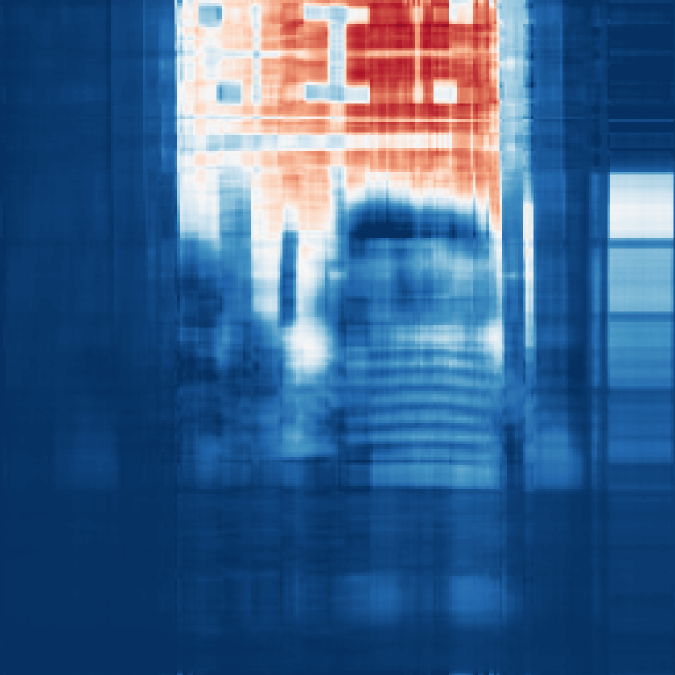}&
			\includegraphics[width=0.3\textwidth]{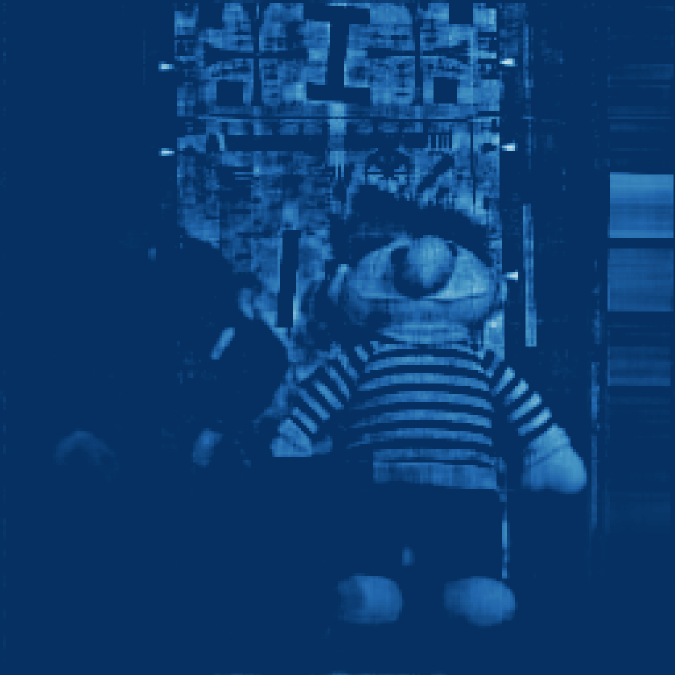}&
			\includegraphics[width=0.3\textwidth]{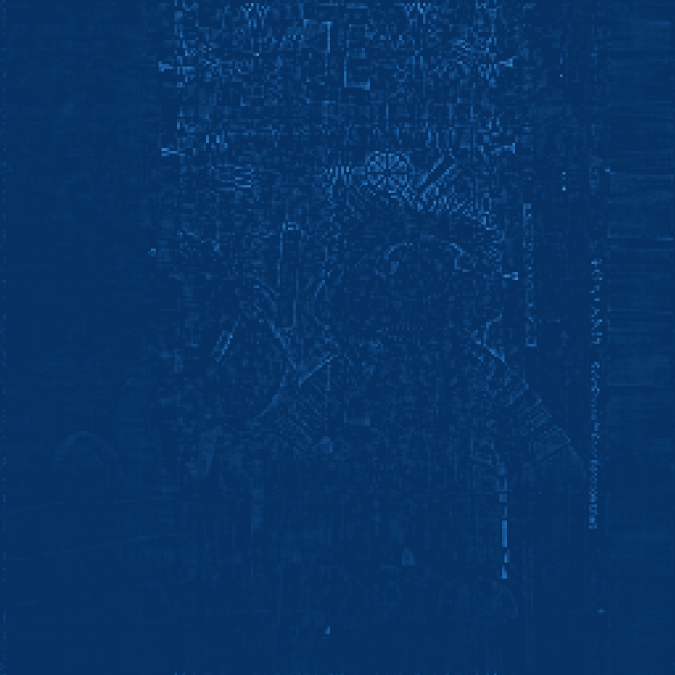}\\[2pt]
			\includegraphics[width=0.3\textwidth]{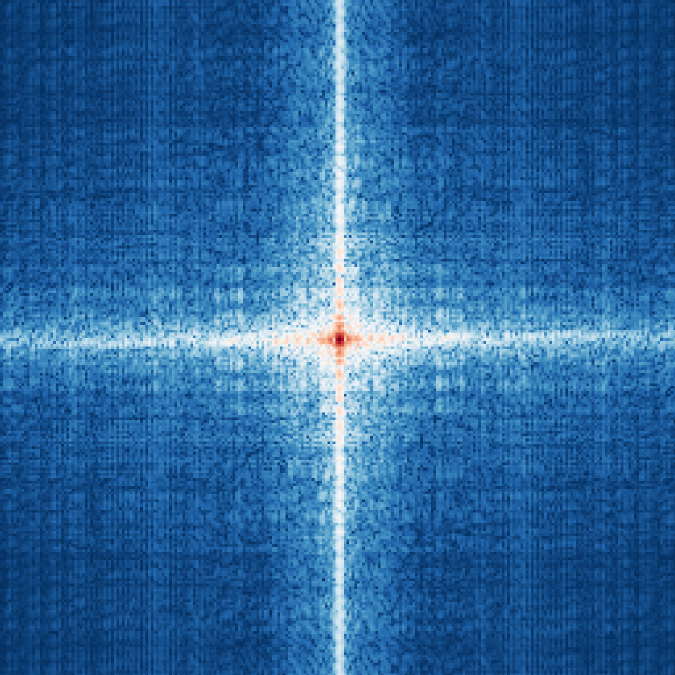}&
			\includegraphics[width=0.3\textwidth]{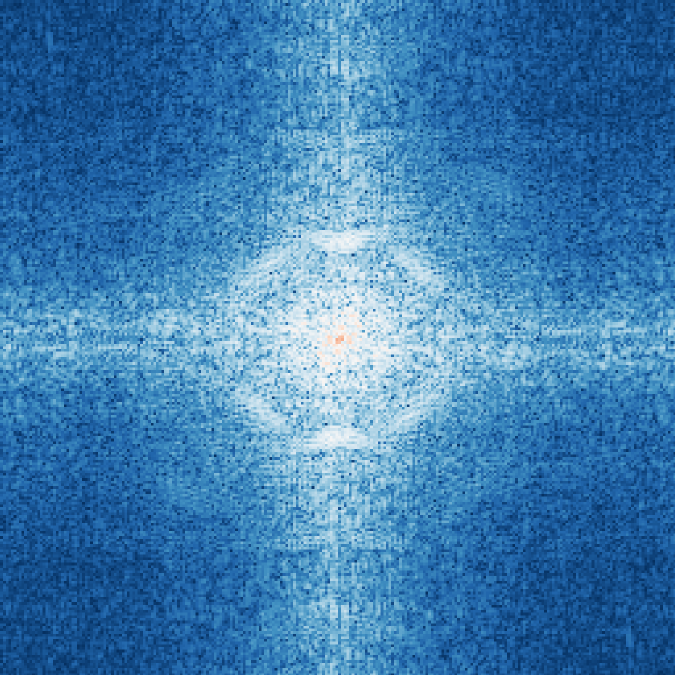}&
			\includegraphics[width=0.3\textwidth]{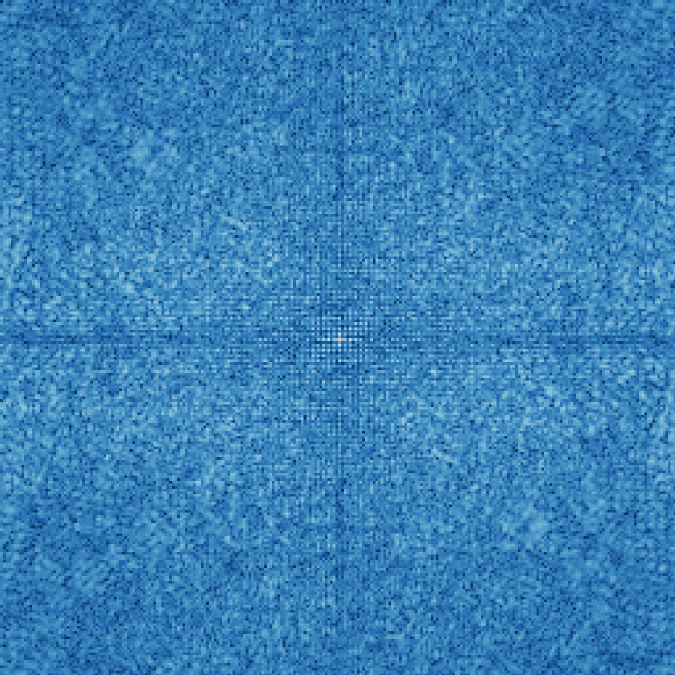}\\[2pt]
		\end{tabular}
		\begin{tabular}{c}
			\hspace{0.05cm}\includegraphics[width=0.9\textwidth]{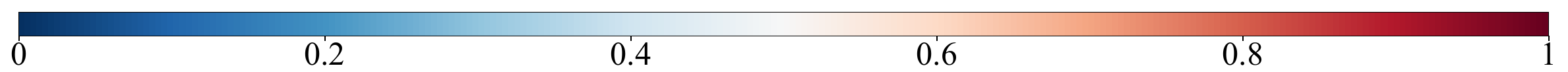}\\
		\end{tabular}
	\end{center}
	\caption{Different block terms of the proposed NeuApprox and their corresponding Fourier spectrum images (bottom) on MSI \textit{Toy} with SR = 0.1. (a)-(c) first, second, and third block terms of the proposed NeuApprox with corresponding Fourier spectrum images.\label{motivation_multiblock}}
\end{figure}

To address aforementioned issues, we suggest the neural basis function, which enjoys powerful approximation ability and flexible data adaptation ability compared with hand-crafted basis functions. Equipped with neural basis functions, we propose a novel approach to approximate the multivariate function behind the multi-dimensional data, termed neural approximation (NeuApprox), which can inherently capture the structural complexities and interplay of different principal components (e.g., different frequency information) by leveraging multiple block term representation (see Fig. \ref{motivation_multiblock}). Specifically, we factorize the target multivariate function into the sum of multiple block terms.  Each block term is the product of learnable coefficients and data-adaptive neural basis functions, which leverages untrained deep neural networks to parameterize the basis function. Notably, the proposed NeuApprox uses different neural bases to approximate different block terms, enabling better utilization of the flexibility and expressive power of neural networks. As a result, NeuApprox exhibits strong approximation ability and data adaptation capacity for multi-dimensional data representation. As shown in Fig. \ref{fig_motivation}, we display the recovered results for different multi-dimensional datasets using NeuApprox with various basis functions and sampling rates (SRs) of 0.1. The results demonstrate that NeuApprox consistently and significantly outperforms methods using hand-crafted basis functions. The advantage of NeuApprox arises from three factors: first, neural basis functions offer more powerful representation capabilities than hand-crafted basis functions, leading to superior recovery results; second, the proposed NeuApprox leverages adaptive neural basis functions that dynamically adjust to diverse data characteristics, enabling superior data-adaptation capabilities compared with traditional hand-crafted basis functions-based methods; third, attributed to the multiple block term representation, the proposed NeuApprox can represent the structural diversity and interactions present among various principal components of multi-dimensional data.

The contributions of this paper are summarized as follows:

\begin{itemize}
	
	\item \textbf{Neural Basis Functions}: Traditional multivariate function approximation relies on hand-crafted basis functions, which inherently limits both its representation ability and data adaptation ability. To address this challenge, we suggest a novel neural basis function, which leverage untrained deep neural networks as adaptive basis functions, thereby enhancing representation ability and adaptation ability. Equipped with the proposed neural basis function, we propose a novel neural approximation (NeuApprox) for multivariate function approximation. Specifically, NeuApprox factorizes the target function into a sum of block terms, each approximated by neural basis functions and corresponding learnable coefficients. Different block terms can faithfully capture different components of the target function. Compared to classical methods based on hand-crafted basis functions, NeuApprox not only offers superior representation ability but also enjoys appealing data adaptation ability.
	
	\item \textbf{Experimental Evaluation}: To validate the effectiveness of the proposed method, we conduct extensive experiments on various data recovery tasks, including multi-dimensional image inpainting, traffic data completion, and point cloud data completion (off-meshgrid). The results demonstrate the significant advantages of NeuApprox over state-of-the-art methods in terms of performance and adaptability.
\end{itemize}

The rest of the paper is organized as follows. Section \ref{sec:related_work} reviews related work on basis function methods and neural network approaches for function approximation. Section \ref{sec:method} presents the proposed NeuApprox model in detail. Section \ref{sec:experiments} presents extensive experimental results. Section \ref{sec:diss} gives the comprehensive discussions about the proposed NeuApprox and Section \ref{sec:conclusion} concludes this work.

\section{Related Work}
\label{sec:related_work}
\subsection{Basis Function Approaches for Function Approximation}  

Traditional function approximation methods typically rely on hand-crafted basis functions, which represent complex and unknown functions as linear combinations of simpler predefined bases. Common bases include polynomials \cite{SIAM_chebfun}, trigonometric functions \cite{Blackman2023MultidimensionalIT}, and wavelets \cite{MUZHOU20112173}, all of which have found successful applications in fields such as signal processing, data compression, and physics modeling. Among these methods, Fourier series expansion is particularly notable for its ability to approximate periodic functions by decomposing them into a sum of sinusoidal components. This approximation offers rigorous convergence properties and mathematical elegance \cite{9851557}. Recent advances in tensor decompositions, such as Canonical Polyadic Decomposition (CPD) and Tucker decomposition, have improved function approximation frameworks by introducing low-rank approximations that capture essential low-dimensional structures in multi-dimensional data \cite{fsa_tsp,td_icml}. These tensor decomposition-based approaches allow complex functions to be approximated using combinations of rank-limited components, significantly reducing computational complexity while retaining expressive power. However, despite these advantages, fixed basis function methods are inherently limited in flexibility. Specifically, the choice of predefined basis functions restricts their adaptability to complex data distributions, which poses challenges when applying them to real-world data exhibiting strong variability or intricate local features.

\subsection{Neural Network-Assisted Function Approximation}  
With the advent of deep learning, neural networks (such as MLPs and convolutional networks) have emerged as powerful tools for function approximation due to their universal approximation capabilities \cite{ua}. Unlike traditional basis function methods, neural networks do not require preselected basis sets. Instead, they learn basis functions directly from data, allowing the bases to adapt to the unique characteristics of each dataset. In addition, neural networks are capable of capturing non-linear patterns present in real-world data. A well-known example along this line is implicit neural representations (INRs), which leverage MLPs to construct continuous and differentiable functions learned from data, achieving high fidelity in applications ranging from image synthesis to 3D shape reconstruction \cite{inr_nips,wire_cvpr,nerf_eccv}. These neural representations have demonstrated remarkable flexibility in capturing complex data structures and fine-grained details, making them highly effective for applications that require precise and adaptive approximations.

In recent years, efforts have been made to integrate neural networks with tensor decomposition methods to combine the adaptability of neural networks with the structural advantages of tensor decompositions. For instance, the CPD-based neural function representations \cite{cp_nc} leverage low-rank tensor structures to balance representational power with computational efficiency. Most related to ours is Tucker decomposition-based neural function representations \cite{lrtfr_tpami}, which leverage MLPs and learnable tensors to parameterize the factor function and coefficients. However, it only utilizes a single block term to approximate the original data, resulting in different components of the data being entangled within the same block. This coupling makes it difficult for a single block term to effectively model the complex components present in the data, thereby leading to limited approximation performance. Thus, they often lack the fine-grained details excavation ability for multi-dimensional data.\par
The proposed NeuApprox decomposes the multivariate function into different block terms, and each block term is parameterized by a coefficient tensor and learnable neural bases. Notably, each block term spontaneously captures and models distinct components of the data, allowing NeuApprox to more effectively represent the complex structures present in the original function. Meanwhile, by leveraging untrained neural network to represent neural basis functions, our NeuApprox capitalizes on the modularity of traditional basis functions while employing neural networks for flexible and data-driven adaptation, setting a better paradigm for function approximation of complex, multi-dimensional data.

\section{Function Approximation Using Neural Basis Functions}
\label{sec:method}
\begin{figure*}[t] 
	\centering 
	\includegraphics[width=\textwidth]{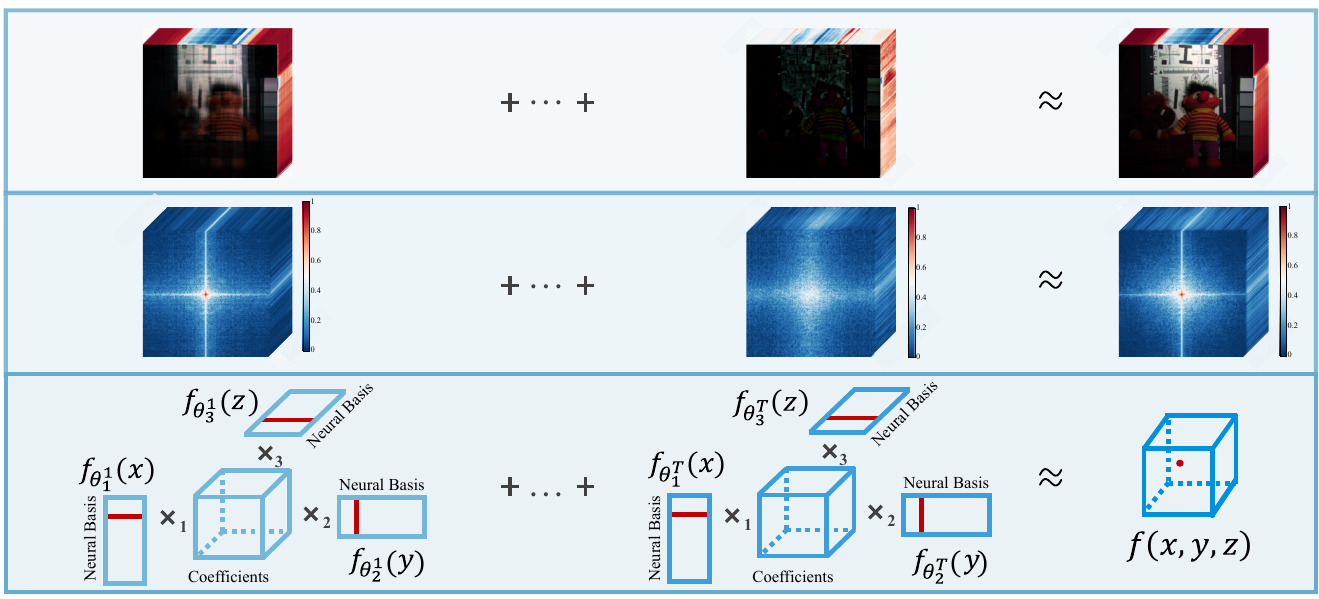}\\
	\caption{Illustration of the proposed NeuApprox for a three-variable function $f(x,y,z)$. Leveraging multiple block terms to approximate the target function behind the multi-dimensional data. Each block term, which can inherently capture different components of the target function, is formed by the product of a neural basis function and its corresponding coefficient. From top to bottom: multi-dimensional data, cross-section of the three-variable functions behind the multi-dimensional data, and the flowchart of our proposed NeuApprox. Here, $(x,y,z)$ is the coordinate; $f_{\theta_{1}^{j}}(\cdot)$, $f_{\theta_{2}^{j}}(\cdot)$,  and $f_{\theta_{3}^{j}}(\cdot)$ are the suggested neural basis functions for block term index $j=1,2,\cdots,T$.} 
	\label{fig::model} 
\end{figure*}

\subsection{Notations and Definitions}
\label{subsec:notation}
In this work, $x,\;\boldsymbol{x},\boldsymbol{X},\mathcal{X}$ respectively denote scalars, vectors, matrices, and tensors. We use $x_m$ as index notations to access the elements of vectors. 
We use $\|\cdot\|_{F}$ and $\|\cdot\|_{\ell_1}$ to represent the Frobenius norm and the $\ell_1$-norm.
For a tensor ${\mathcal X}\in {\mathbb R}^{D_1\times\cdots\times D_M}$, we define the mode-$i$ unfolding operator as ${\mathtt {unfold}}_i(\cdot) : {\mathbb R}^{D_1\times\cdots\times D_M} \to {\mathbb R}^{D_i\times\prod_{j\neq i} D_j}$, with ${\mathtt {fold}}_i(\cdot)$ as its inverse operator. The mode-$i$ product between tensor ${\mathcal X}$ and matrix ${\boldsymbol A}$ is defined as ${\mathcal X}\times_i{\boldsymbol A}\triangleq {\mathtt {fold}_i}({\boldsymbol A}{\mathtt{unfold}_i({\mathcal X})})$. For a vector-valued function $f$ and $g$, we use $f\circ g$ to denote the composition of $f$ and $g$.

\subsection{Multivariate Function Approximation}
\label{subsec:basis function}

Multivariate function approximation seeks to represent a complex target function using a combination of simple basis functions and corresponding coefficients. For any continuous multivariate function $f: [0,1]^N \to \mathbb{R}$, it is possible to construct an infinite series expansion involving univariate basis functions $\{\phi_k(x): [0,1] \to \mathbb{R}\}_{k=1}^{\infty}$ and associated coefficients $\{w_k \in \mathbb{R}\}_{k=1}^{\infty}$.

This idea naturally extends to multivariate settings, where $f$ can be approximated by a tensor product expansion:
\begin{equation}
	\label{infty}
	f(\boldsymbol{x}) = \sum_{k_1=1}^{\infty} \cdots \sum_{k_N=1}^{\infty} w_{\boldsymbol{k}} \prod_{n=1}^{N} \phi_{k_n}^{(n)}(x_n),
\end{equation}
where $\boldsymbol{x} = (x_1, \dots, x_N) \in [0,1]^N$, $\boldsymbol{k} = (k_1, \dots, k_N) \in \mathbb{N}^N$ is a multi-index, and $w_{\boldsymbol{k}}$ is the set of coefficients.

In practical applications, this infinite series is truncated to finite summation limits $K_1, \dots, K_N$:
\begin{equation}
	\label{eq:basis_function}
	f(\boldsymbol{x}) \approx \sum_{k_1=1}^{K_1} \cdots \sum_{k_N=1}^{K_N} w_{\boldsymbol{k}} \prod_{n=1}^{N} \phi_{k_n}^{(n)}(x_n).
\end{equation}

The approximation error depends on how rapidly the coefficients $w_{\boldsymbol{k}}$ decay. Smoother functions tend to yield faster-decaying coefficients, which in turn leads to smaller truncation errors. This relationship is captured by taking the Fourier basis function as an example.

\begin{theorem}[cf. \cite{bf_sn}]
	Let $f: [0,1]^N \to \mathbb{R}$ be a target function expanded in a Fourier basis as in \eqref{infty}, and let $p \in \mathbb{N}$. Suppose that all mixed partial derivatives
	$$\frac{\partial \beta_1}{\partial \boldsymbol{x}_1}\cdot\frac{\partial \beta_2}{\partial \boldsymbol{x}_2}  \cdots \frac{\partial \beta_N}{\partial\boldsymbol{x}_N}f(\boldsymbol{x})$$
	exist and are absolutely integrable for all multi-indices $(\beta_1,\beta_2, \dots, \beta_N)$ such that $\sum_{n=1}^{N} \beta_n \leq p$. Then,
	\begin{equation}
		\label{bf_th}
		\lim_{\|\boldsymbol{k}\|_2 \to \infty} (1 + \|\boldsymbol{k}\|_p^2) w_{\boldsymbol{k}} = 0.
	\end{equation}
\end{theorem}
This result implies that the coefficients $w_{\boldsymbol{k}}$ decay faster than $1 / (1 + \|\boldsymbol{k}\|_p^2)$, demonstrating the effectiveness of truncated Fourier basis function series in approximating sufficiently smooth multivariate functions.

However, multivariate function approximations based on hand-crafted basis function (e.g., Fourier basis function) suffer from limited approximation ability and inflexible data-adaptation ability, which usually causes unsatisfactory results in multi-dimensional data recovery tasks.

\subsection{Neural Basis Functions}
\label{subsec:implicit neural representation}

Recently, INR~\cite{inr_nips,wire_cvpr,nerf_eccv}, which leverages deep neural networks to learn mappings from input coordinates to corresponding output values, has attracted increasing attention. Fundamentally, an INR models the underlying signal, whether an image, shape, or physical field, as a continuous function defined over coordinate space. By utilizing the expressive power of neural networks, INRs provide flexible and high-capacity representations that adapt well to diverse types of data.

To enhance both approximation ability and data-adaptation ability, we propose neural basis functions, which adopt INRs as learnable data-driven basis functions. Unlike traditional basis expansions that rely on hand-crafted basis functions (e.g., Fourier basis and polynomial basis), neural basis functions allow the basis itself to be optimized in accordance with the underlying structure of the target function. Each basis function is represented by an MLP, which holds strong approximation ability due to the composition of affine transformations and nonlinear activation functions.

A typical neural basis function $f_\theta : [0,1] \to \mathbb{R}^{N}$ is defined by the following architecture:
\begin{equation}
	\label{inr}
	f_\theta(x) = \boldsymbol{W}_L \big(\cdots \sigma(\boldsymbol{W}_0 x + \boldsymbol{b}_0) \cdots \big) + \boldsymbol{b}_L,
\end{equation}
where $\theta := \{ \boldsymbol{W}_l, \boldsymbol{b}_l \}_{l=0}^{L}$ denotes the set of learnable parameters, $\boldsymbol{W}_l \in \mathbb{R}^{h_{l+1} \times h_{l}}$, and $\boldsymbol{b}_l \in \mathbb{R}^{h_{l+1}}$. Specifically, $h_0 = 1$ and $h_{L+1} = N$. $\sigma(\cdot)$ is the activation function. While ReLU is common in standard neural networks, here we consider periodic activations such as $\sin(\cdot)$ to better capture fine-grained and high-frequency variations, as suggested in~\cite{inr_nips}. In this framework, the input $x$ is treated as a coordinate, and the output $f_\theta(x)$ provides the value of the function at that location. When $N = 1$, the proposed neural basis function is a scalar-valued function. When $N > 1$, the proposed neural basis function is a vector-valued function.

The theoretical foundation of neural basis functions lies in the universal approximation property of neural networks, as formalized in the following theorem:

\begin{theorem}[Universal Approximation Theorem~\cite{inr_th}]
	\label{inr_th}
	Let $\sigma: \mathbb{R} \to \mathbb{R}$ be any non-affine continuous function that is differentiable at least at one point with a nonzero derivative. Then, for any continuous function $g: [0,1] \to \mathbb{R}$ and any $\epsilon > 0$, there exists a neural network $f_\theta(\cdot)$ with activation $\sigma(\cdot)$ such that
	\begin{equation}
		\sup_{x \in [0,1]} |f_\theta(x) - g(x)| < \epsilon.
	\end{equation}
\end{theorem}

This result establishes the universal approximation capability of neural networks, even under broad activation function assumptions, thereby validating their use as basis functions in function approximation tasks.

\subsection{Neural Approximation}
\label{subsec:block term decomposition}

Empowered by the proposed neural basis functions, we propose the Neural Approximation (NeuApprox) framework for multivariate function approximation. NeuApprox employs neural basis functions to characterize the underlying function space and leverages learnable coefficients to adaptively represent the target function. Unlike classical expansions that rely on fixed and hand-crafted functions, NeuApprox simultaneously optimizes both the neural basis functions and their associated coefficients during training. This joint optimization enables the basis system to dynamically adapt to the intrinsic structures and complexities of data, significantly enhancing approximation ability and data adaptation ability over traditional methods.

Specifically, our NeuApprox framework is formally defined as follows:

\begin{definition}{\bf (Neural Approximation)}
	For any multivariate continuous function $\psi: [0,1]^n \to \mathbb{R}$, the NeuApprox approximates $\psi$ using a finite sum of products of univariate neural basis functions:
	
	\begin{equation}\label{nba_approx}
		\begin{split}
			&\psi(x_1, x_2,\dots,x_n)\\
			&\approx\sum\limits_{j=1}^{T}\mathcal{C}^j\times_1f_{\theta_1^{j}}(x_1)\times_2f_{\theta_2^{j}}(x_2)\dots\times_nf_{\theta_n^{j}}(x_n),\\
		\end{split}
	\end{equation}
	where $x_1, x_2, \dots, x_n \in [0,1]$ represent input coordinates, $\mathcal{C}^{j}\in\mathbb{R}^{R_1^j \times R_2^j \times \dots \times R_n^j}$ denotes the learnable coefficient tensor of NeuApprox for each component $j=1,2,\dots,T$, and each univariate neural basis function $f_{\theta_i^{j}}: [0,1]\to\mathbb{R}^{R_i^j}$ is parameterized by an independent deep neural network with learnable parameters $\theta_i^{j}$ as indicated in Eq.~(\ref{inr_th}).
\end{definition}

Compared with single block term representation, employing multiple block terms representation allows the proposed NeuApprox to spontaneously characterize distinct components of the target function, thereby enhancing the overall approximation ability and resisting rapid overfitting (see details in Section~\ref{multiblockterms}). The powerful neural basis functions inherently endow the proposed NeuApprox with superior approximation ability and data adaptation ability compared to hand-crafted basis function-based methods.

Although the proposed NeuApprox offers the aforementioned advantages over classical hand-crafted basis functions-based multivariate approximation methods, those classical methods usually benefit from rigorous approximation error bounds. Therefore, a natural question is whether the proposed NeuApprox can similarly be endowed with theoretical error guarantees. To solidify the theoretical foundation of the proposed NeuApprox, we provide an explicit approximation error guarantee. The approximation ability of NeuApprox relies fundamentally on two classical theoretical results: the universal approximation theorem and the Weierstrass theorem. The universal approximation theorem guarantees that neural networks can approximate any continuous function on a compact domain to arbitrary accuracy, while the Weierstrass theorem establishes that any continuous function can be approximated to arbitrary precision using polynomial functions on a closed interval. Leveraging these fundamental results, we formalize the theoretical approximation capability of our NeuApprox framework in the following theorem:

\begin{theorem}{\bf (Neural Approximation Theorem)}
	For any continuous multivariate function $\psi: [0,1]^n \to \mathbb{R}$ and any desired accuracy $\epsilon > 0$, there exists a finite set of neural basis functions $\{f_{\theta_1^{j}}, f_{\theta_2^{j}}, \dots, f_{\theta_n^{j}}\}_{j=1}^T$ and corresponding coefficients tensors $\{\mathcal{C}^{j}\in\mathbb{R}^{R_{1}^j\times R_{2}^j\times R_n^j}\}_{j=1}^{T}$ such that:
	\begin{equation}
		\sup_{\boldsymbol{x} \in [0,1]^n} \left| \psi(\boldsymbol{x}) - f_{\Theta}(\boldsymbol{x}) \right| < \epsilon,
	\end{equation}
	where the NeuApprox approximation is given by:
	\begin{equation}\label{eq8}
		f_{\Theta}(\boldsymbol{x})=\sum\limits_{j=1}^{T}\mathcal{C}^j\times_1f_{\theta_1^{j}}(x_1)\times_2f_{\theta_2^{j}}(x_2)\dots\times_nf_{\theta_n^{j}}(x_n).
	\end{equation}
	Here, $\boldsymbol{x}=(x_1,x_2,\dots,x_n)$ with each $x_i\in[0,1]$, and the set $\Theta=\{\mathcal{C}^j, \theta_1^j, \theta_2^j, \dots, \theta_n^j\}_{j=1}^T$ encompasses all learnable parameters.
\end{theorem}

\begin{proof}
	Let $\mathcal{X}=[0,1]^n$ and consider the following set
	\[
	\mathcal{A}:=\left\{\sum_{j=1}^{T} \mathcal{C}^j \prod_{i=1}^{n} \times_ig_{i}^{\,j}(x_i)\;:\;
	T\in\mathbb{N},\ \mathcal{C}^j\in\mathbb{R}^{R_{1}^j\times R_{2}^j\times R_n^j},\ g_{i}^{\,j}\in C([0,1])\right\}\subset C(X).
	\]
	Then $\mathcal{A}$ is a subalgebra of $C(X)$: it is closed under addition and scalar
	multiplication by construction, and closed under multiplication since the product of two
	separable products remains separable.
	Moreover, $\mathcal{A}$ contains the constants (take all $g_i^{\,j}\equiv 1$), and it separates
	points of $X$: if $x\neq y$, choose $k$ with $x_k\neq y_k$ and set $g_k(t)=t$ and $g_i\equiv 1$
	for $i\neq k$, yielding $\prod_{i=1}^n g_i(x_i)=x_k\neq y_k=\prod_{i=1}^n g_i(y_i)$.
	By the Stone--Weierstrass theorem, $\mathcal{A}$ is dense in $C(X)$; hence, for any desired accuracy $\epsilon > 0$, there exist
	$T\in\mathbb{N}$, coefficients tensors $\mathcal{C}^1,\mathcal{C}^2,\dots,\mathcal{C}^T$, and functions $g_i^{\,j}\in C([0,1])$ such that
	\begin{equation}
		\label{eq:SW_step}
		\sup_{\boldsymbol{x}}\left|\psi(\boldsymbol{x})-\sum_{j=1}^{T} \mathcal{C}^j \prod_{i=1}^{n}\times_i g_{i}^{\,j}(x_i)\right|< \epsilon.
	\end{equation}

Let
\[
s(\boldsymbol{x}):=\sum_{j=1}^{T} \mathcal{C}^j \prod_{i=1}^{n} \times_ig_{i}^{\,j}(x_i).
\]
Fix any $\epsilon>0$ sufficiently small. By the universal approximation theorem (Theorem~\ref{inr_th}),
for each pair $(i,j)$ there exists a neural basis function $f_i^{\,j}\in\mathcal{N}$ such that

\begin{equation*}
	\label{eq:1d_approx}
	\sup_{x_i}\left|f_{\theta_i^j}(x_i)-g_{i}^{\,j}(x_i)\right| < \epsilon,
	\qquad i=1,\dots,n,\ \ j=1,\dots,T .
\end{equation*}

Then, we have

\begin{equation}\label{eq10}
	\begin{split}
		\sup_{\boldsymbol{x}}\left|f_{\Theta}(\boldsymbol{x})-s(\boldsymbol{x})\right|&=\sup_{\boldsymbol{x}}\left|\sum_{j=1}^{T} \mathcal{C}^j \prod_{i=1}^{n}\times_i \left(f_{\theta_i^j}(x_i)-g_{i}^{\,j}(x_i)\right)\right|\\
		&\leq\sum_{j=1}^{T}\sup_{\boldsymbol{x}}\left| \mathcal{C}^j \prod_{i=1}^{n}\times_i \left(f_{\theta_i^j}(x_i)-g_{i}^{\,j}(x_i)\right)\right|\\
		&\leq\sum_{j=1}^{T}c^j\epsilon^n\prod_{i=1}^{n}R_i^j\\
		&\leq\epsilon,
	\end{split}
\end{equation}
where $c^j$ is the maxima element of coefficients tensor $\mathcal{C}^j$.

Combining this estimate with \eqref{eq:SW_step} and \eqref{eq10}
the following inequality holds
\[
\sup_{\boldsymbol{x}}\left|\psi(\boldsymbol{x})-f_\Theta(\boldsymbol{x})\right|
\le \sup_{\boldsymbol{x}}\left|\psi(\boldsymbol{x})-s(\boldsymbol{x})\right| + \sup_{\boldsymbol{x}}\left|s(\boldsymbol{x})-f_\Theta(\boldsymbol{x})\right|
< \epsilon + \epsilon
=\epsilon.
\]

\end{proof}

This theoretical guarantee not only underscores the expressive potential of the proposed NeuApprox but also provides rigorous assurance for its performance in practical scenarios, given sufficient capacity in terms of network complexity and number of basis functions.

Additionally, NeuApprox closely aligns with low-rank tensor function representations, enabling intuitive interpretations in terms of well-known tensor decompositions. Specifically, the proposed NeuApprox reduces to the low-rank Canonical Polyadic (CP) function representation when $R_i^j=1$ for all $j=1,2,\dots,T$. Under discrete sampling conditions (meshgrid coordinates), this further corresponds precisely to the classical CP decomposition. Similarly, when the number of components $T=1$, NeuApprox becomes identical to a low-rank Tucker function representation. With discrete sampling on the meshgrid, NeuApprox exactly recovers the classical low-rank Tucker decomposition.

In summary, the proposed NeuApprox constitutes a powerful, flexible, and theoretically sound methodology for multivariate continuous function approximation. NeuApprox bridges classical approximation methods and modern deep learning techniques, offering substantial improvements in capturing complex data-driven structures across diverse multi-dimensional application domains.

\subsection{NeuApprox for Multi-Dimensional Data Recovery}\label{sec_model}

To examine the effectiveness of the proposed NeuApprox, we build some representative NeuApprox-based multi-dimensional data recovery models, including meshgrid data completion (i.e., multi-dimensional image inpainting and traffic data completion) and off-meshgrid data completion (i.e., point cloud data completion). Next, we will introduce the details of these multi-dimensional data recovery models.

\subsubsection{Multi-Dimensional Data Inpainting}
Multi-dimensional data inpainting, which is a classical ill-posed inverse problem, aims to recover the underlying data from the observed incomplete data. In this paper, we consider the representative multi-dimensional image inpainting task and traffic data completion task. Specifically, given an observed data ${\mathcal Y}\in{\mathbb R}^{I_1\times I_2\times\cdots\times I_n}$ with observed indicator set $\Omega$, the optimization model of our NeuApprox-based multi-dimensional data completion can be formulated as follows:
\begin{equation}\label{loss_inpainting}
	\begin{split}
		&\min_{\substack{{\Theta}}}\lVert {\mathcal P}_\Omega({\mathcal Y}\!-\!\mathcal{X})\rVert_{F}^2,\\
		&{\rm s.t.} \mathcal{X}(i_1,\cdots,i_n)=\sum\limits_{j=1}^{T}\mathcal{C}^j\!\times_1\!f_{\theta_1^{j}}(i_1)\!\dots\!\times_n\!f_{\theta_n^{j}}(i_n),
	\end{split}
\end{equation}
where $\mathcal{X}$ is the underlying multi-dimensional data represented by the proposed NeuApprox. $\{f_{\theta_1^j},\cdots,f_{\theta_n^j}\}_{j=1}^T$ are neural basis functions and $\{\mathcal{C}^j\}_{j=1}^T$ are the corresponding coefficients. $\Theta=\{\mathcal{C}^j, \theta_1^j, \theta_2^j, \dots, \theta_n^j\}_{j=1}^T$ encompasses all learnable parameters including the coefficients and parameters of neural basis functions. ${\mathcal P}_\Omega(\cdot)$ denotes the projection operator that keeps the elements in $\Omega$ and sets others to zero. The recovered result is obtained by ${\mathcal P}_\Omega({\mathcal Y})+{\mathcal P}_{\Omega^C}({\mathcal X})$, where $\Omega^c$ denotes the complement of $\Omega$. We adopt the Adam optimizer to address the multi-dimensional data completion model (\ref{loss_inpainting}) by optimizing the parameters $\Theta$. We summarize the proposed NeuApprox-based multi-dimensional data completion method in Algorithm \ref{multi-dimensional data inpainting}.

\begin{algorithm}[h]
	\caption{NeuApprox-based Multi-Dimensional Data Completion}
	\label{multi-dimensional data inpainting}
	\begin{algorithmic}[1] 
		\State \textbf{Input:}  Observed data $\mathcal{Y}$, observed indicator set $\Omega$, number of terms $T$, and coefficient size $\Big\{\Big(R_{1}^{j},R_{2}^{j},\cdots,R_{n}^{j}\Big)\Big\}_{j=1}^{T}$. 
		\State \textbf{Output:} Recovered data $\mathcal{X}$.

		\State \textbf{For} {$k=1,2,\cdots,k_{max}$}
		\State ~~~ Element-wisely compute the recovered data via \begin{equation*}
			\begin{split}
				&\mathcal{X}(i_1,i_2,\cdots,i_n)\\
				&=\sum\limits_{j=1}^{T}\mathcal{C}^j\times_1f_{\theta_1^{j}}(i_1)\times_2f_{\theta_2^{j}}(i_2)\dots\times_nf_{\theta_n^{j}}(i_n);
			\end{split}
		\end{equation*}
		\State ~~~ Update $\{\theta_1^{j},\theta_2^{j},\cdots,\theta_n^{j}\}_{j=1}^{T}$ and $\{\mathcal{C}^{j}\}_{j=1}^{T}$ via Adam;
		\State \textbf{End for}
		\State \textbf{return} $\mathcal{X}=\mathcal{P}_{\Omega}\left(\mathcal{Y}\right)+\mathcal{P}_{\Omega^c}\left(\mathcal{X}\right)$.
	\end{algorithmic}
\end{algorithm}

\subsubsection{Point Cloud Data Completion}\label{subsec:traffic_completion}

Since the proposed NeuApprox can approximate any continuous function, it can be used to represent multi-dimensional data both on and beyond the meshgrid. Therefore, to thoroughly evaluate the superiority of the proposed NeuApprox, we establish a NeuApprox-based point cloud data completion model. Generally, given an observed incomplete point cloud data $\mathcal{Y}$ and the index set of the observed points $\Omega=\{(i_{1,m}, i_{2,m}, i_{3,m}, i_{4,m})\}_{m=1}^M$, the NeuApprox-based point cloud data completion model can be formulated as follows:

\begin{equation}\label{pc_completion}
	\begin{split}
		\min\limits_{\Theta}&\sum\limits_{{\boldsymbol{i}}\in\Omega}\|\mathcal{Y}(i_1,i_2,i_3,i_4)-\mathcal{X}(i_1,i_2,i_3,i_4)\|_F^2,\\
		{\rm s.t.} \quad &\mathcal{X}(i_1,i_2,i_3,i_4)\\
		&=\!\sum\limits_{j=1}^{T}\mathcal{C}^j\!\!\times_1\!\!f_{\theta_1^{j}}(i_1)\!\!\times_2\!\!f_{\theta_2^{j}}(i_2)\!\!\times_3\!\!f_{\theta_3^{j}}(i_3)\!\!\times_4\!\!f_{\theta_4^{j}}(i_4),
	\end{split}
\end{equation}
where ${\boldsymbol{i}}=(i_1,i_2,i_3,i_4)$ is the coordinate of the point cloud data. $(i_1,i_2,i_3)$ denotes the physical spatial coordinates of the point, and $i_4=1,2,3$ denotes the color channels, i.e., red, green, and blue channels, respectively. $\{f_{\theta_1^j}, f_{\theta_2^j}, f_{\theta_3^j}, f_{\theta_4^j}\}_{j=1}^T$ are neural basis functions, and $\{\mathcal{C}^j\}_{j=1}^T$ are the corresponding learnable coefficients. $\Theta=\{\mathcal{C}^{j},\theta_{1}^{j},\theta_{2}^{j},\theta_{3}^{j},\theta_{4}^{j}\}_{j=1}^{T}$ includes all learnable coefficients and the parameters of the neural basis functions. 
To address the optimization problem (\ref{pc_completion}), we consider using the popular gradient descent optimization algorithm, e.g., Adam, to efficiently optimize $\Theta$.

\begin{table*}[!htbp]
	\caption{The average quantitative results by different methods for multi-dimensional image inpainting. The {\bf best} and \underline{second-best} values are highlighted. (PSNR $\uparrow$, SSIM $\uparrow$, and NRMSE $\downarrow$)\label{tab_completion}}\vspace{0.2cm}
	\begin{center}
		\tiny
		\setlength{\tabcolsep}{0.5pt}
		\begin{spacing}{1.2}
			\begin{tabular}{clccccccccccccccc}
				\toprule
				\multicolumn{2}{c}{Sampling rate}&\multicolumn{3}{c}{0.1}&\multicolumn{3}{c}{0.15}&\multicolumn{3}{c}{0.2}&\multicolumn{3}{c}{0.25}&\multicolumn{3}{c}{0.3}\\
				\cmidrule{1-17}
				Data&Method&PSNR &SSIM &NRMSE \;\; &PSNR &SSIM &NRMSE \;\; &PSNR &SSIM &NRMSE \;\; &PSNR &SSIM&NRMSE \;\;&PSNR &SSIM&NRMSE\\
				\midrule
				\multirow{7}{*}{\makecell[c]{MSIs\\ \textit{Toys}\\ \textit{Painting}\\ (256$\times$256$\times$31)}}
				
				&Observed&{12.30}&{0.315}&{0.247}\;\;&{12.54}&{0.349}&{0.240}\;\;&{12.81}&{0.382}&{0.232}\;\;&{13.09}&{0.413}&{0.225}\;\;&{13.38}&{0.444}&{0.218} \\
				&TRLRF&{30.08}&{0.919}&{0.031}\;\;&{31.58}&{0.939}&{0.026}\;\;&{32.79}&{0.953}&{0.023}\;\;&{33.73}&{0.961}&{0.021}\;\;&{34.54}&{0.967}&{0.019}\\
				&TNN&{29.20}&{0.928}&{0.035}\;\;&{31.64}&{0.956}&{0.026}\;\;&{33.73}&{0.972}&{0.021}\;\;&{35.71}&{\bf 0.981}&\underline{0.016}\;\;&\underline{37.62}&{\bf 0.988}&\underline{0.013}\\
				&FCTN&{32.08}&{0.940}&{0.025}\;\;&{33.74}&{0.956}&{0.021}\;\;&{34.25}&{0.962}&{0.019}\;\;&{35.63}&{0.970}&{0.017}\;\;&{36.12}&{0.973}&{0.016}\\
				&FSA&{21.54}&{0.664}&{0.084}\;\;&{21.55}&{0.669}&{0.084}\;\;&{21.75}&{0.675}&{0.082}\;\;&{21.84}&{0.680}&{0.081}\;\;&{21.78}&{0.687}&{0.082}\\
				&LRTFR&\underline{33.03}&{\bf 0.956}&\underline{0.022}\;\;&\underline{34.01}&\underline{0.963}&\underline{0.020}\;\;&\underline{35.50}&\underline{0.973}&\underline{0.017}\;\;&\underline{36.15}&{0.976}&\underline{0.016}\;\;&{36.64}&{0.978}&{0.015} \\
				&NeuApprox&{\bf 33.70}&\underline{0.951}&{\bf 0.021}\;\;&{\bf 35.09}&{\bf 0.964}&{\bf 0.018}\;\;&{\bf 36.99}&{\bf 0.975}&{\bf 0.014}\;\;&{\bf 38.18}&\underline{0.980}&{\bf 0.012}\;\;&{\bf 38.77}&\underline{0.982}&{\bf 0.012}\\
				\midrule
				\multirow{7}{*}{\makecell[c]{Videos\\ \textit{Foreman}\\ \textit{Carphone}\\ (144$\times$176$\times$150)}}
				
				&Observed&{5.42}&{0.039}&{0.593}\;\;&{5.66}&{0.051}&{0.576}\;\;&{5.93}&{0.062}&{0.559}\;\;&{6.21}&{0.073}&{0.541}\;\;&{6.51}&{0.084}&{0.523}\\
				&TRLRF&{26.43}&{0.863}&{0.052}\;\;&{27.27}&{0.881}&{0.048}\;\;&{27.97}&{0.894}&{0.044}\;\;&{28.42}&{0.902}&{0.042}\;\;&{28.93}&{0.912}&{0.040}\\
				&TNN&{25.24}&{0.851}&{0.060}\;\;&{27.10}&\underline{0.895}&{0.048}\;\;&\underline{28.55}&{\bf 0.922}&\underline{0.041}\;\;&\underline{29.77}&{\bf 0.939}&\underline{0.036}\;\;&\underline{30.87}&{\bf 0.952}&\underline{0.031}\\
				&FCTN&{26.32}&{0.843}&{0.054}\;\;&\underline{27.33}&{0.872}&{0.048}\;\;&{28.01}&{0.888}&{0.044}\;\;&{28.40}&{0.898}&{0.042}\;\;&{28.96}&{0.909}&{0.040}\\
				&FSA&{18.98}&{0.639}&{0.123}\;\;&{19.37}&{0.651}&{0.117}\;\;&{19.03}&{0.646}&{0.122}\;\;&{19.04}&{0.648}&{0.122}\;\;&{18.54}&{0.635}&{0.129}\\
				&LRTFR&\underline{26.77}&\underline{0.881}&\underline{0.050}\;\;&\underline{27.33}&{0.894}&\underline{0.047}\;\;&{27.72}&{0.903}&{0.045}\;\;&{28.02}&{0.909}&{0.043}\;\;&{28.19}&{0.912}&{0.042}\\
				&NeuApprox&{\bf 28.15}&{\bf 0.887}&{\bf 0.043}\;\;&{\bf 29.44}&{\bf 0.917}&{\bf 0.037}\;\;&{\bf 29.79}&\underline{0.913}&{\bf 0.036}\;\;&{\bf 31.02}&\underline{0.932}&{\bf 0.031}\;\;&{\bf 31.75}&\underline{0.944}&{\bf 0.028}\\
				\midrule
				\multirow{7}{*}{\makecell[c]{Light field data\\ \textit{Greek}\\ \textit{Origami}\\ (128$\times$128$\times240$)}}
				
				&Observed & 4.79 & 0.015 & 0.589& 5.04 & 0.022 & 0.573& 5.30 & 0.031 & 0.556& 5.58 & 0.040 & 0.538 & 5.88 & 0.051 & 0.520\\

				&TRLRF & 31.73 & 0.934 & 0.039& 34.47 & 0.964 & 0.025& 37.58 & 0.982 & \underline{0.015}& \underline{39.54} & \underline{0.987} & \underline{0.012}& 41.00 & \underline{0.991} & 0.012\\
				&TNN&{29.20}&{0.928}&{0.035}&{31.64}&{0.956}&{0.026}&{33.73}&{0.972}&{0.021}&{35.71}&{0.981}&0.016&37.62&0.988&0.013\\
				&FCTN & 31.28 & 0.939 & 0.028 & \underline{34.70} & \underline{0.967} & \underline{0.020}& \underline{37.75} & \underline{0.987} & 0.029& 39.09 & \underline{0.987} & 0.026& \underline{41.23}& \underline{0.991}& \underline{0.011}\\
				&FSA & 18.95 & 0.508 & 0.115& 18.80 & 0.513 & 0.118& 18.74 & 0.507 & 0.123& 18.41 & 0.509 & 0.126 & 19.11 & 0.519 & 0.114\\
				&LRTFR & \underline{32.06} & 0\underline{.949} & \underline{0.027}& 33.56 & 0.960 & 0.023& 34.46 & 0.966 & 0.022& 34.89 & 0.969 & 0.020 & 35.43 & 0.972 & 0.018 \\
				&NeuApprox & \textbf{32.72} & \textbf{0.976} & \textbf{0.021}& \textbf{34.98} & \textbf{0.974} & \textbf{0.018}& \textbf{38.74} & \textbf{0.989} & \textbf{0.013}& \textbf{40.70} & \textbf{0.988} & \textbf{0.011}&\textbf{42.70} & \textbf{0.995} & \textbf{0.009}\\
				
				\bottomrule
			\end{tabular}
		\end{spacing}
	\end{center}
\end{table*}

\section{Experiments}
\label{sec:experiments}

The proposed NeuApprox is capable of handling both meshgrid and off-meshgrid data. To comprehensively validate the effectiveness of the proposed method, we conduct experiments on the classical problem of meshgrid multi-dimensional data recovery (i.e., multi-dimensional image inpainting). To further assess its adaptability, we evaluate the method on non-image datasets, such as traffic data. Because our approach is fundamentally a multivariate function approximation, we also test it on off-meshgrid data, e.g., point cloud datasets. Next, we introduce some detailed experimental settings for diverse multi-dimensional data and then provide baselines and corresponding experimental results for different tasks. Our method is implemented on PyTorch 2.0.1 with an i9-9900K CPU and an RTX 2080Ti GPU.

\begin{figure*}[h]
	\tiny
	\setlength{\tabcolsep}{1pt}
	\begin{center}
		\begin{tabular}{cccccccc}              
			\includegraphics[width=0.12\textwidth]{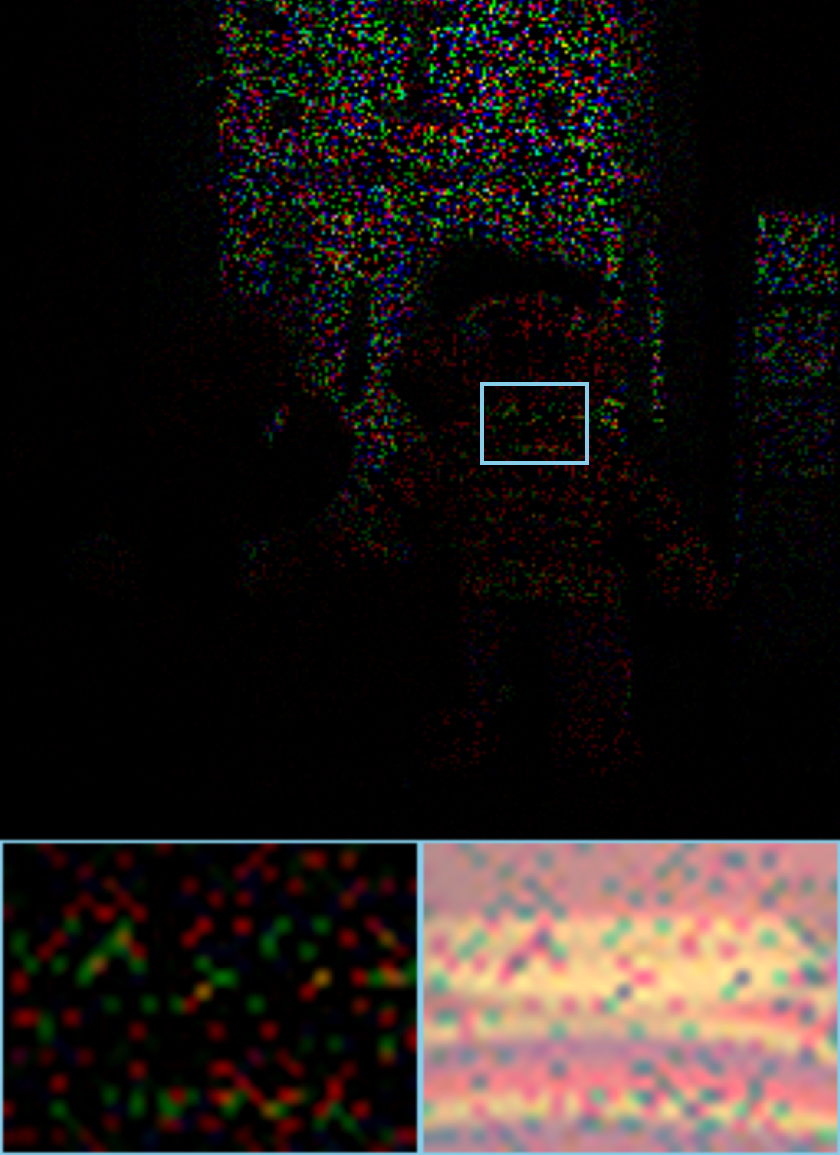}&
			\includegraphics[width=0.12\textwidth]{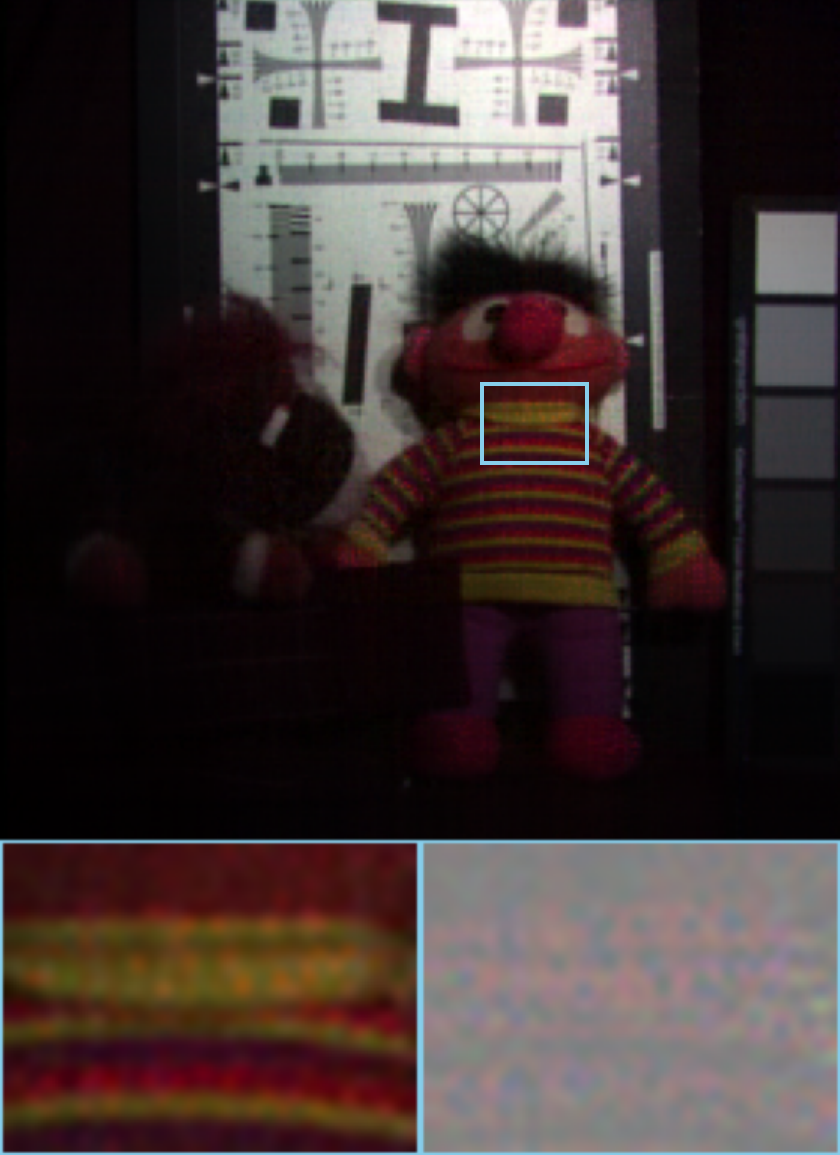}&
			\includegraphics[width=0.12\textwidth]{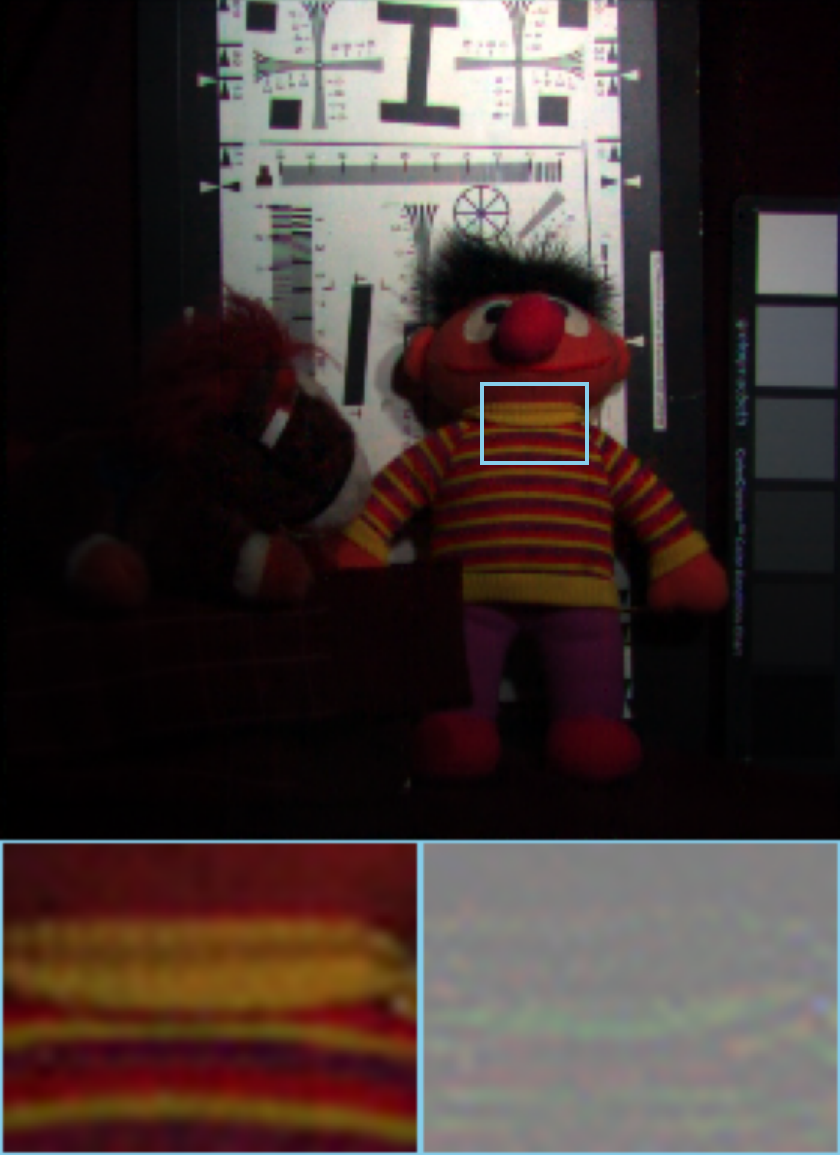}&
			\includegraphics[width=0.12\textwidth]{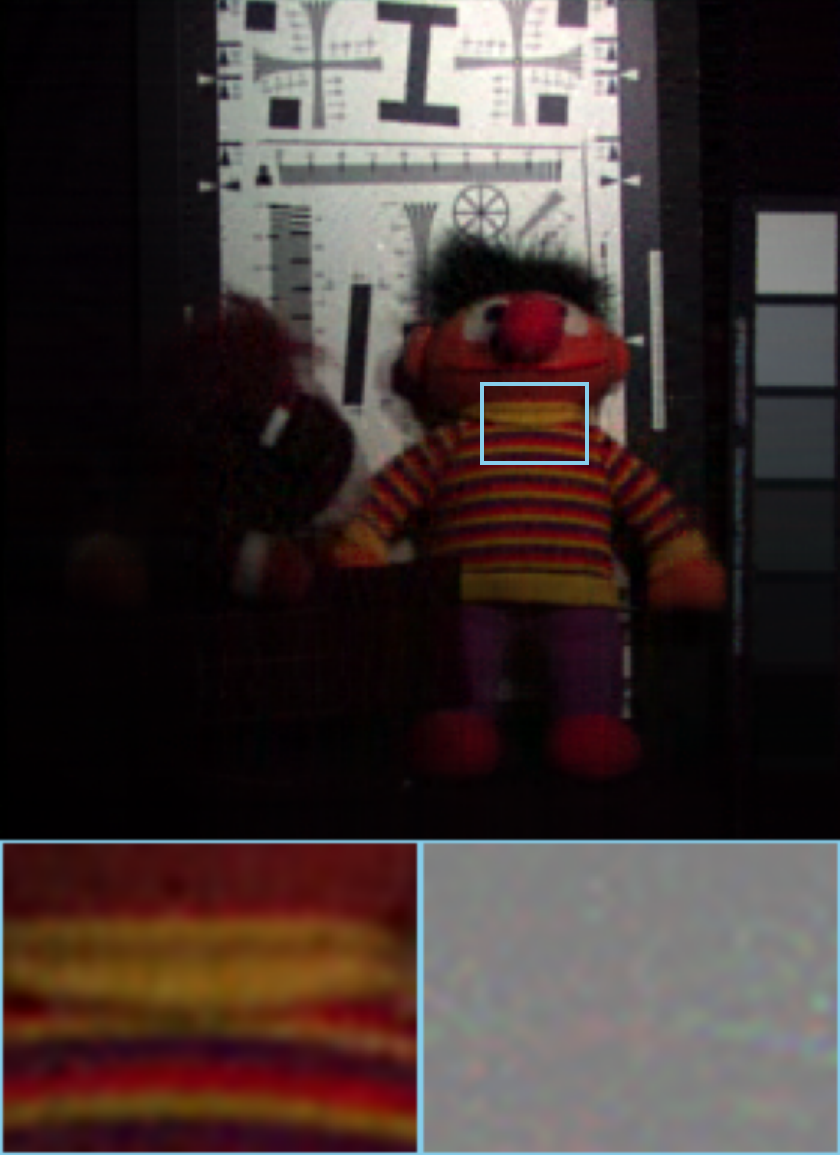}&
			\includegraphics[width=0.12\textwidth]{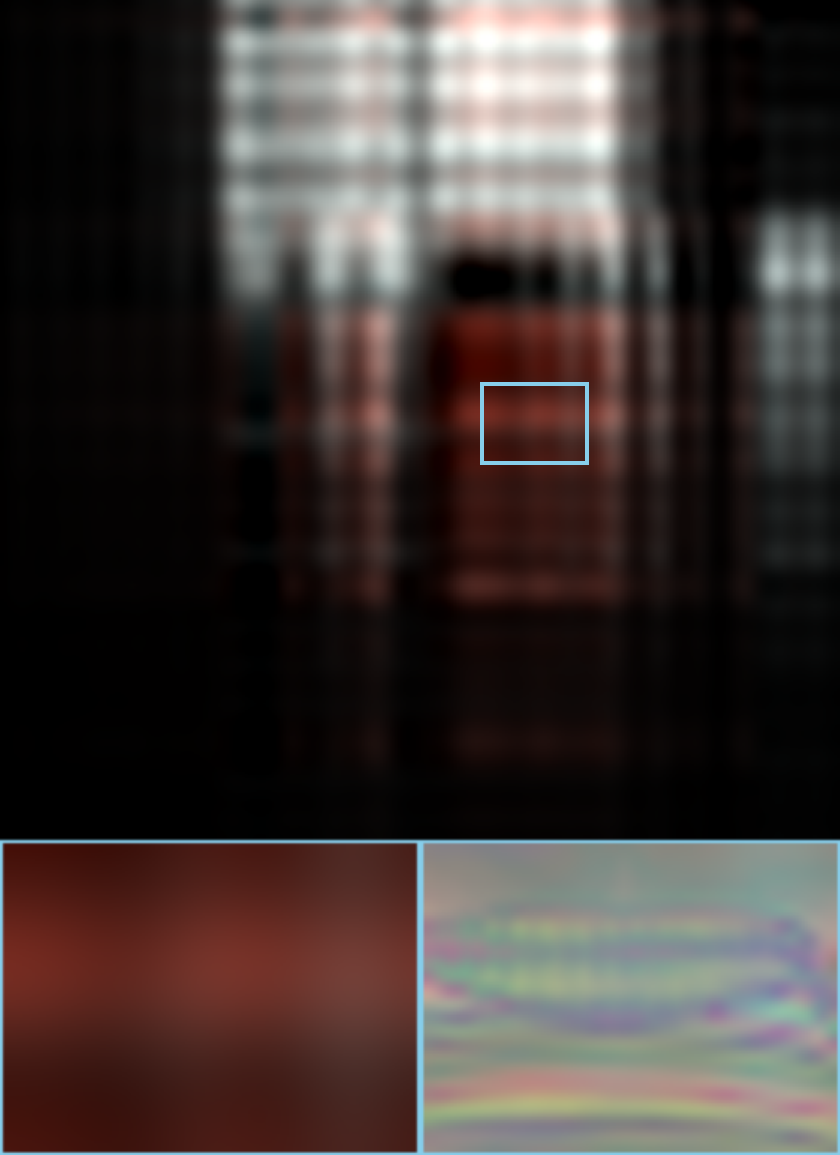}&
			\includegraphics[width=0.12\textwidth]{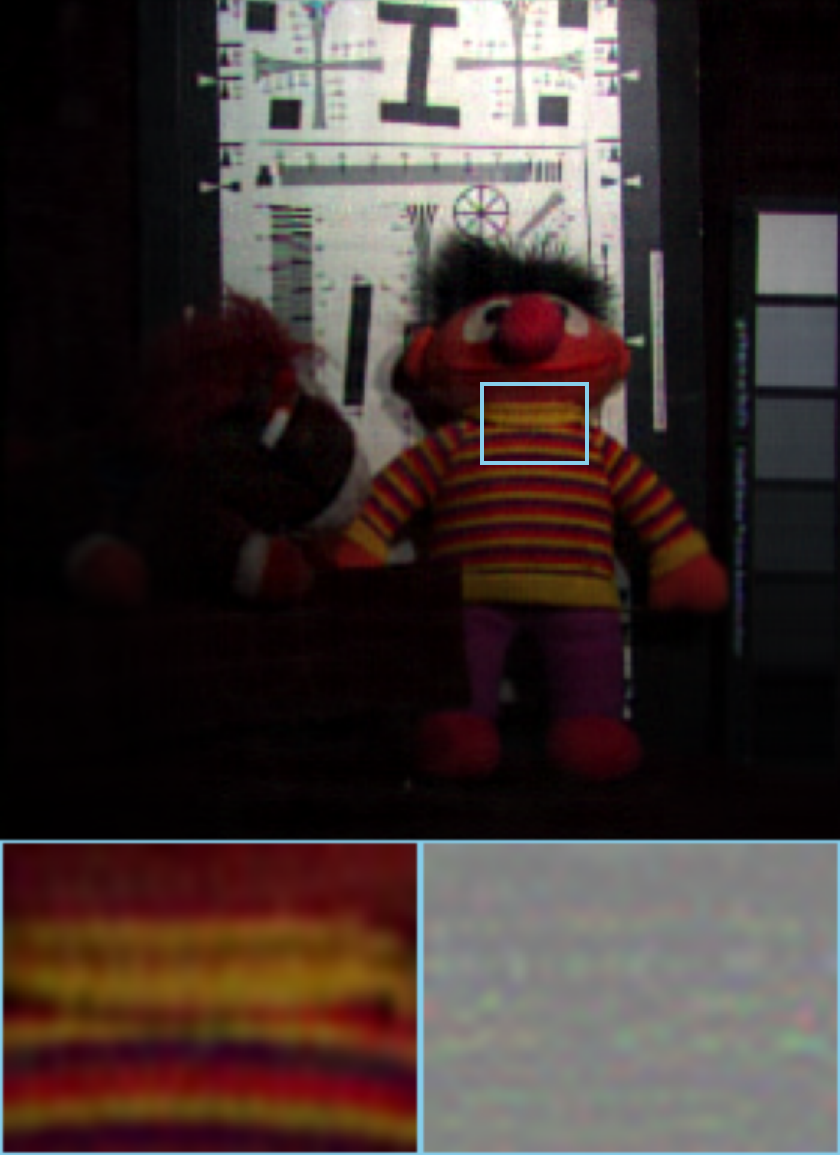}&
			\includegraphics[width=0.12\textwidth]{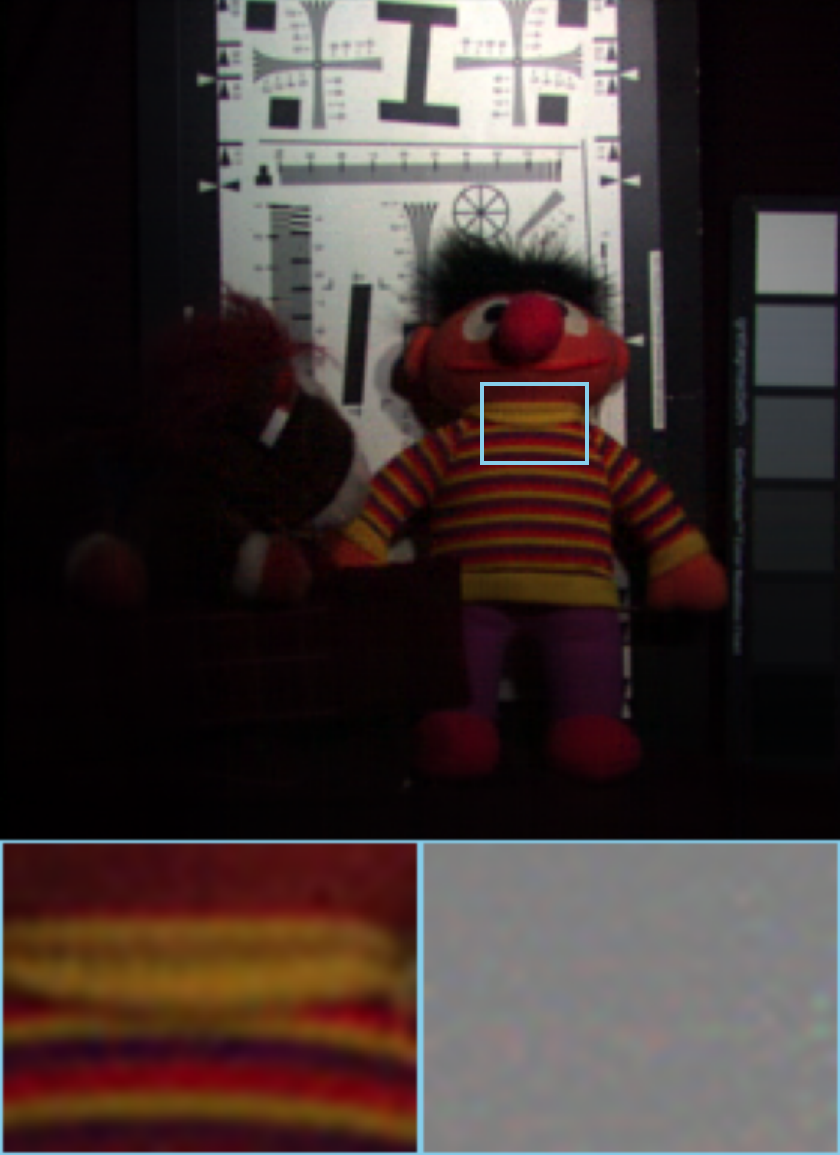}&
			\includegraphics[width=0.12\textwidth]{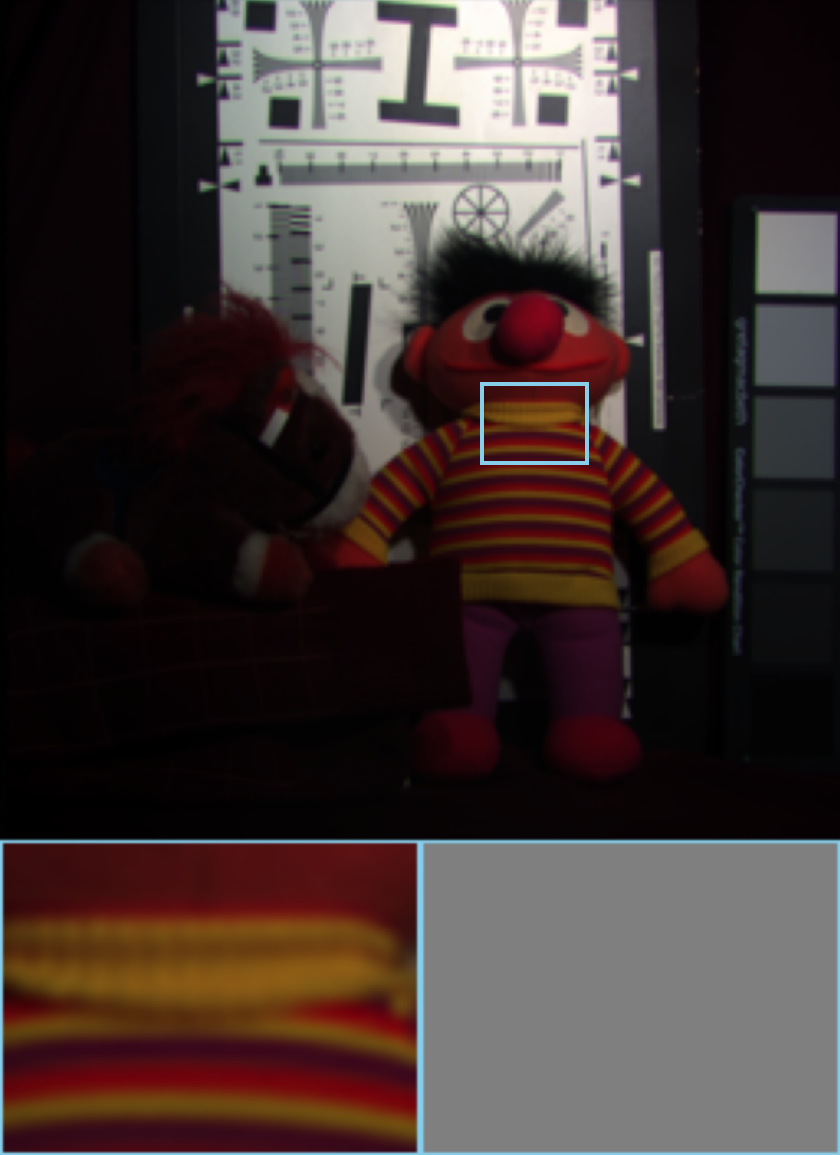}\\[1pt]
			
			\includegraphics[width=0.12\textwidth]{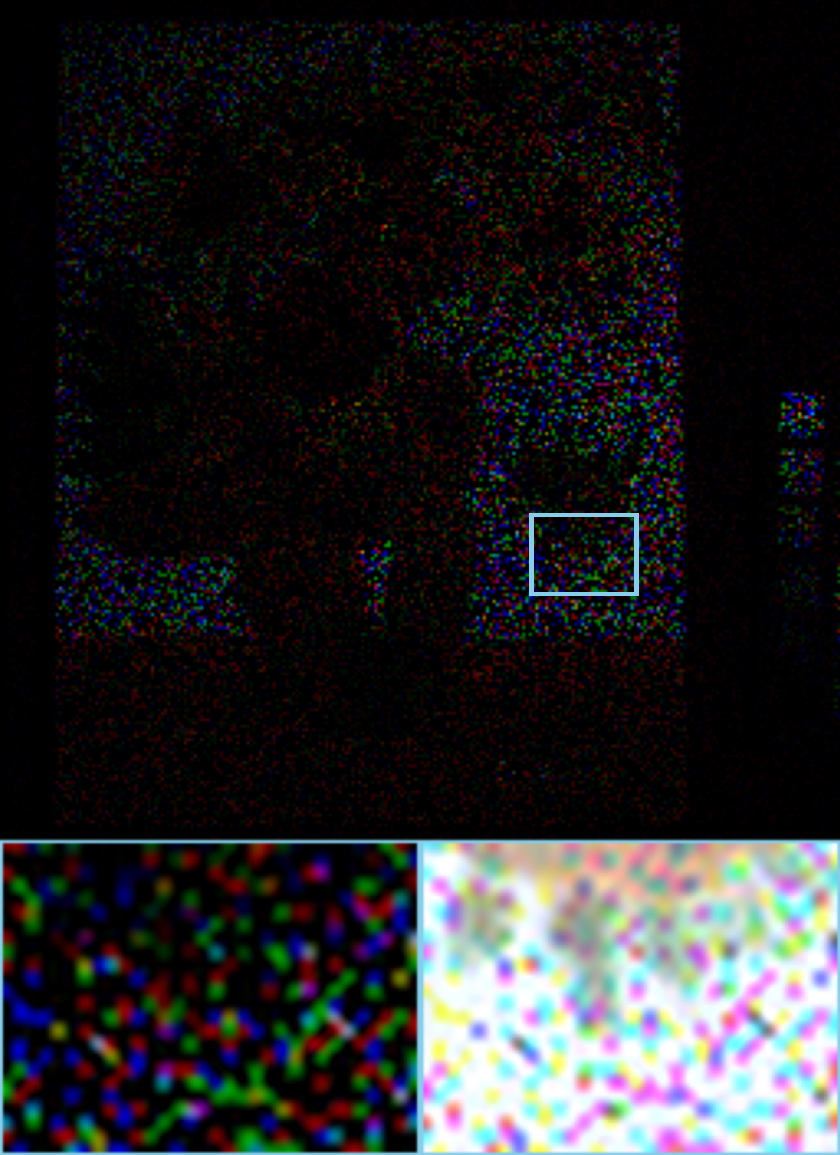}&
			\includegraphics[width=0.12\textwidth]{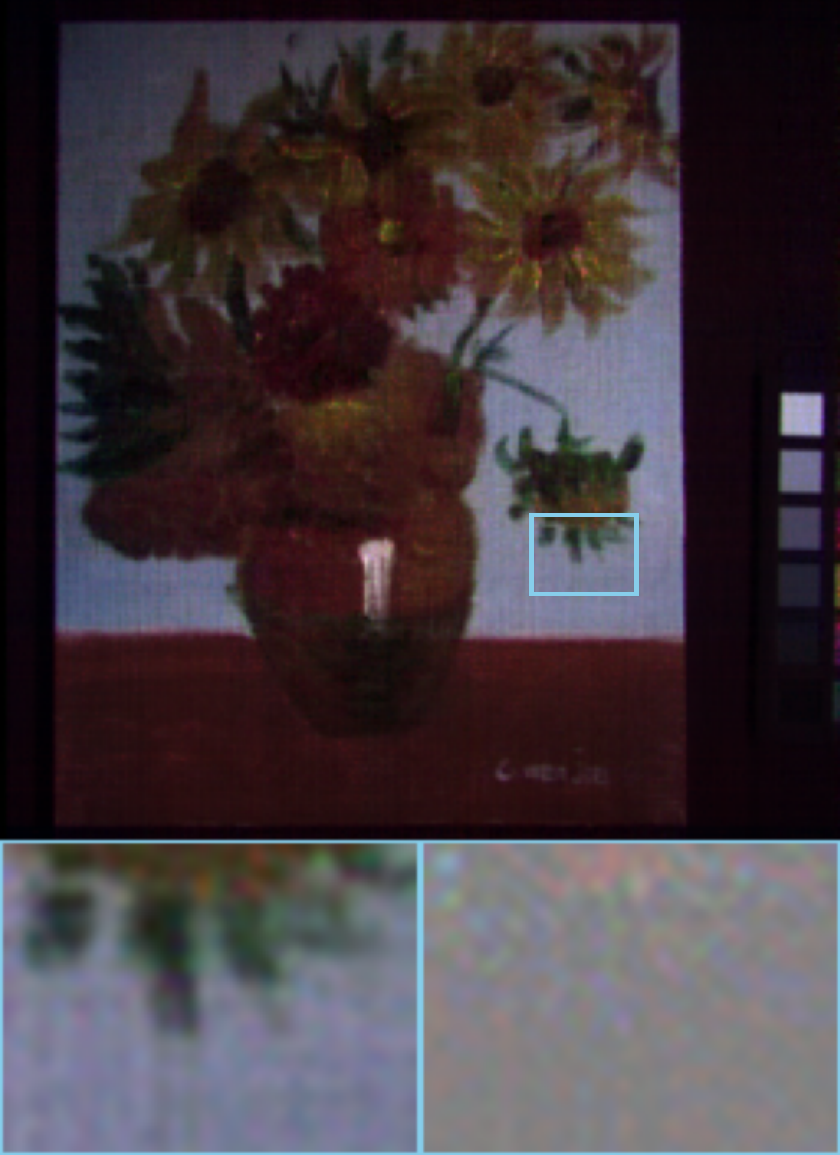}&
			\includegraphics[width=0.12\textwidth]{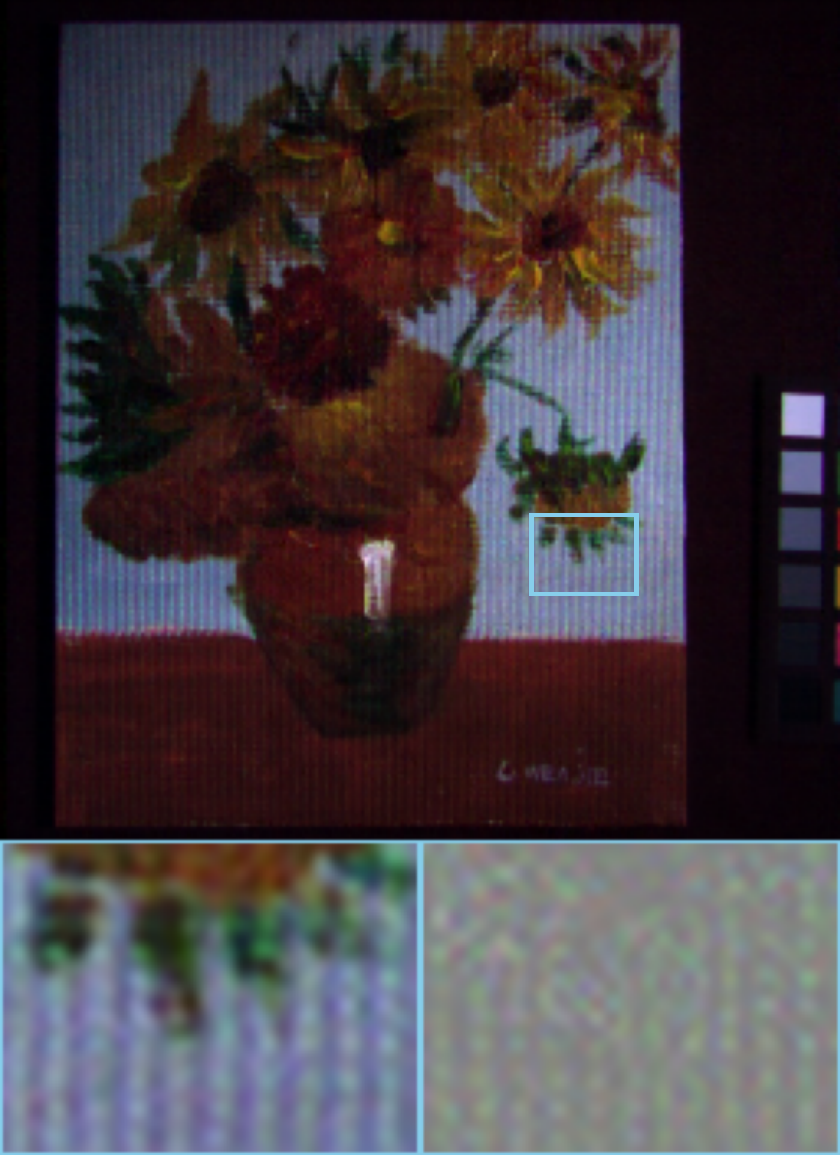}&
			\includegraphics[width=0.12\textwidth]{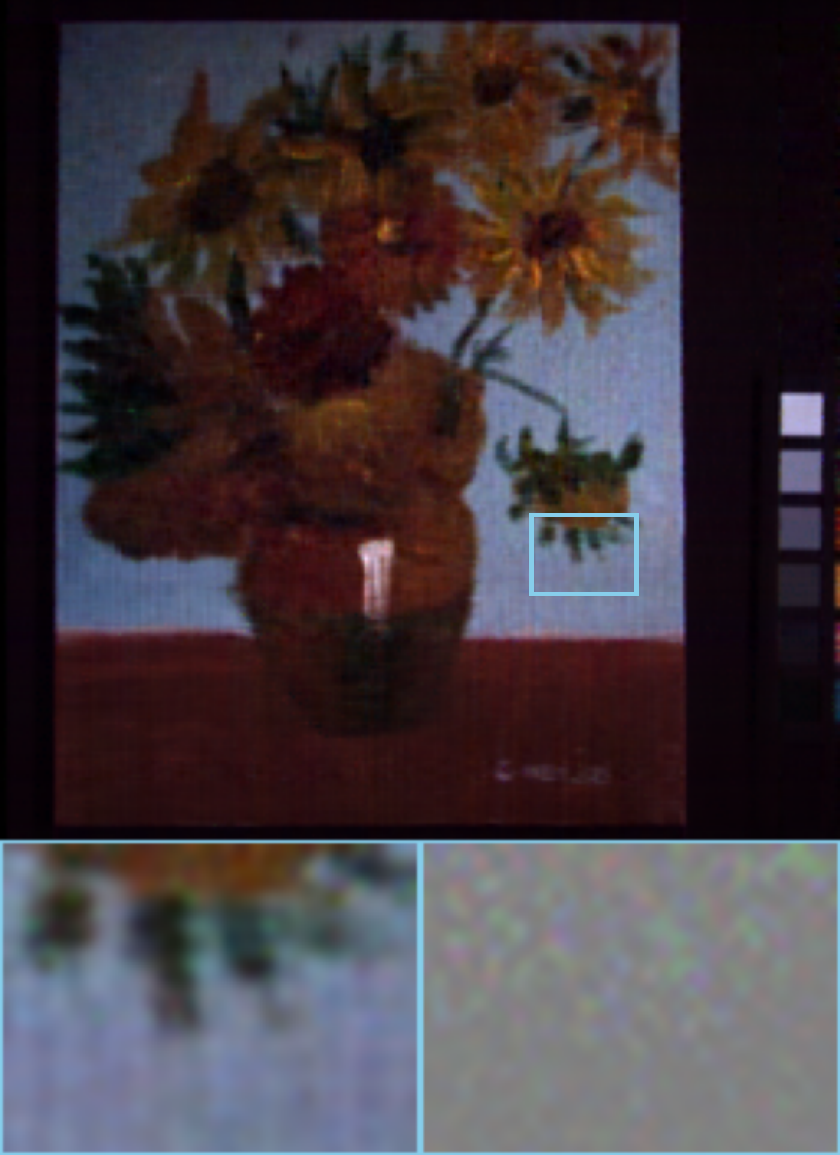}&
			\includegraphics[width=0.12\textwidth]{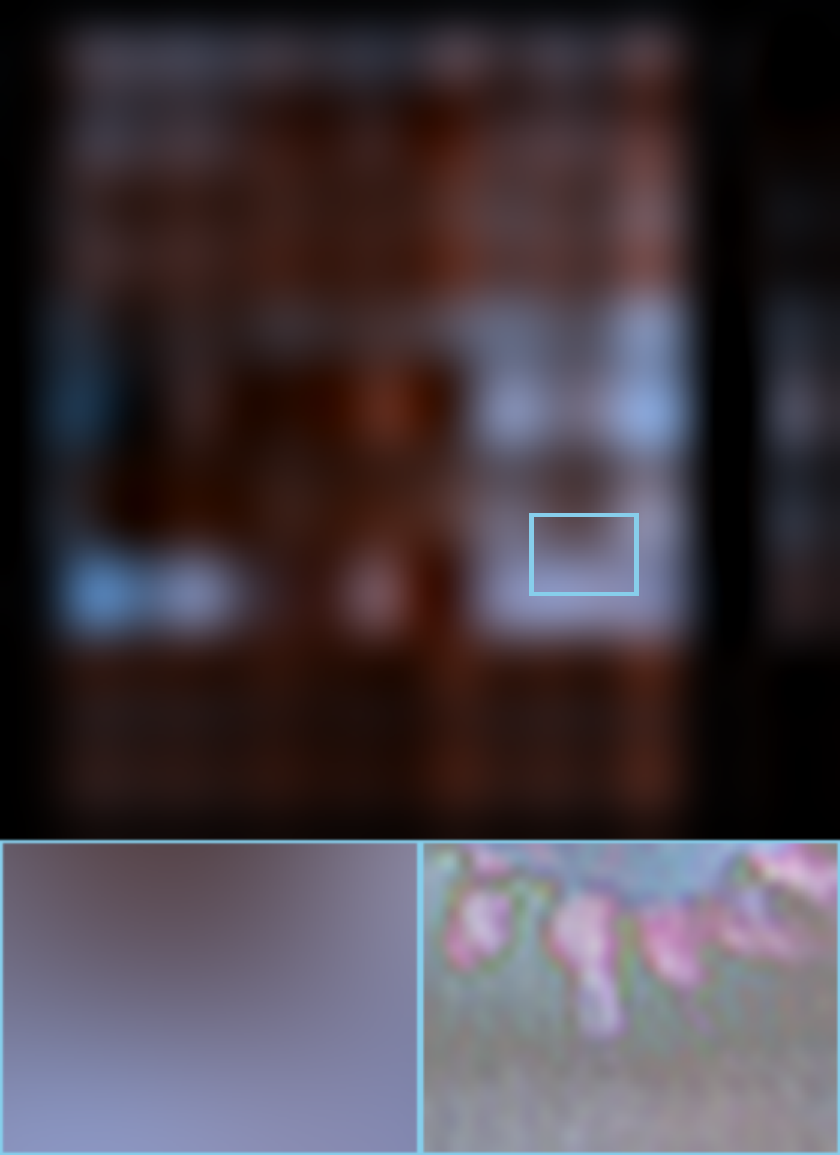}&
			\includegraphics[width=0.12\textwidth]{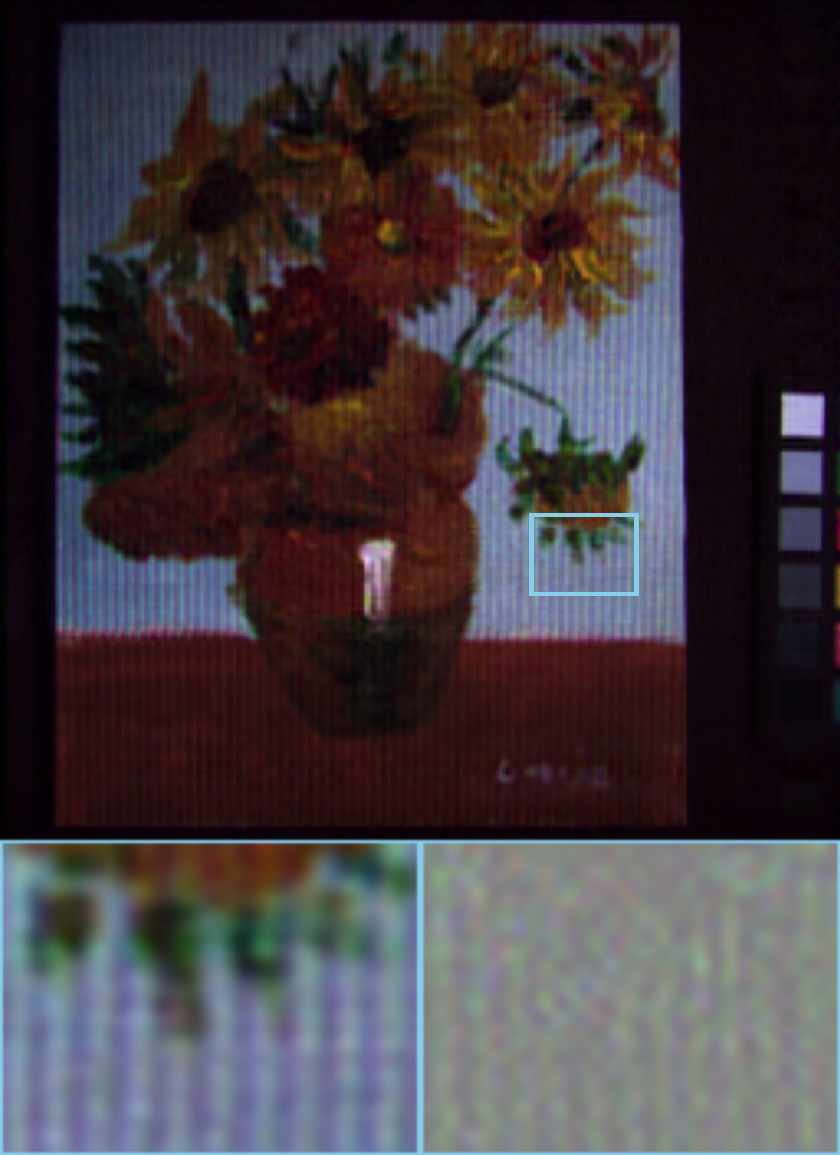}&
			\includegraphics[width=0.12\textwidth]{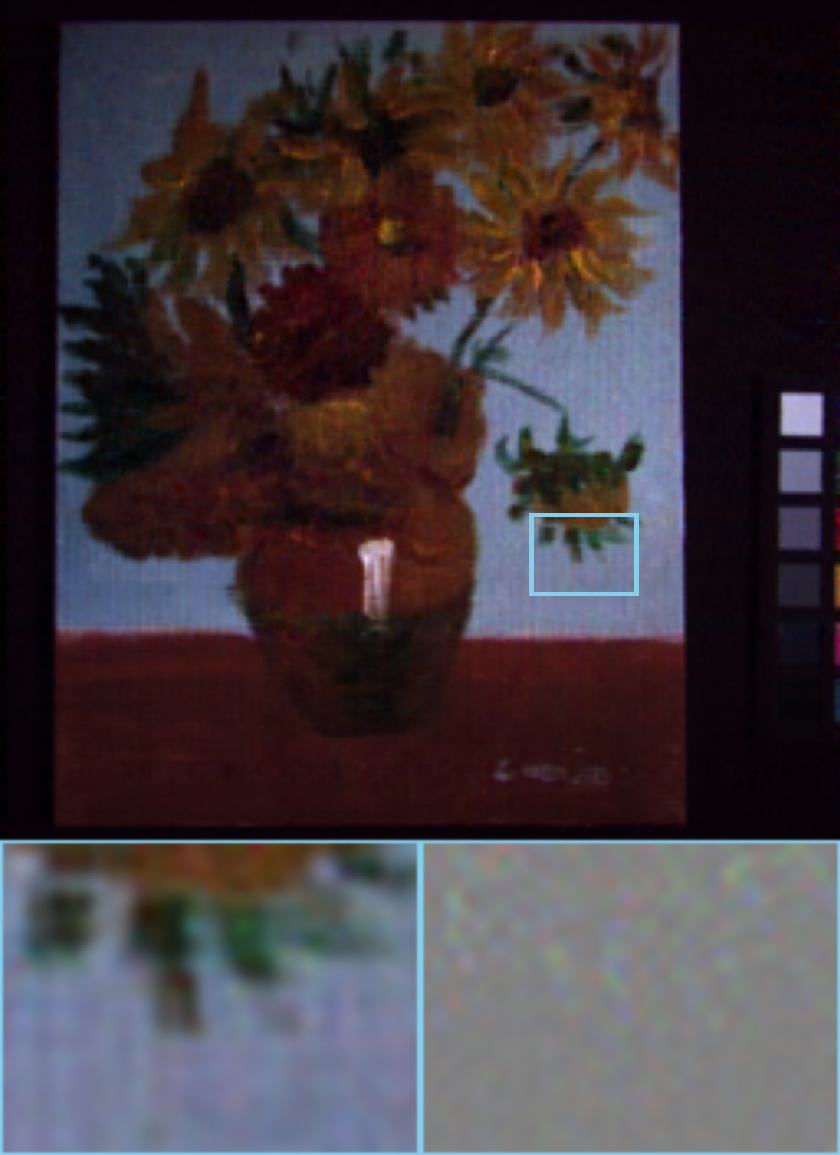}&
			\includegraphics[width=0.12\textwidth]{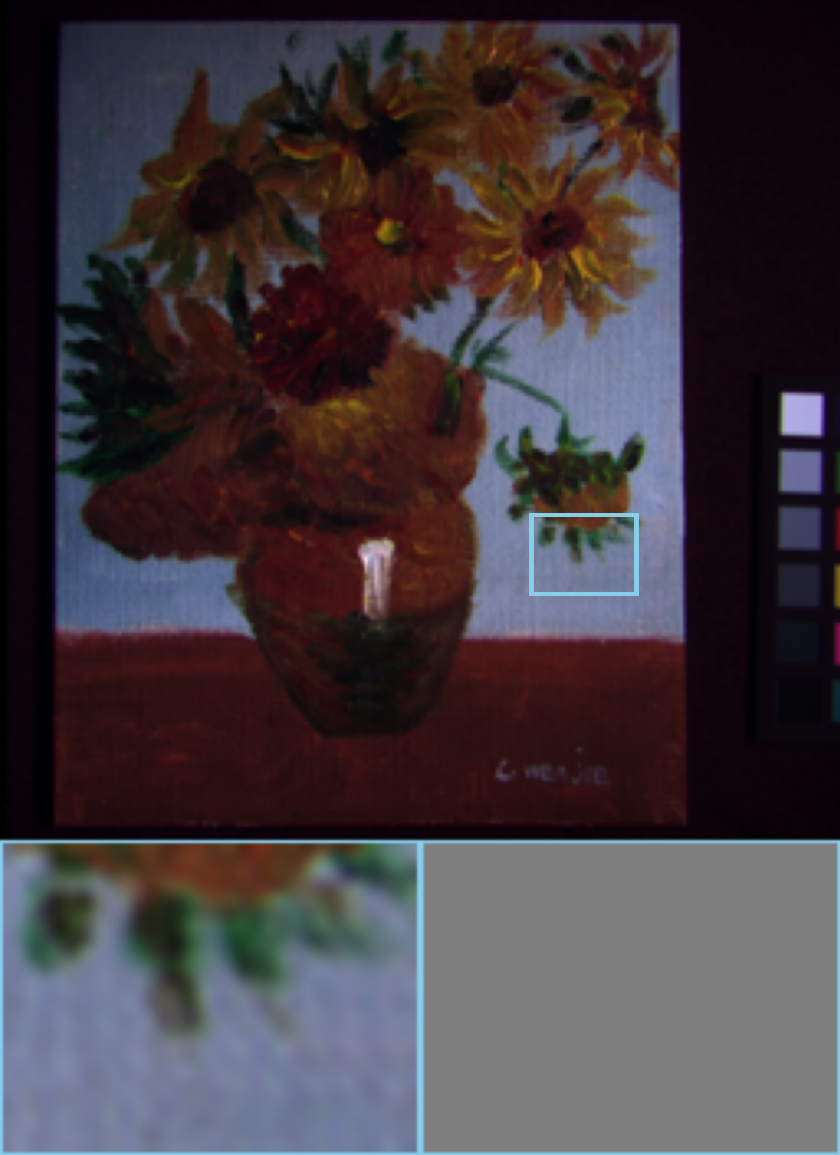}\\[1pt]
			\includegraphics[width=0.12\textwidth]{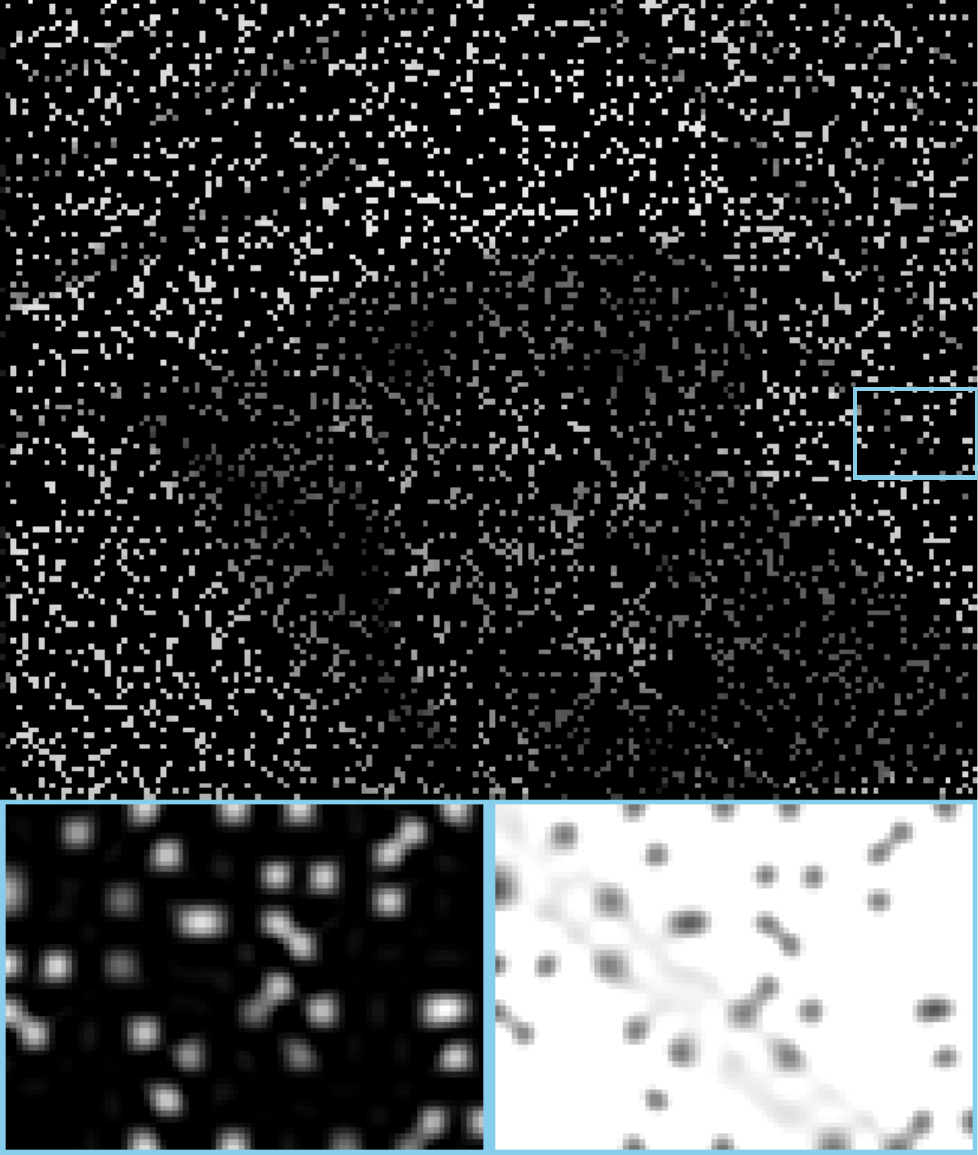}&
			\includegraphics[width=0.12\textwidth]{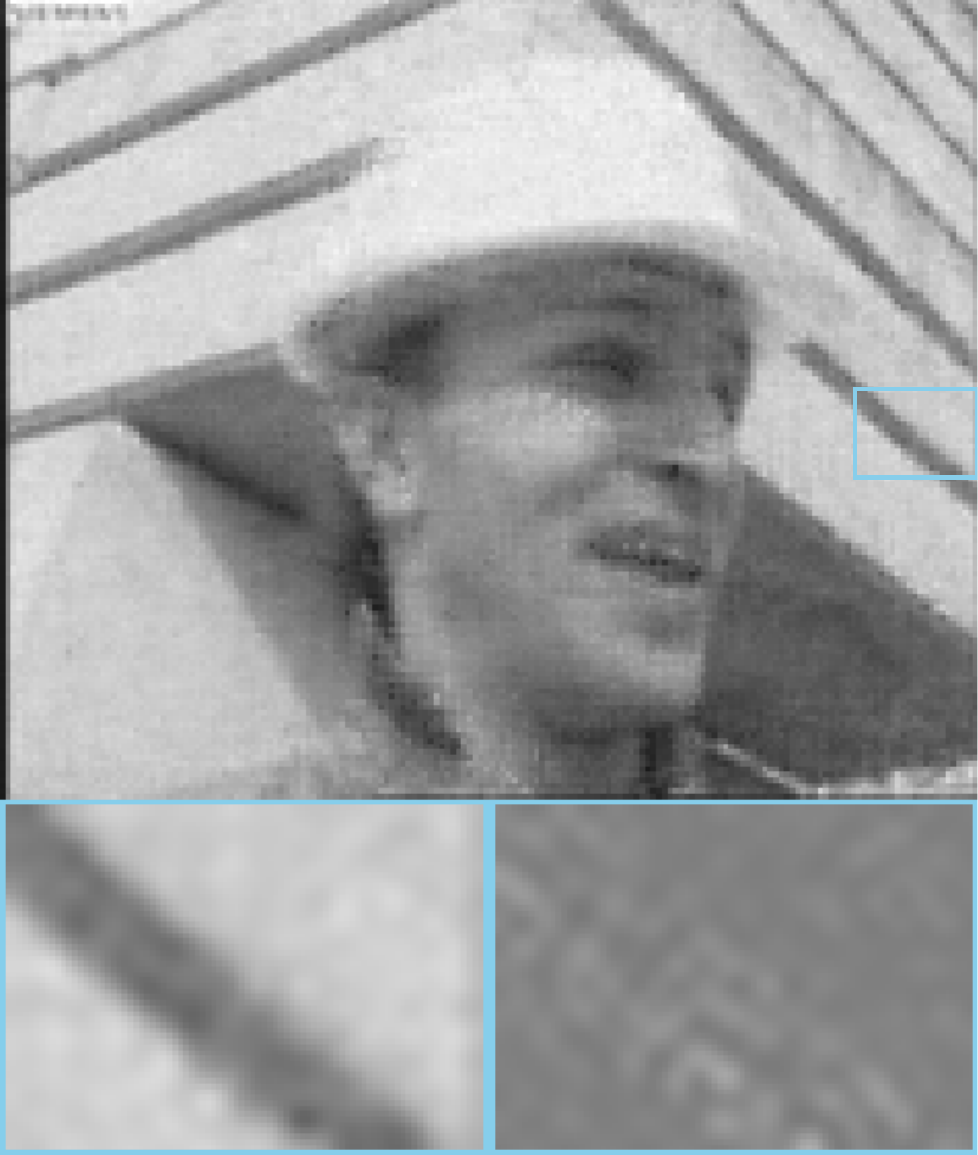}&
			\includegraphics[width=0.12\textwidth]{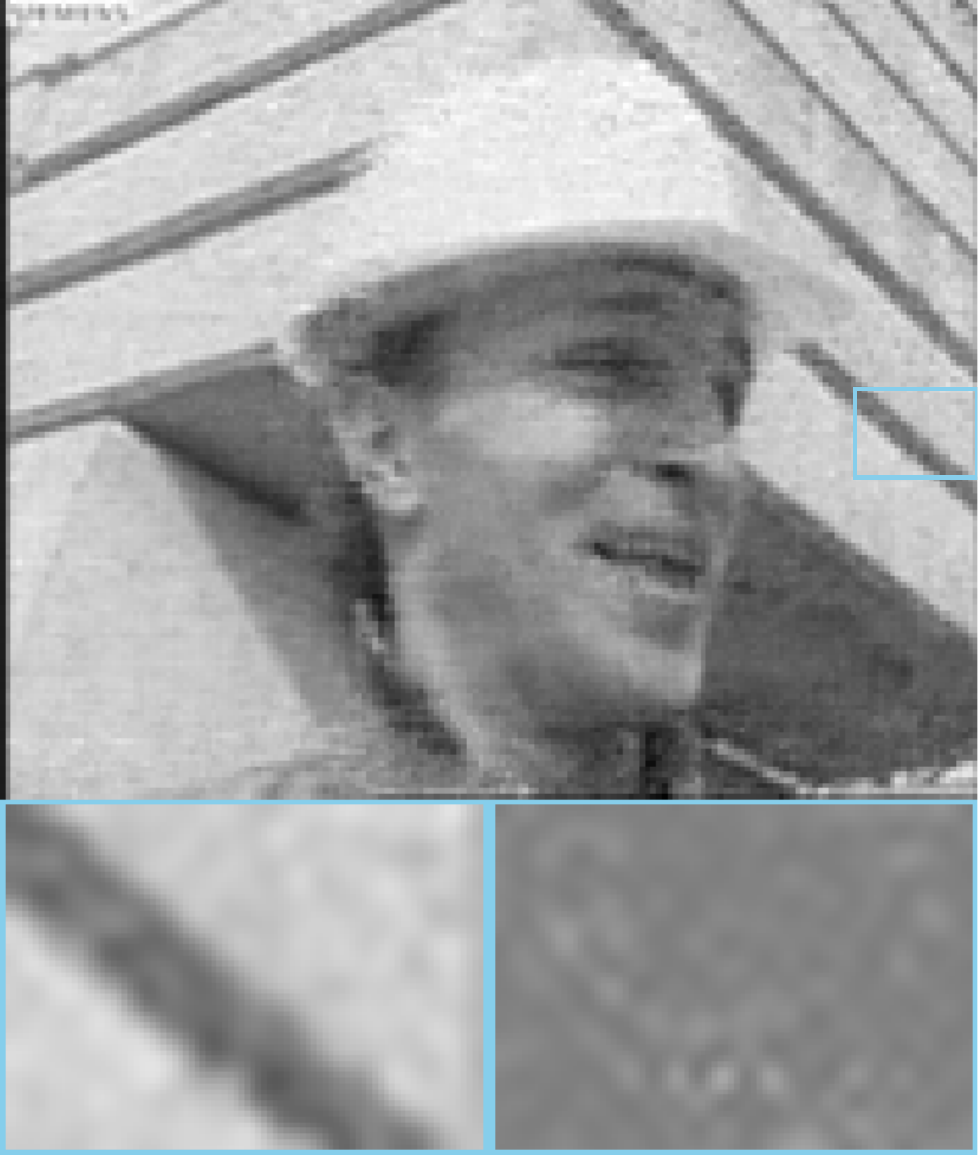}&
			\includegraphics[width=0.12\textwidth]{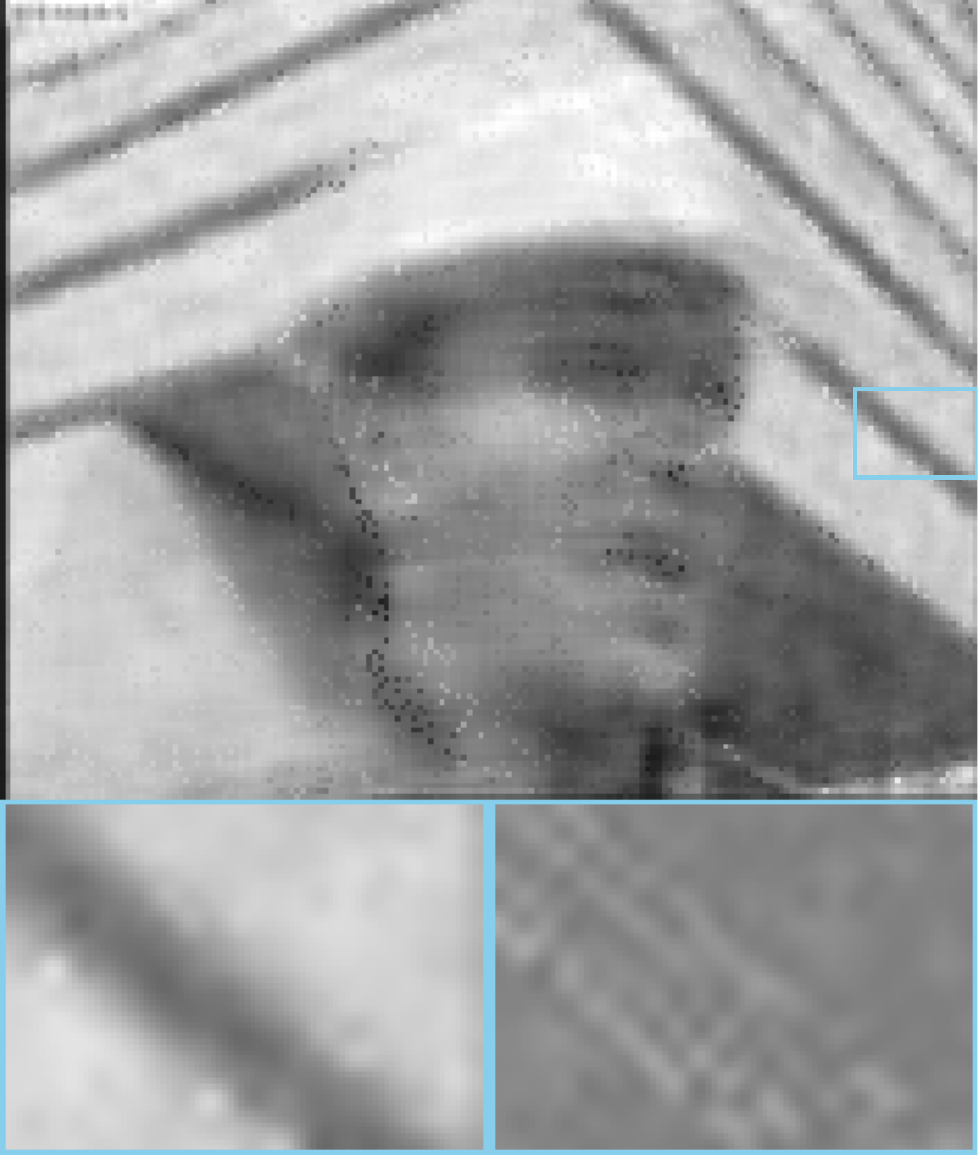}&
			\includegraphics[width=0.12\textwidth]{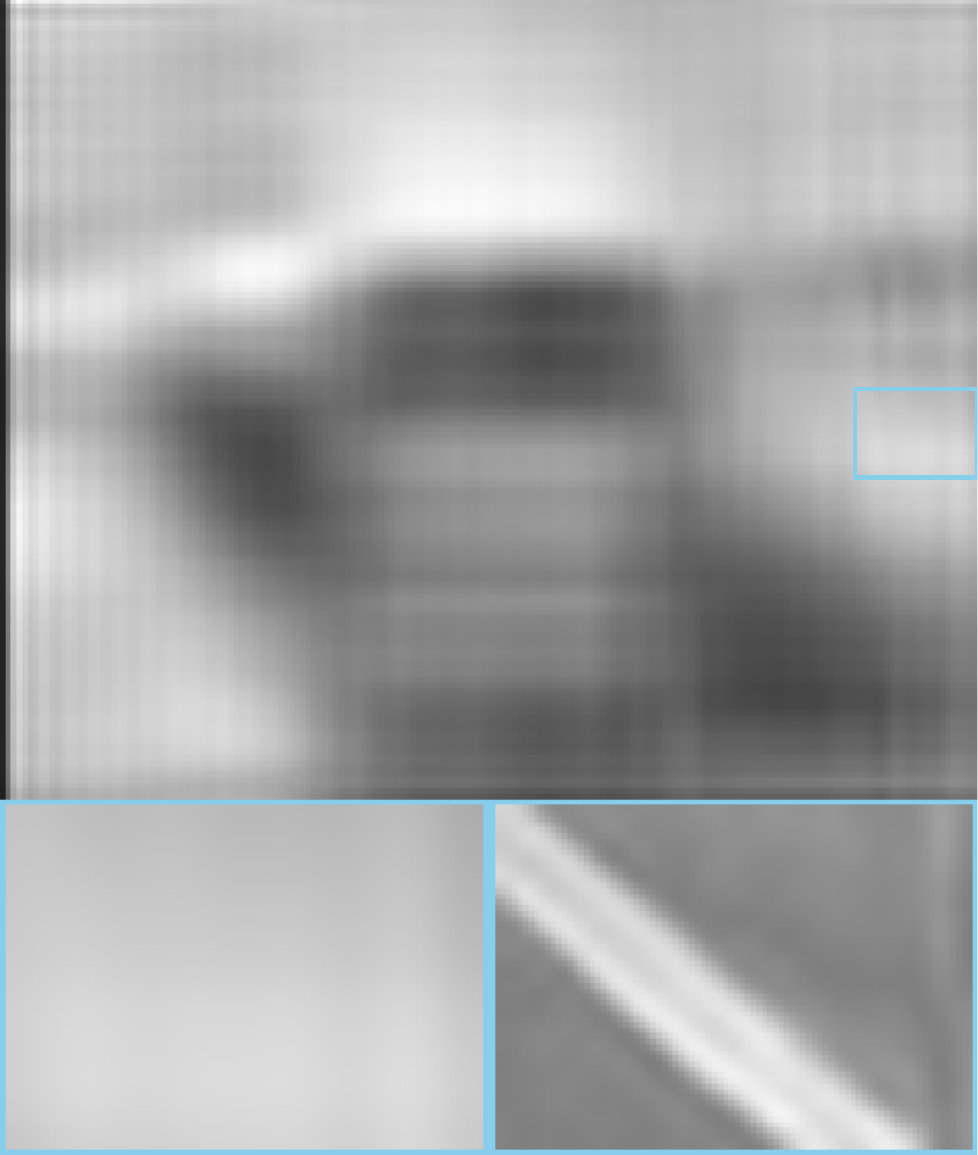}&
			\includegraphics[width=0.12\textwidth]{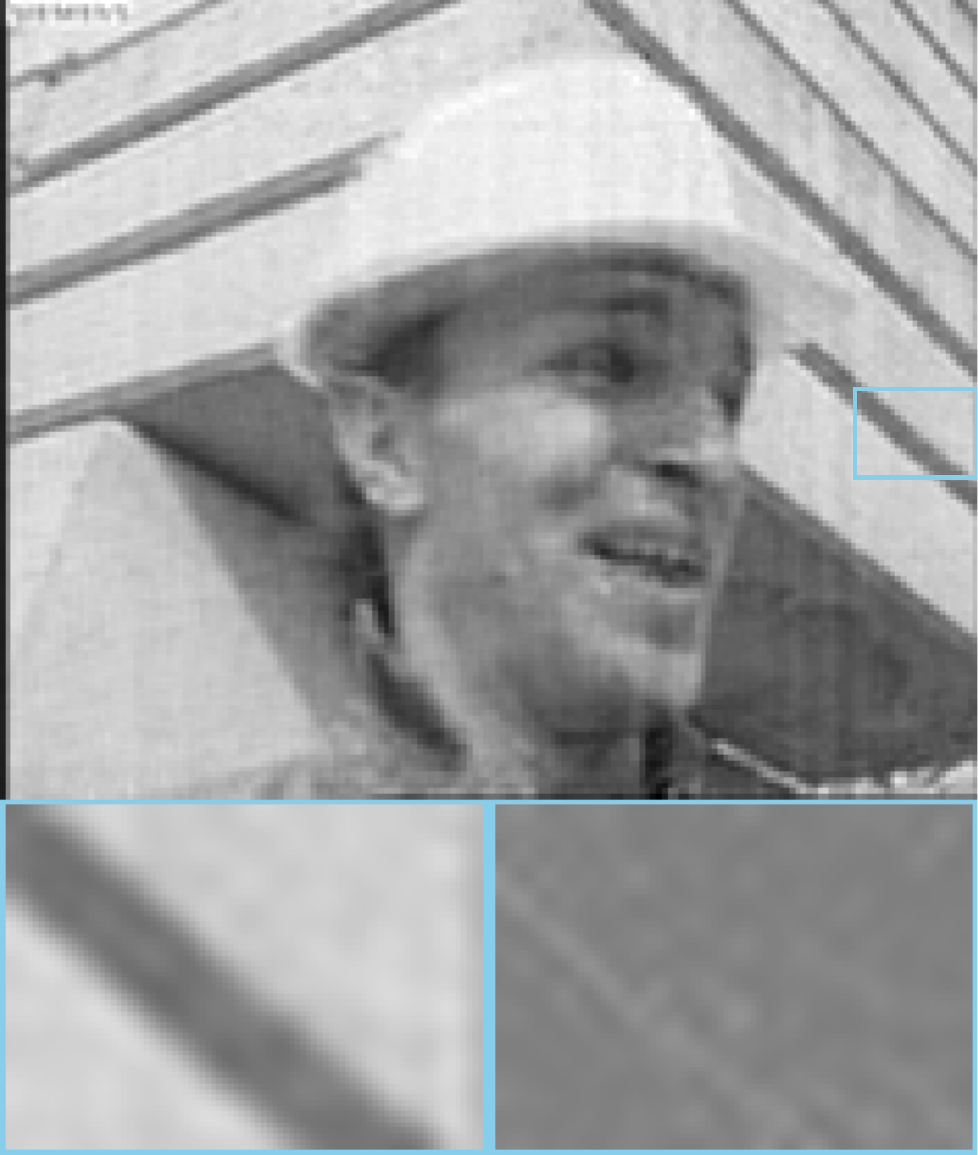}&
			\includegraphics[width=0.12\textwidth]{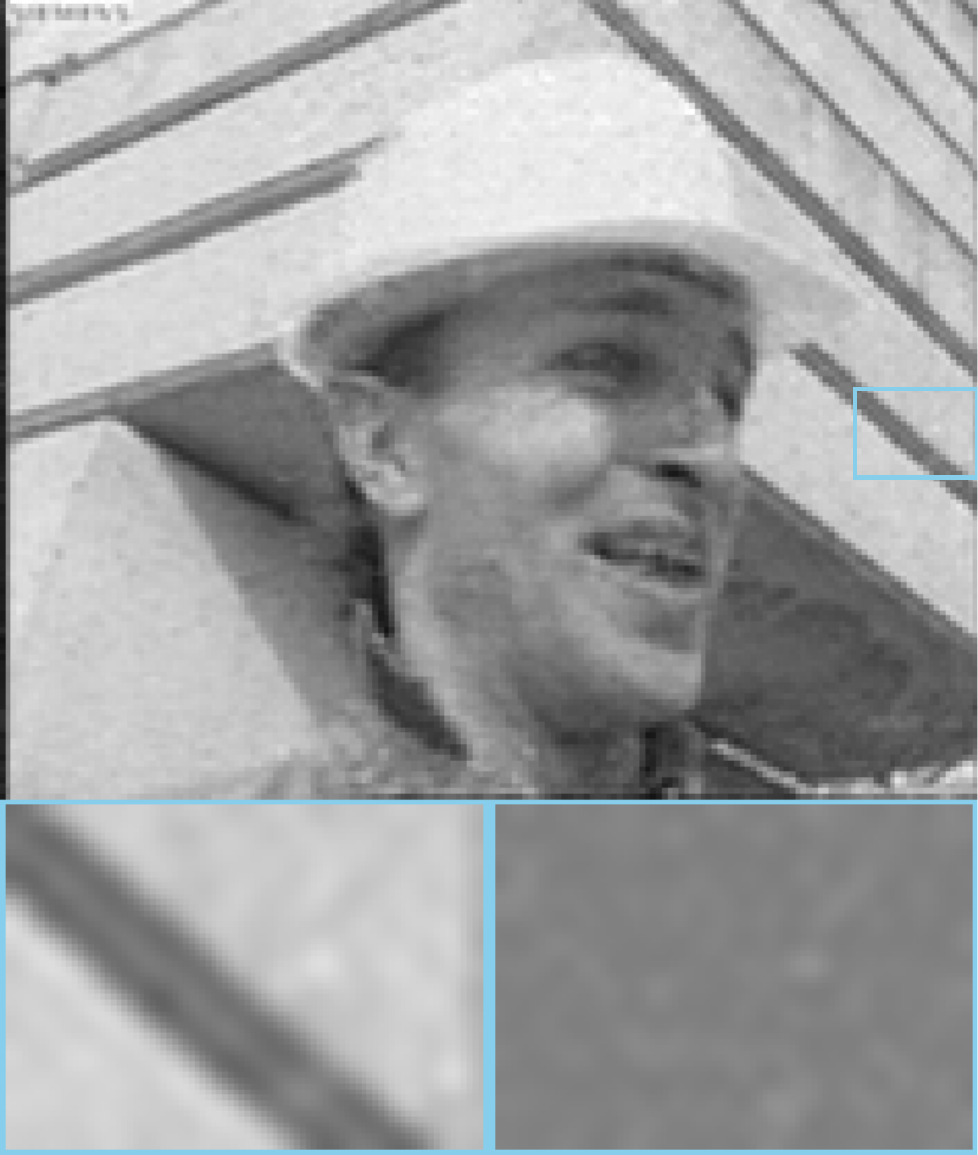}&
			\includegraphics[width=0.12\textwidth]{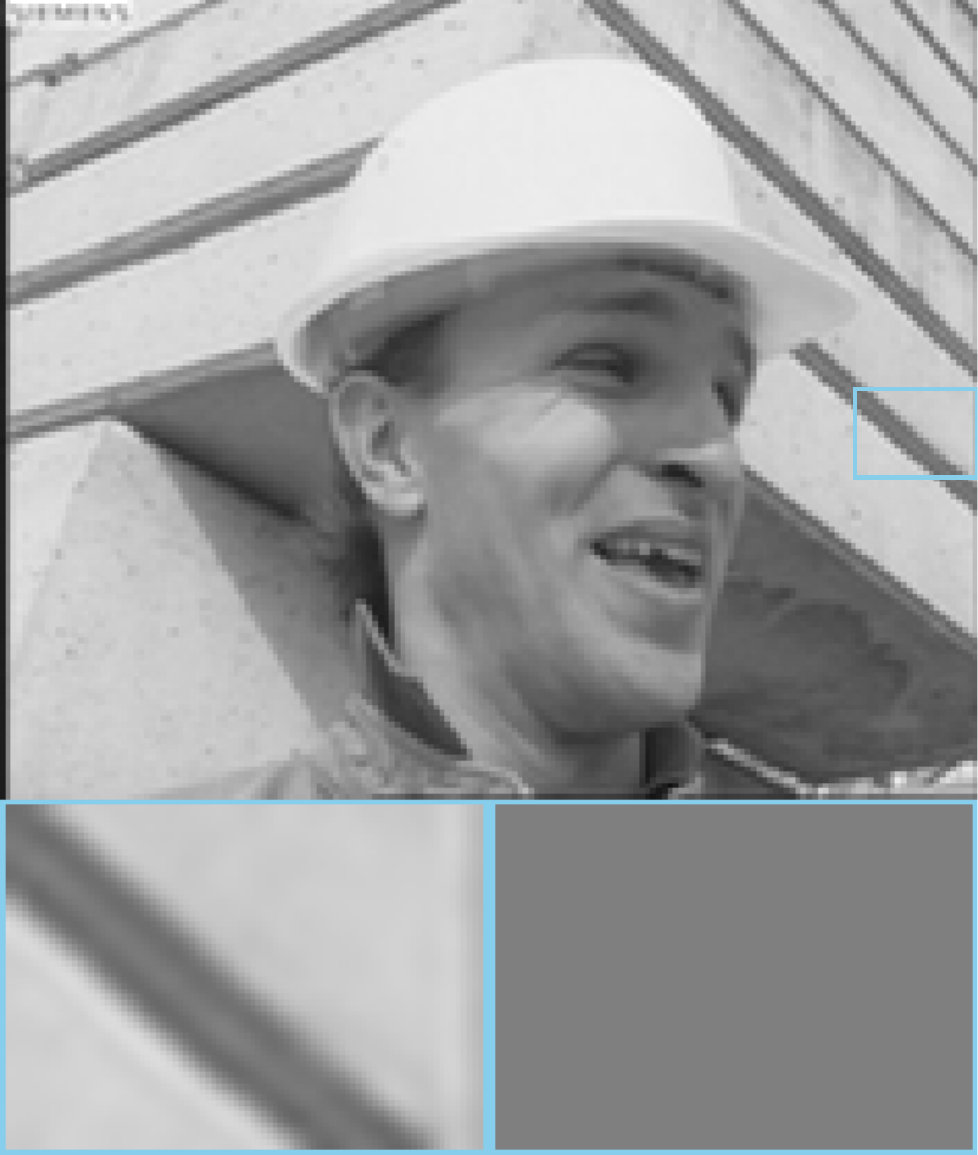}\\[1pt]
			
			\includegraphics[width=0.12\textwidth]{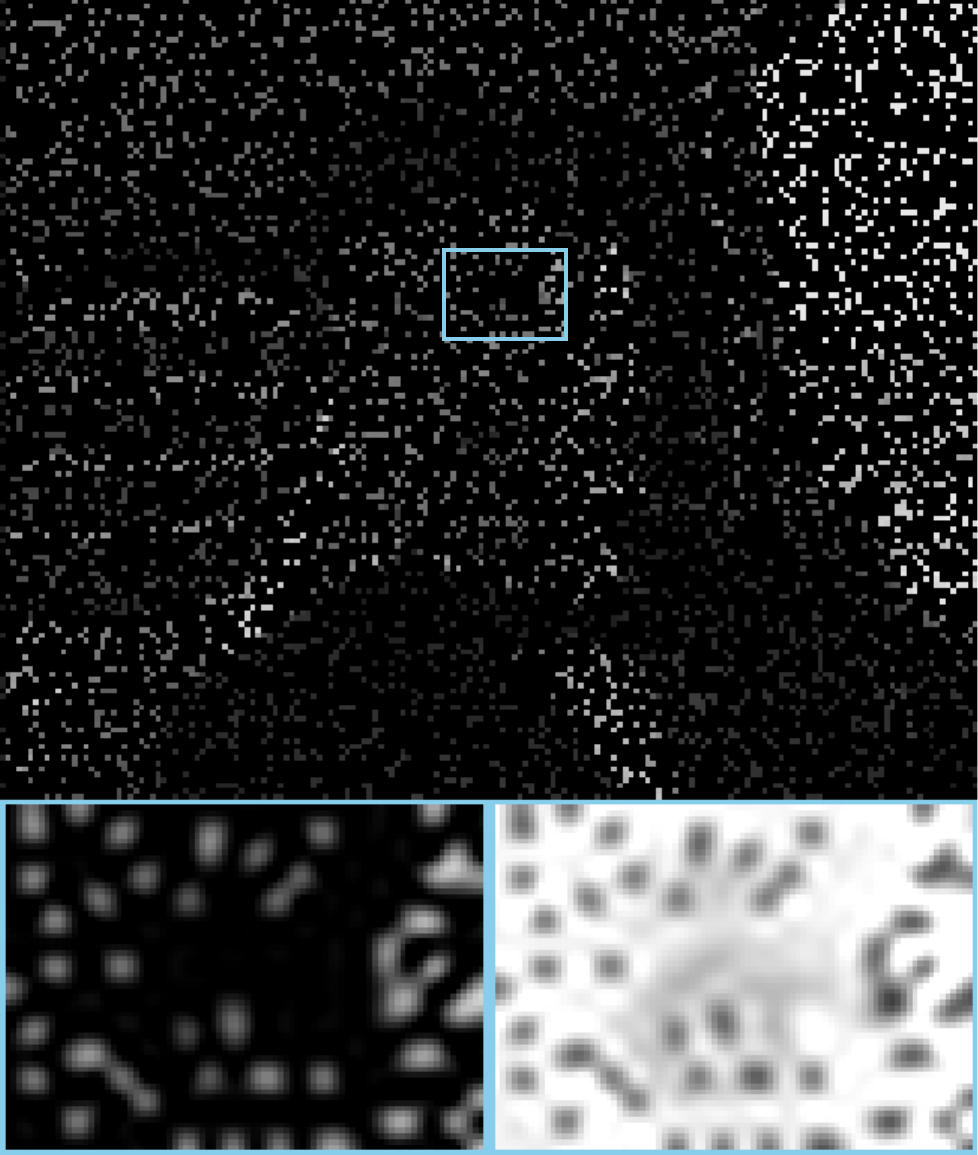}&
			\includegraphics[width=0.12\textwidth]{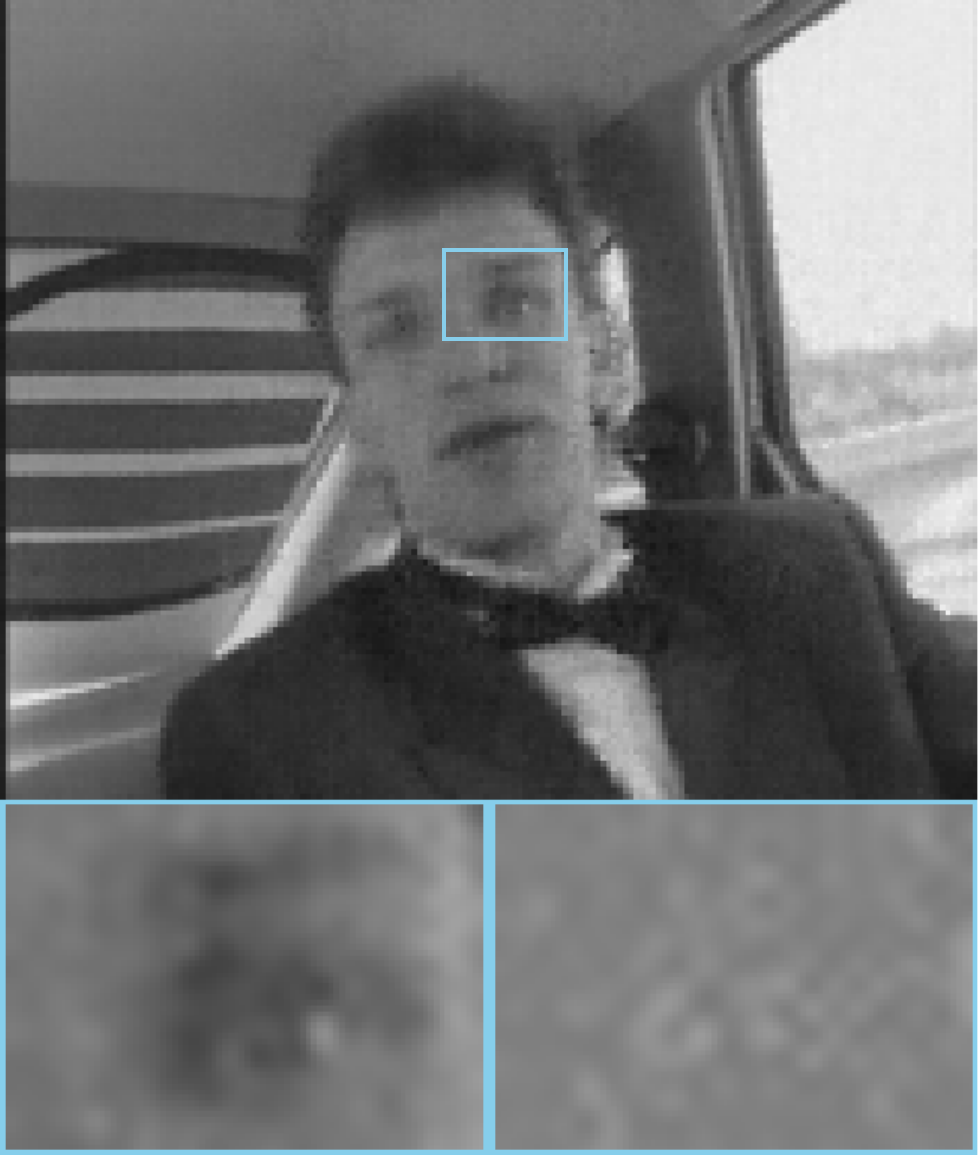}&
			\includegraphics[width=0.12\textwidth]{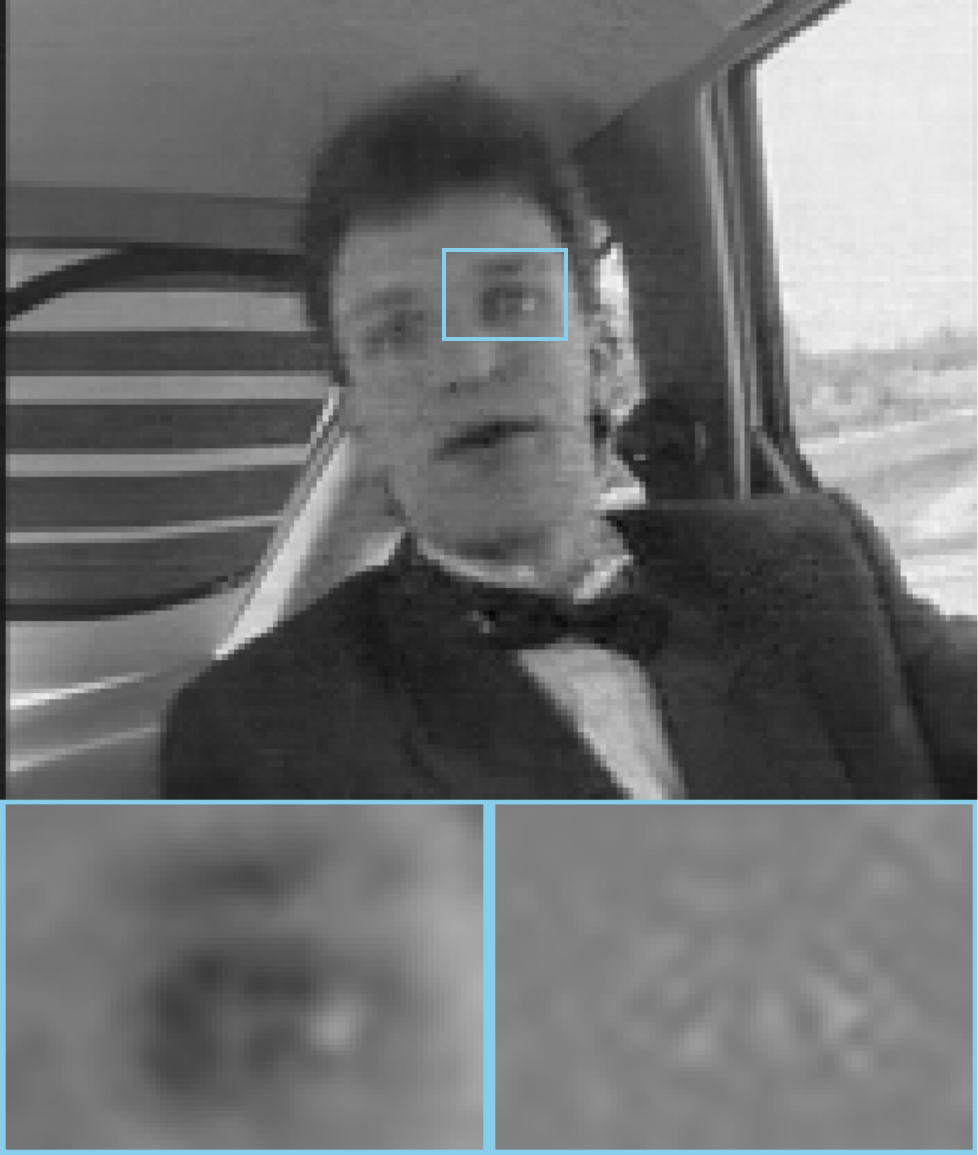}&
			\includegraphics[width=0.12\textwidth]{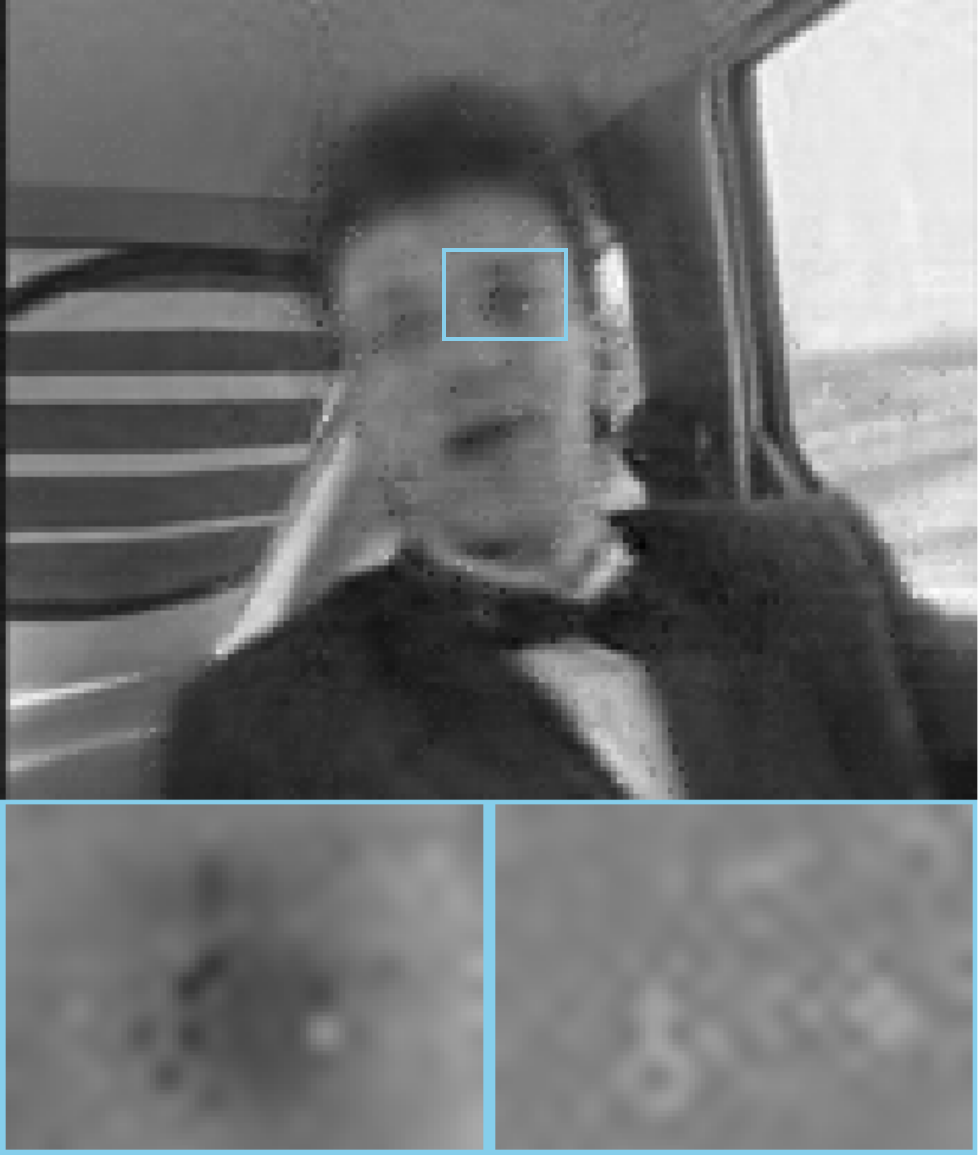}&
			\includegraphics[width=0.12\textwidth]{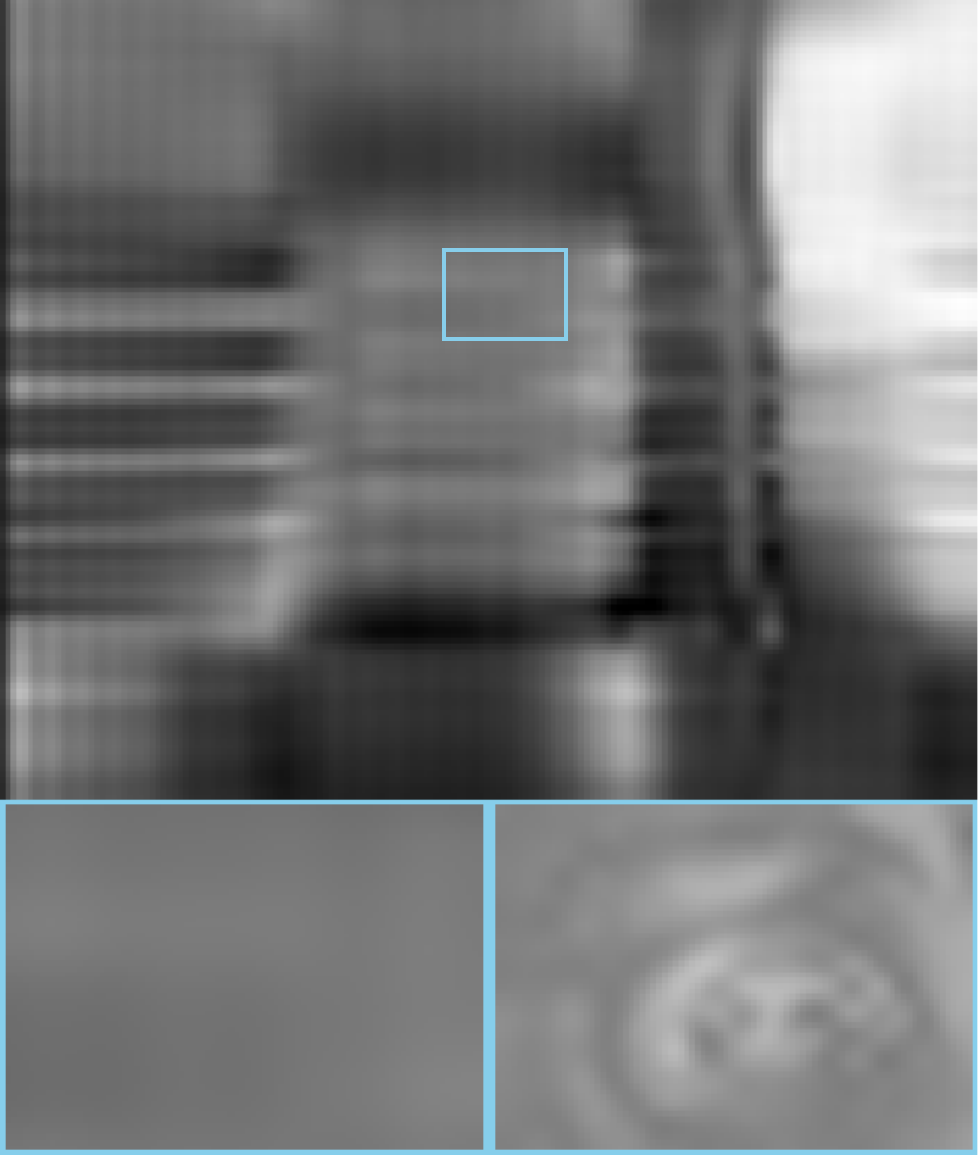}&
			\includegraphics[width=0.12\textwidth]{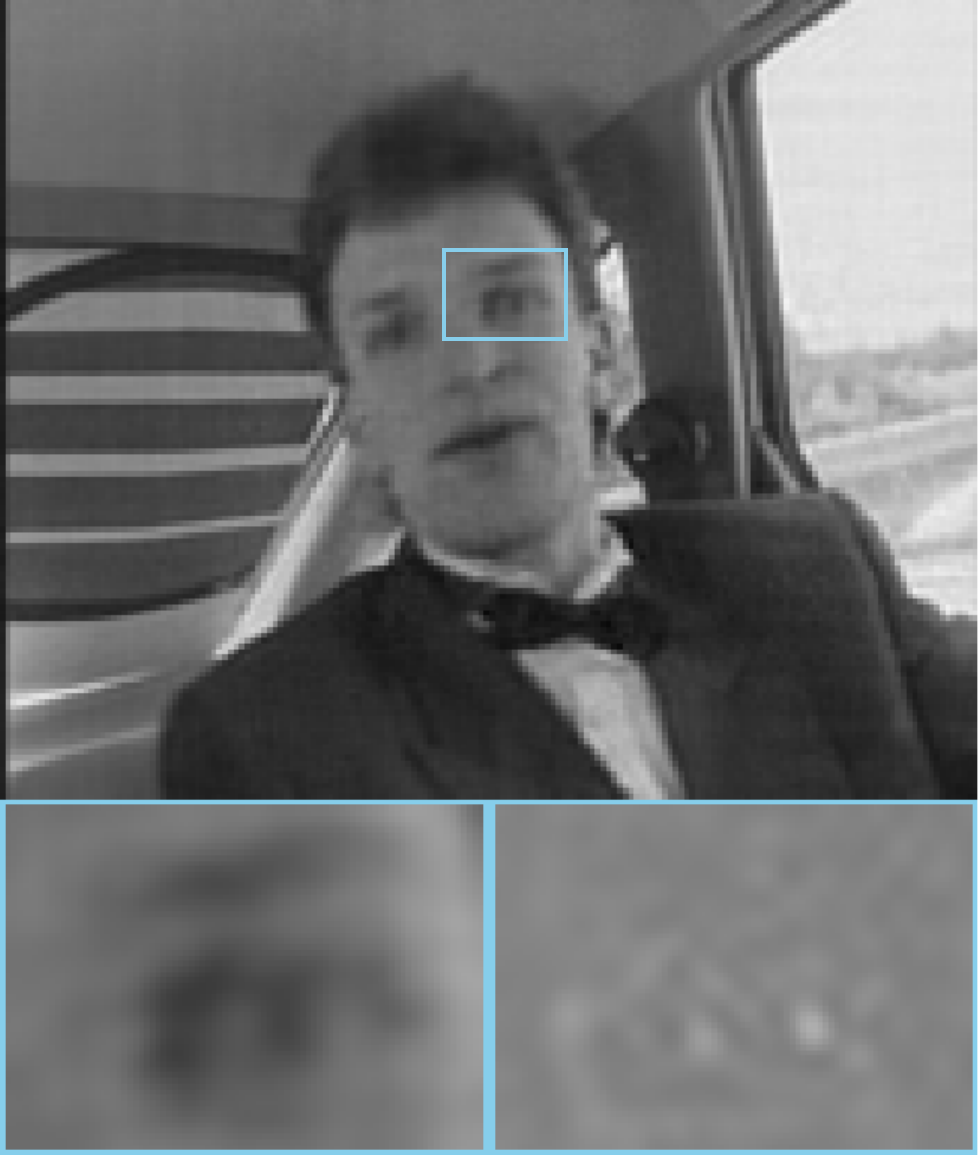}&
			\includegraphics[width=0.12\textwidth]{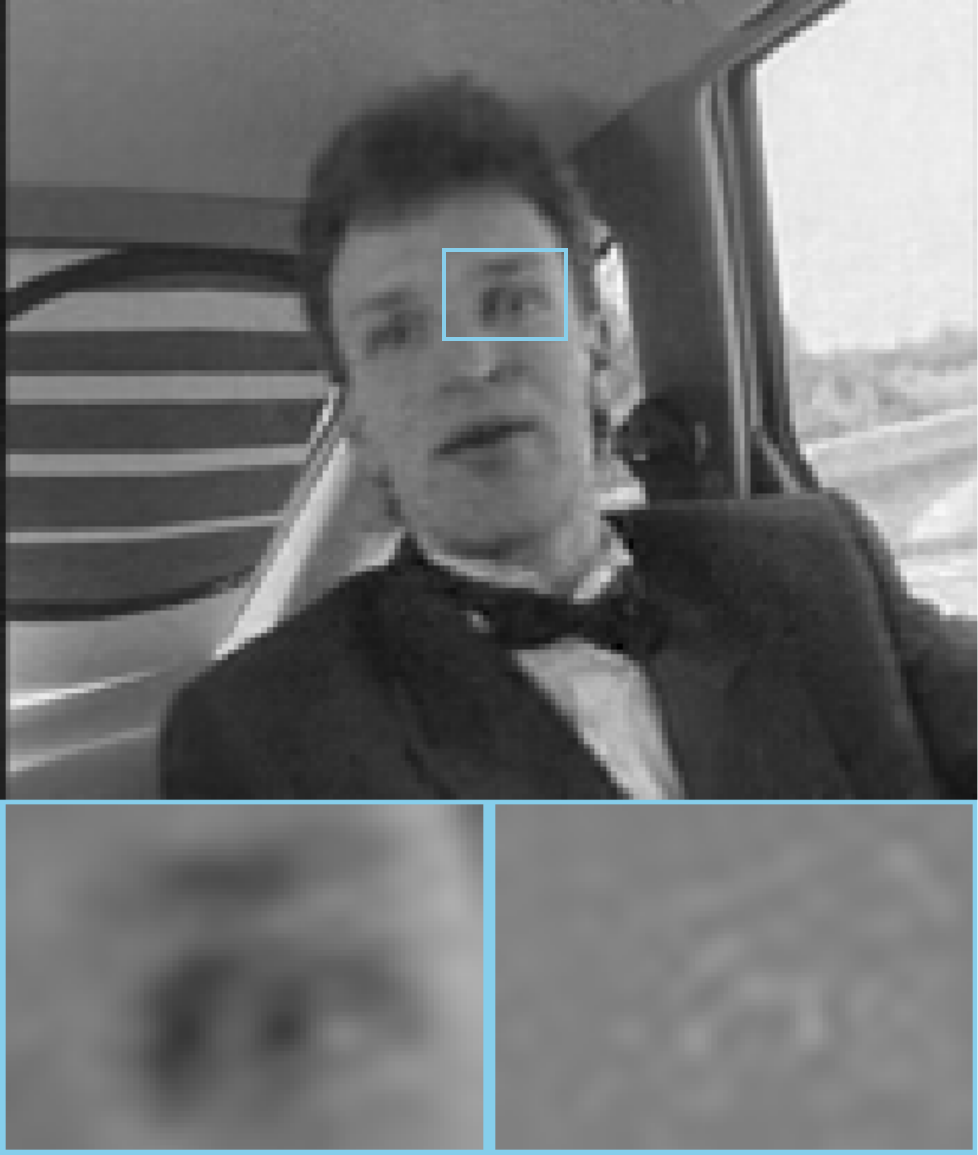}&
			\includegraphics[width=0.12\textwidth]{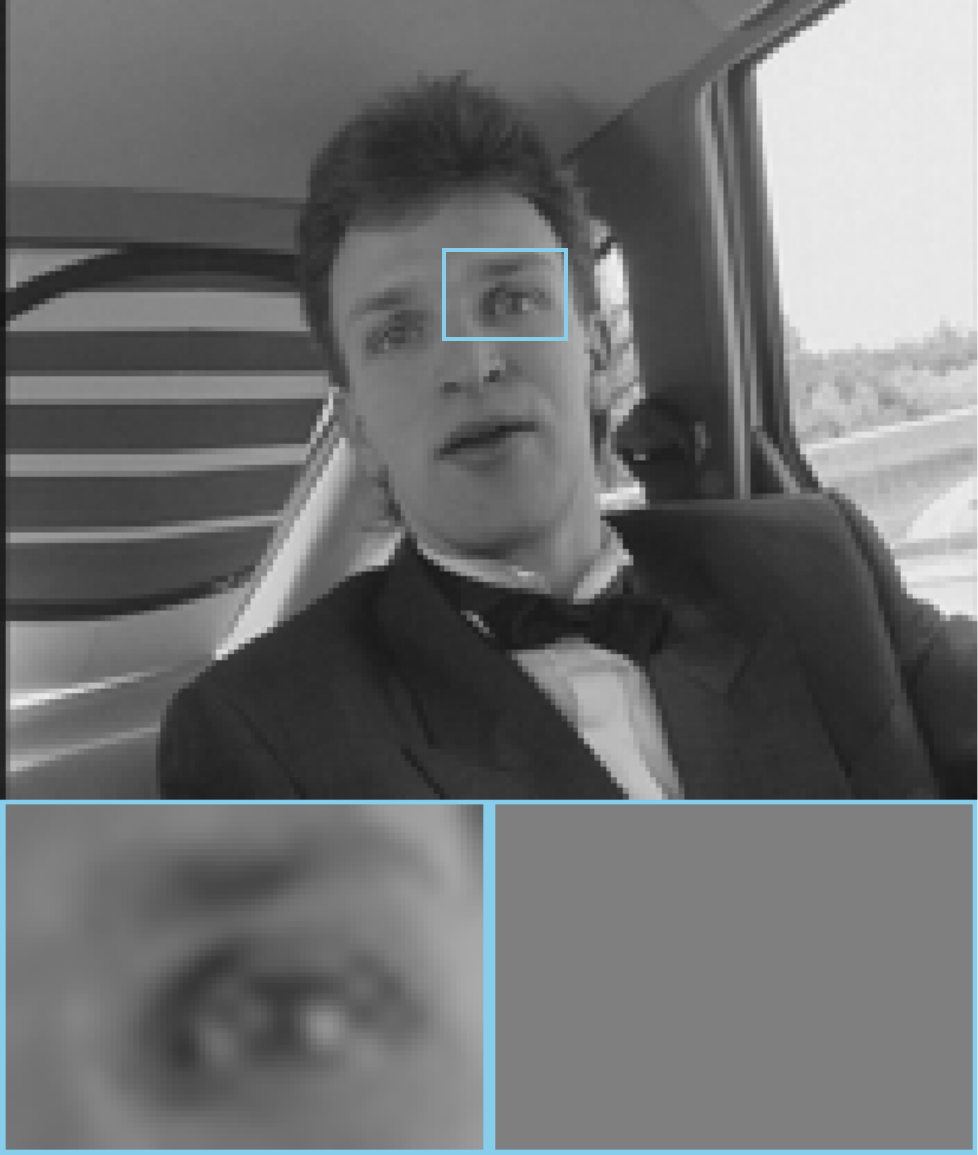}\\[1pt]
			
			\includegraphics[width=0.12\textwidth]{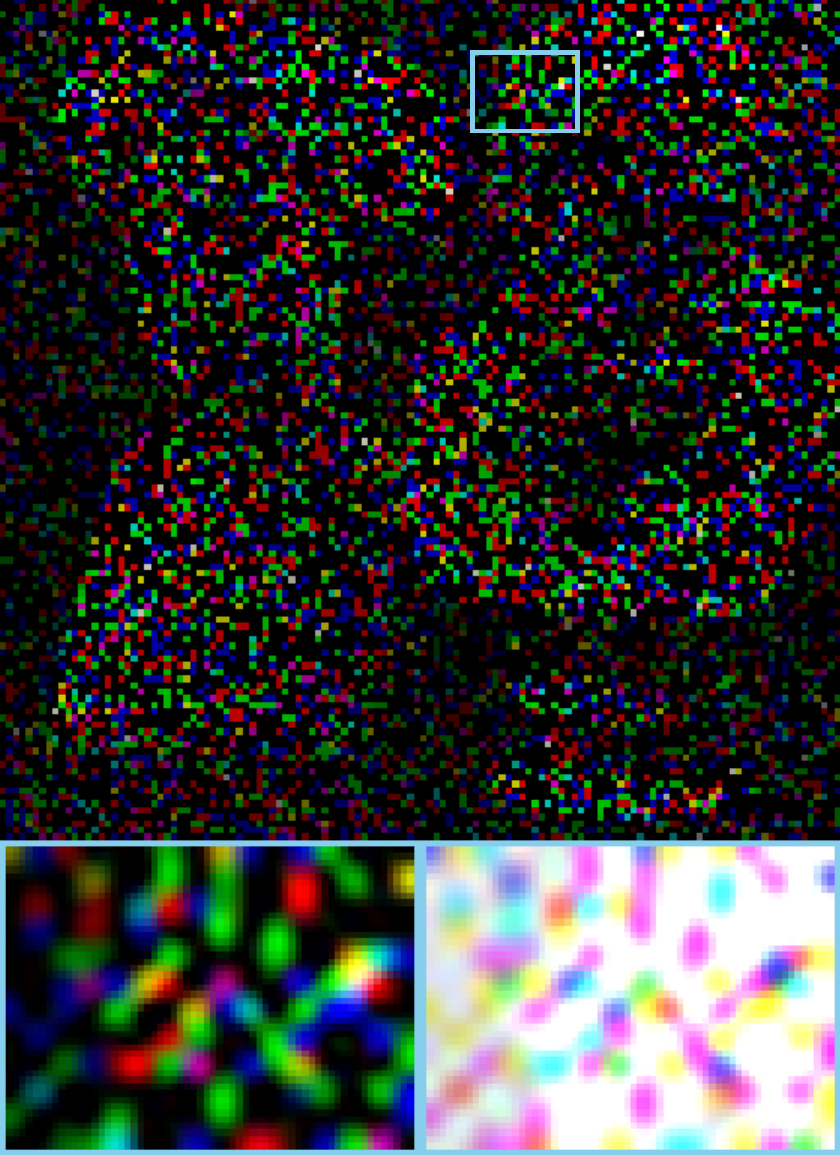}&
			\includegraphics[width=0.12\textwidth]{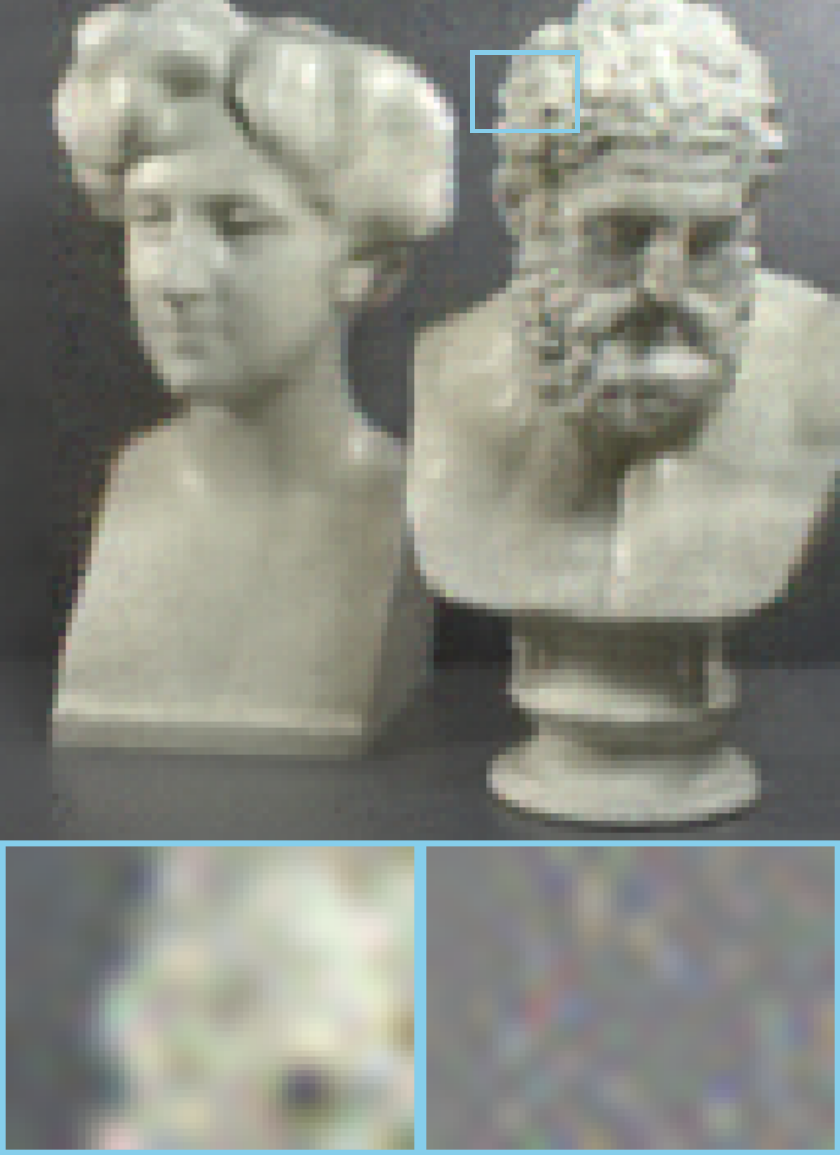}&
			\includegraphics[width=0.12\textwidth]{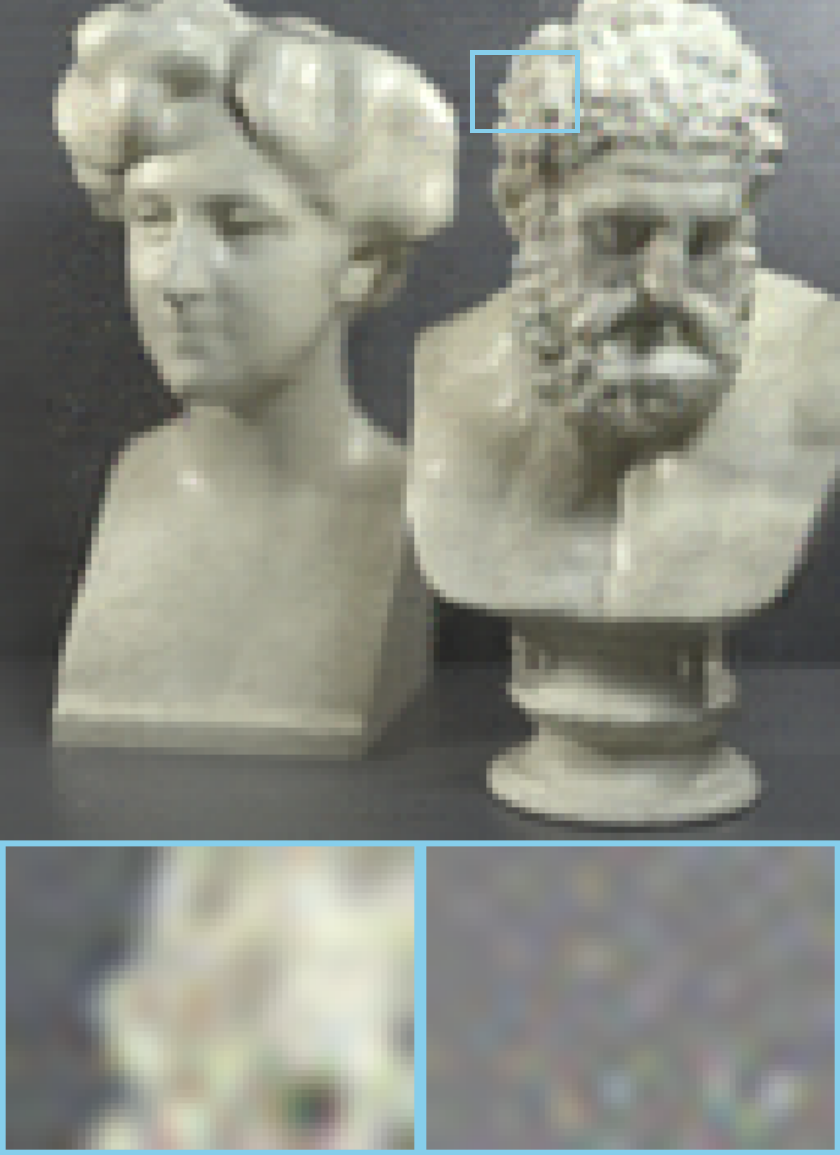}&
			\includegraphics[width=0.12\textwidth]{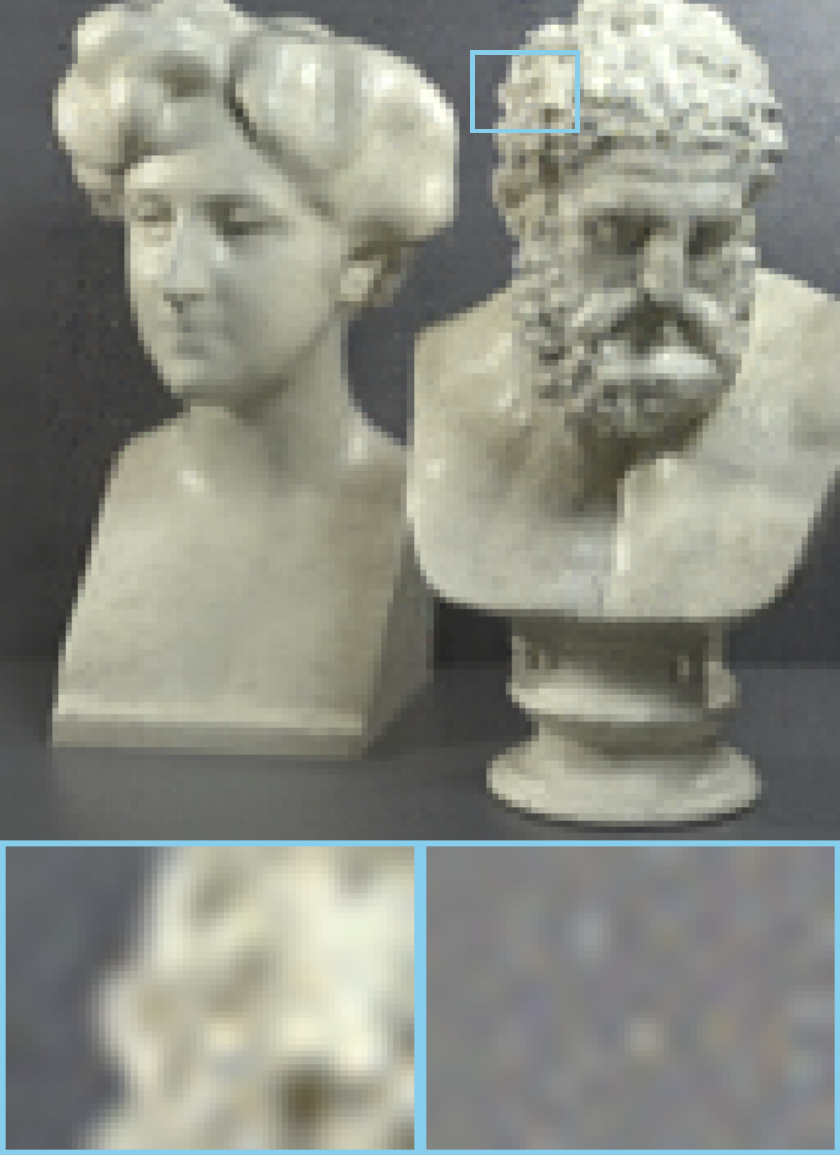}&
			\includegraphics[width=0.12\textwidth]{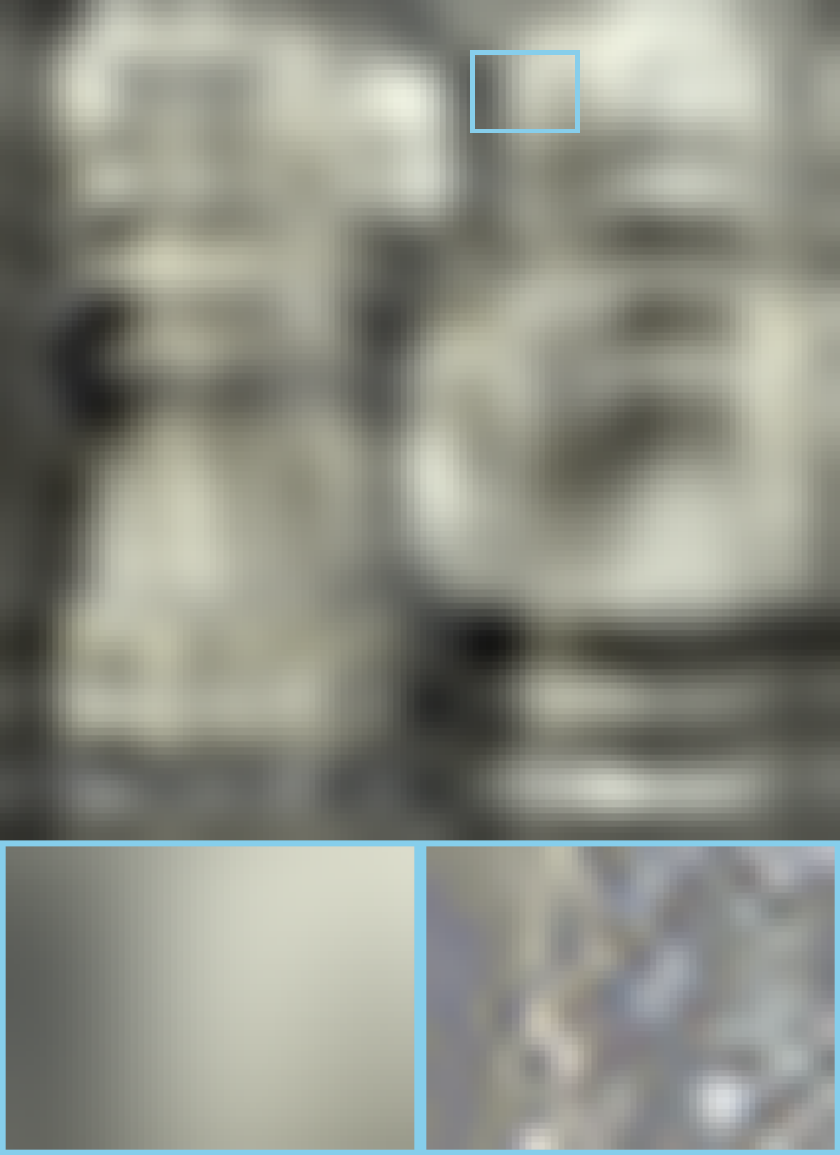}&
			\includegraphics[width=0.12\textwidth]{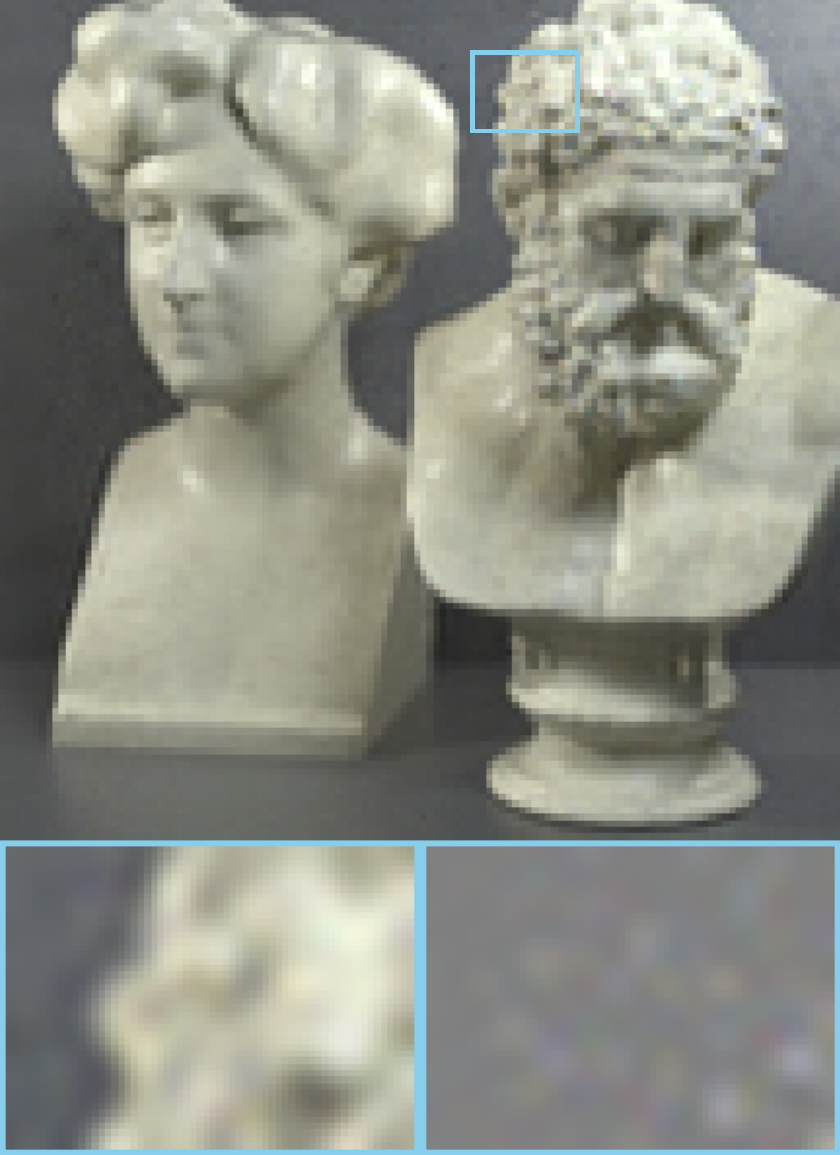}&
			\includegraphics[width=0.12\textwidth]{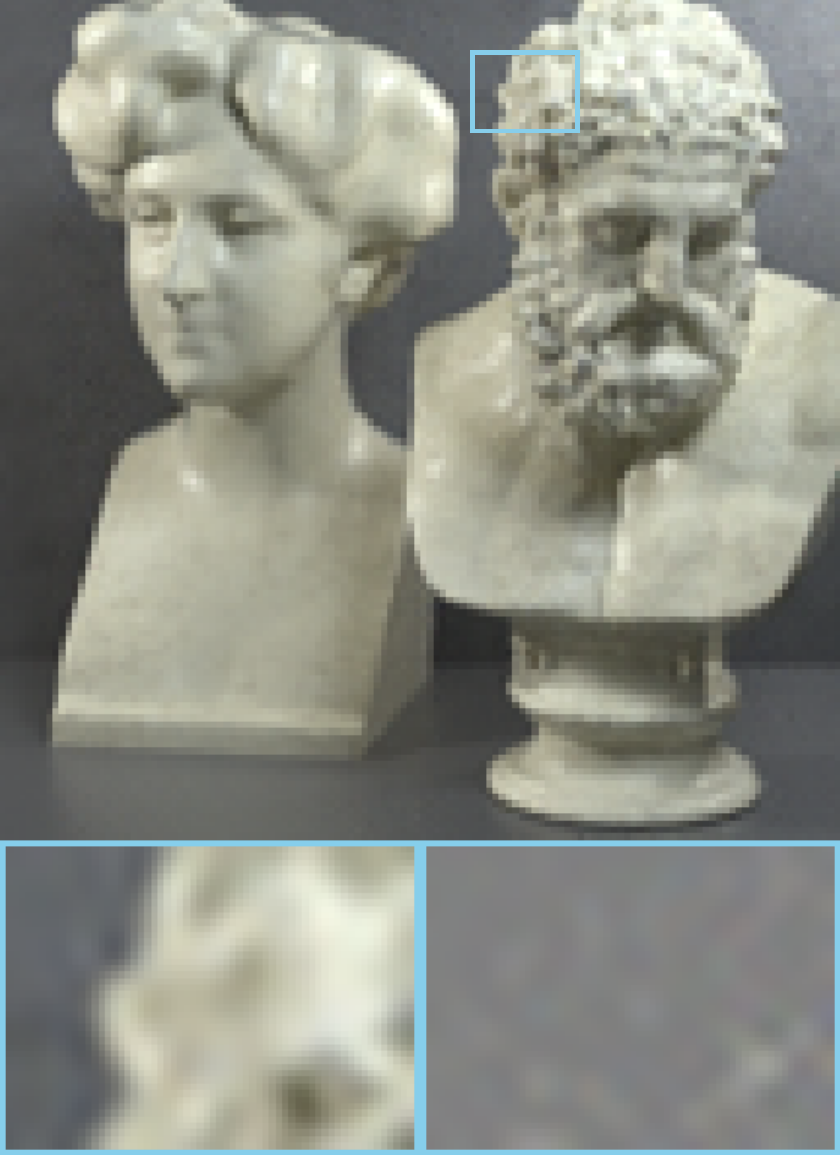}&
			\includegraphics[width=0.12\textwidth]{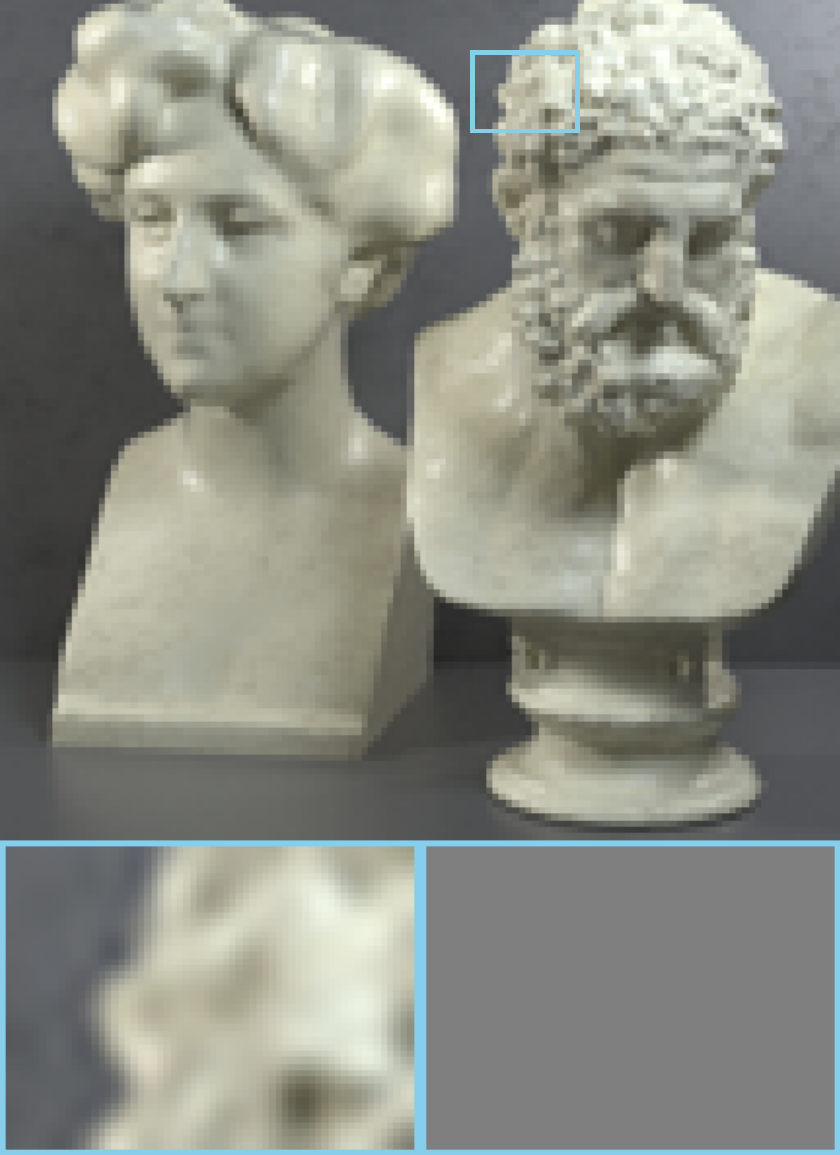}\\[1pt]    
			\includegraphics[width=0.12\textwidth]{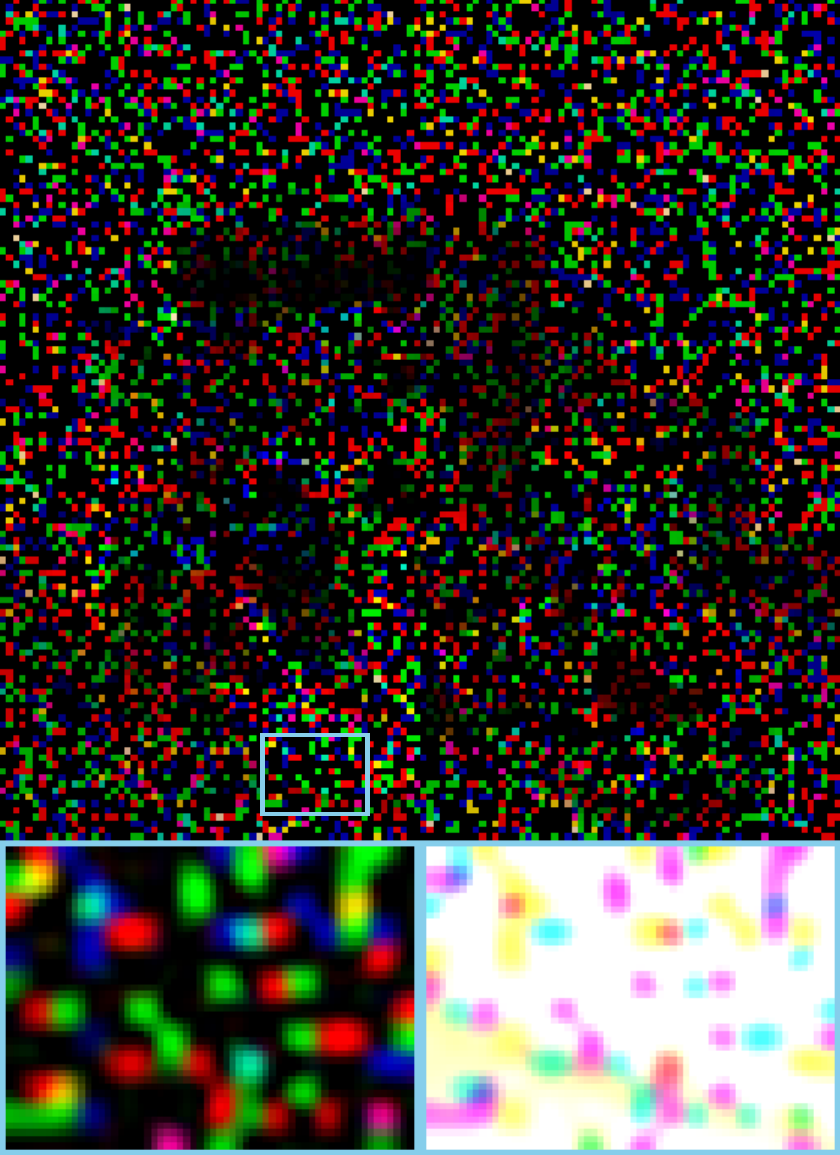}&
			\includegraphics[width=0.12\textwidth]{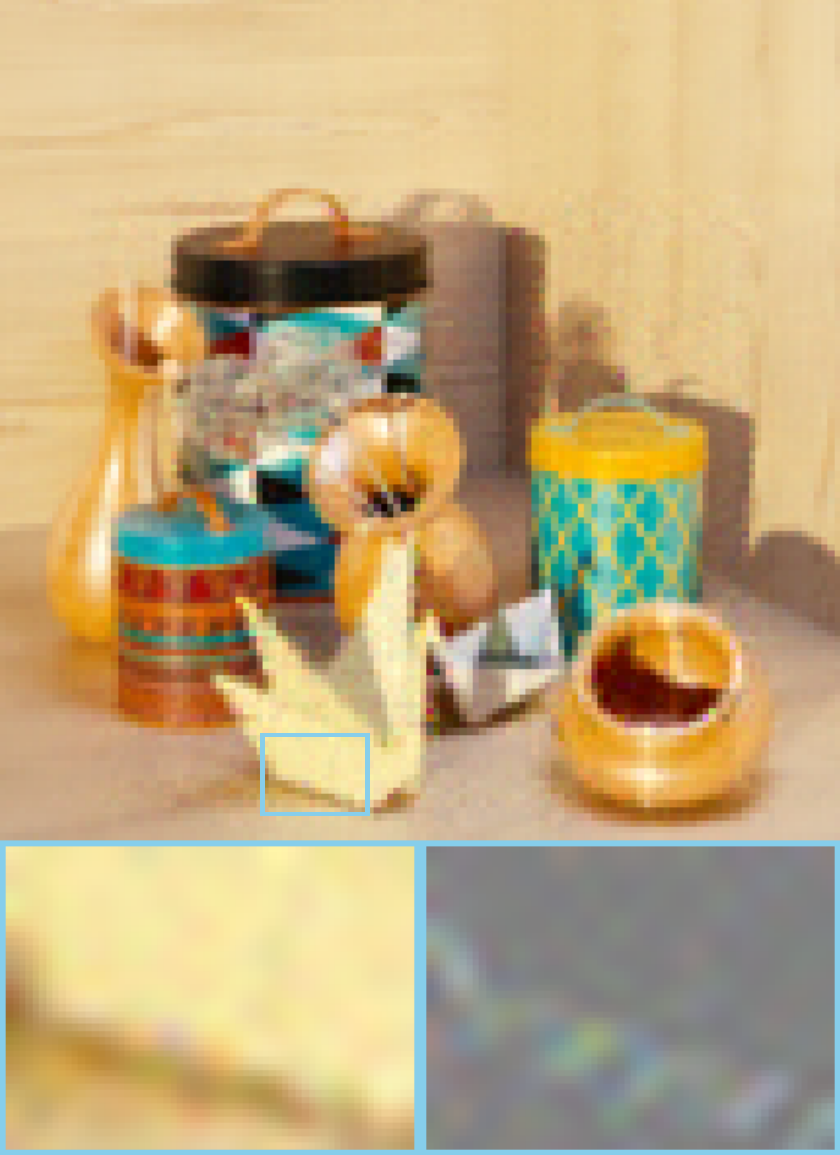}&
			\includegraphics[width=0.12\textwidth]{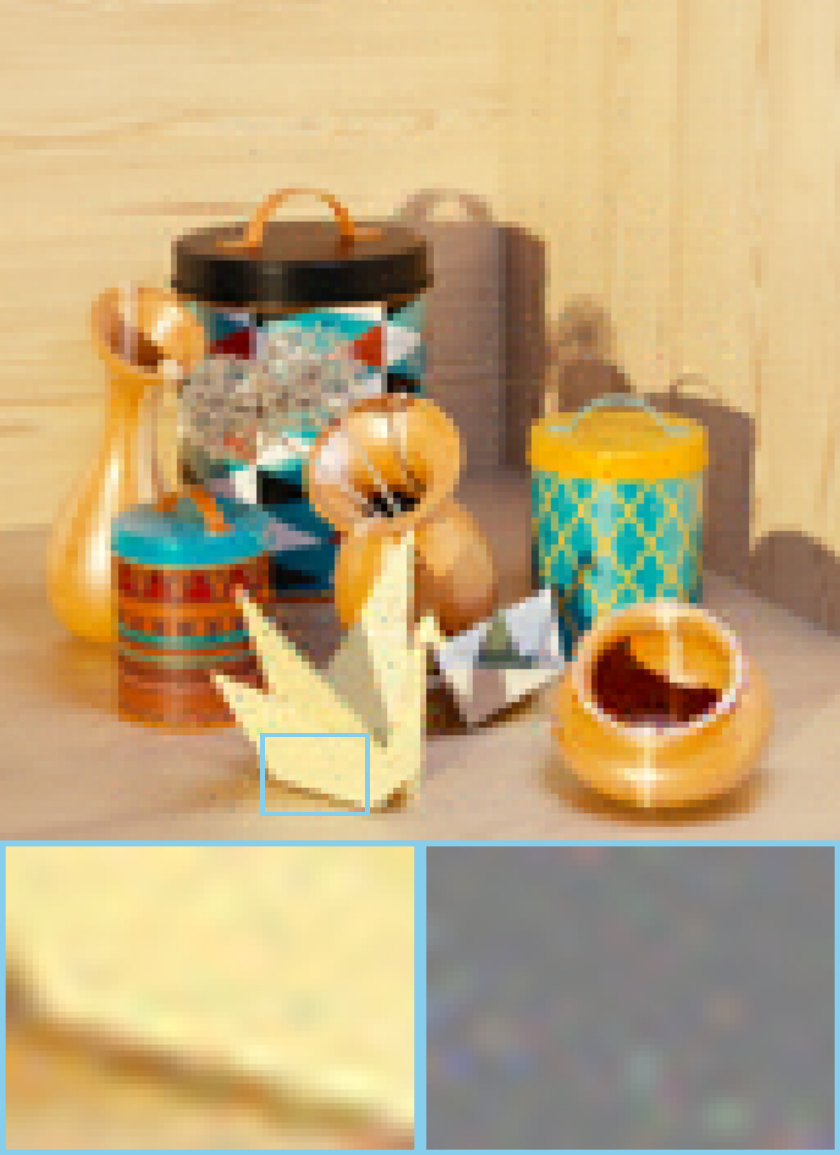}&
			\includegraphics[width=0.12\textwidth]{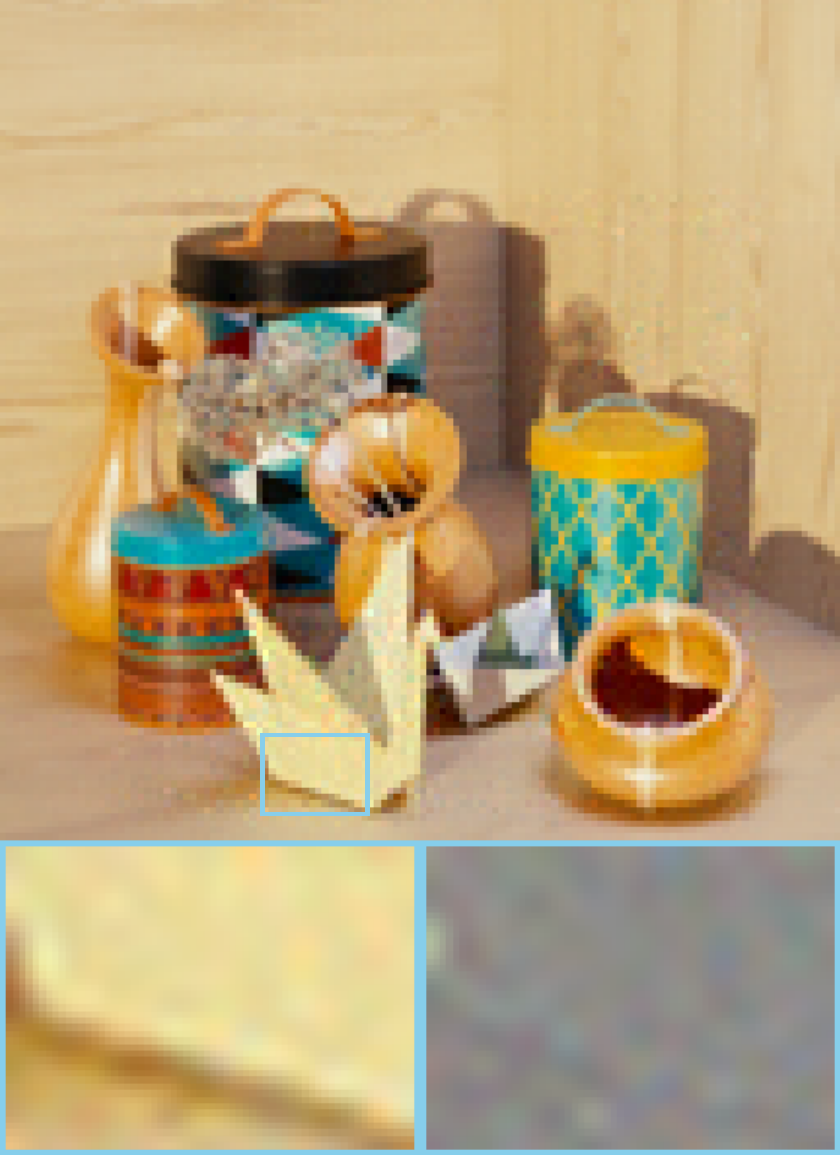}&
			\includegraphics[width=0.12\textwidth]{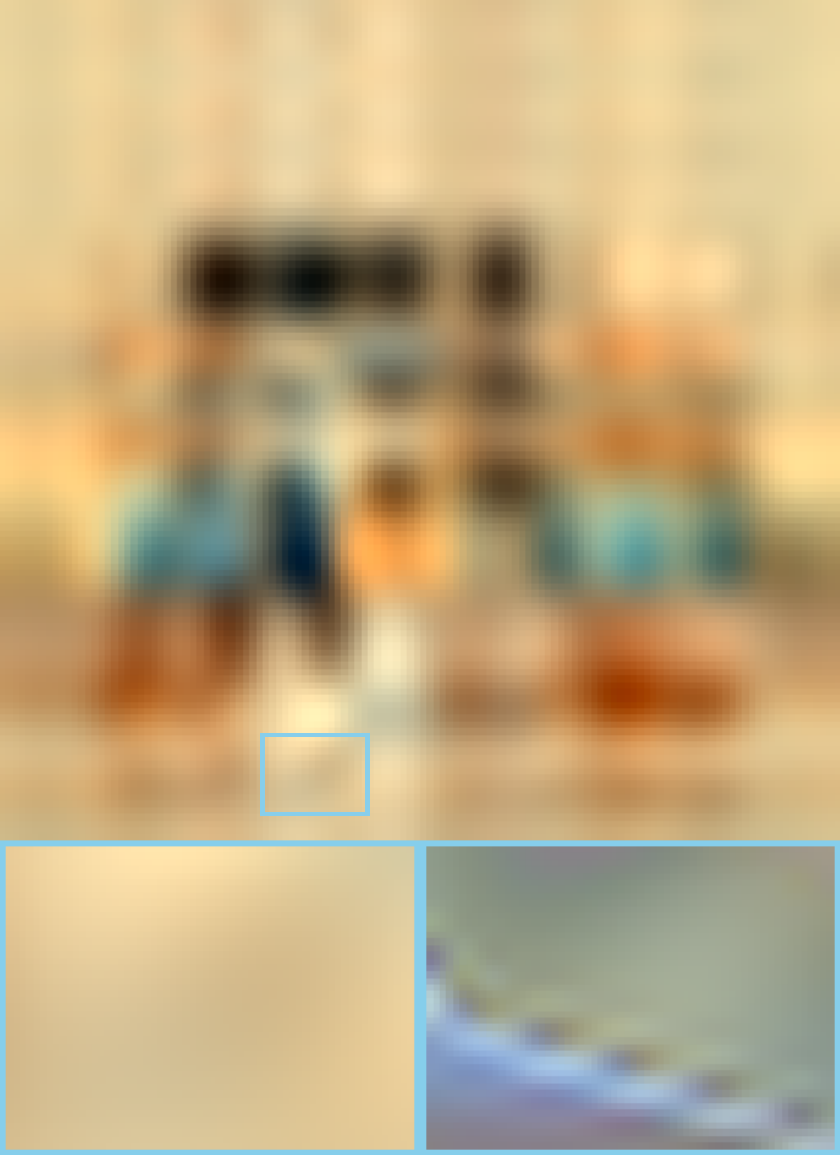}&
			\includegraphics[width=0.12\textwidth]{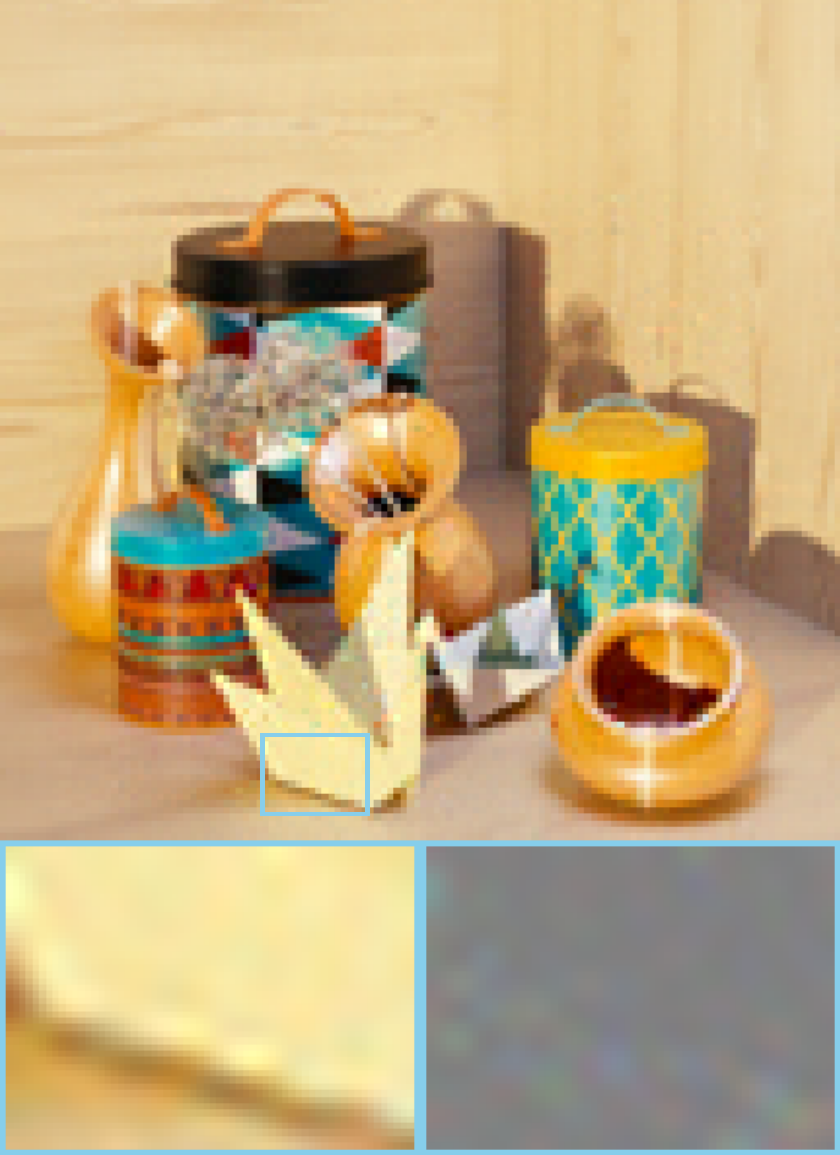}&
			\includegraphics[width=0.12\textwidth]{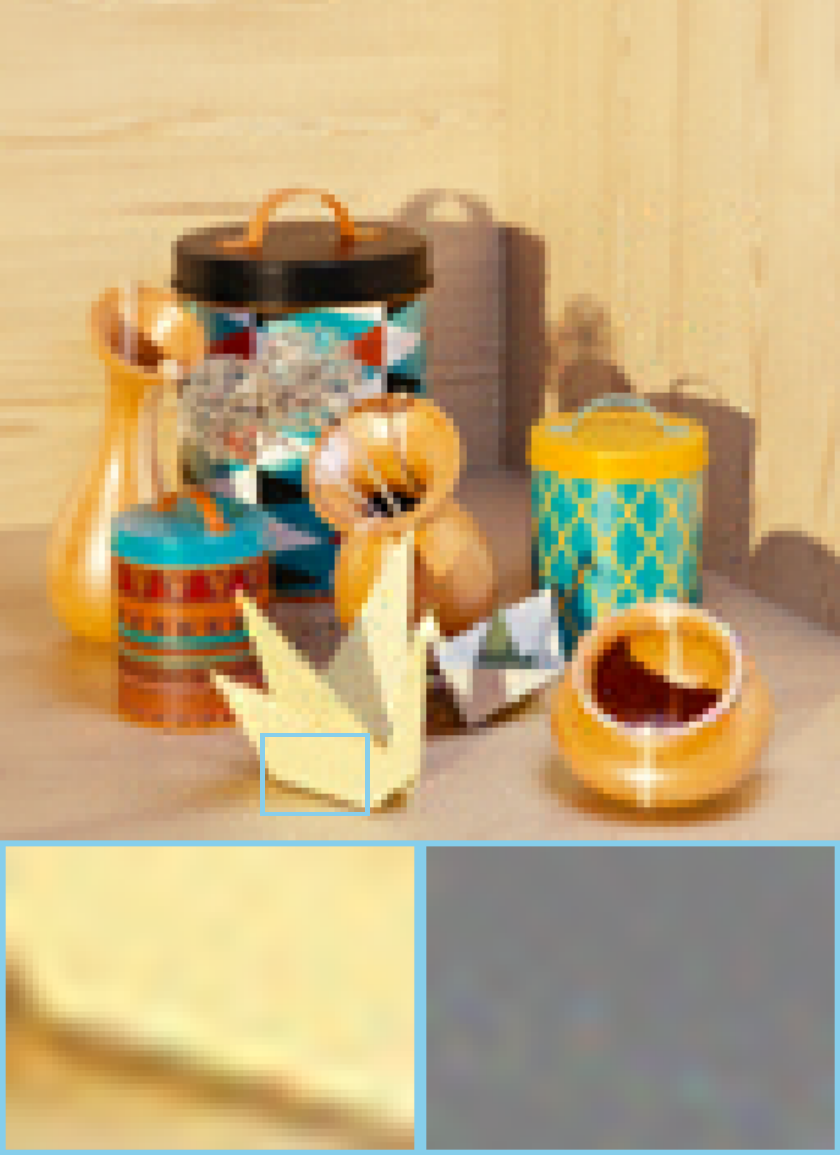}&
			\includegraphics[width=0.12\textwidth]{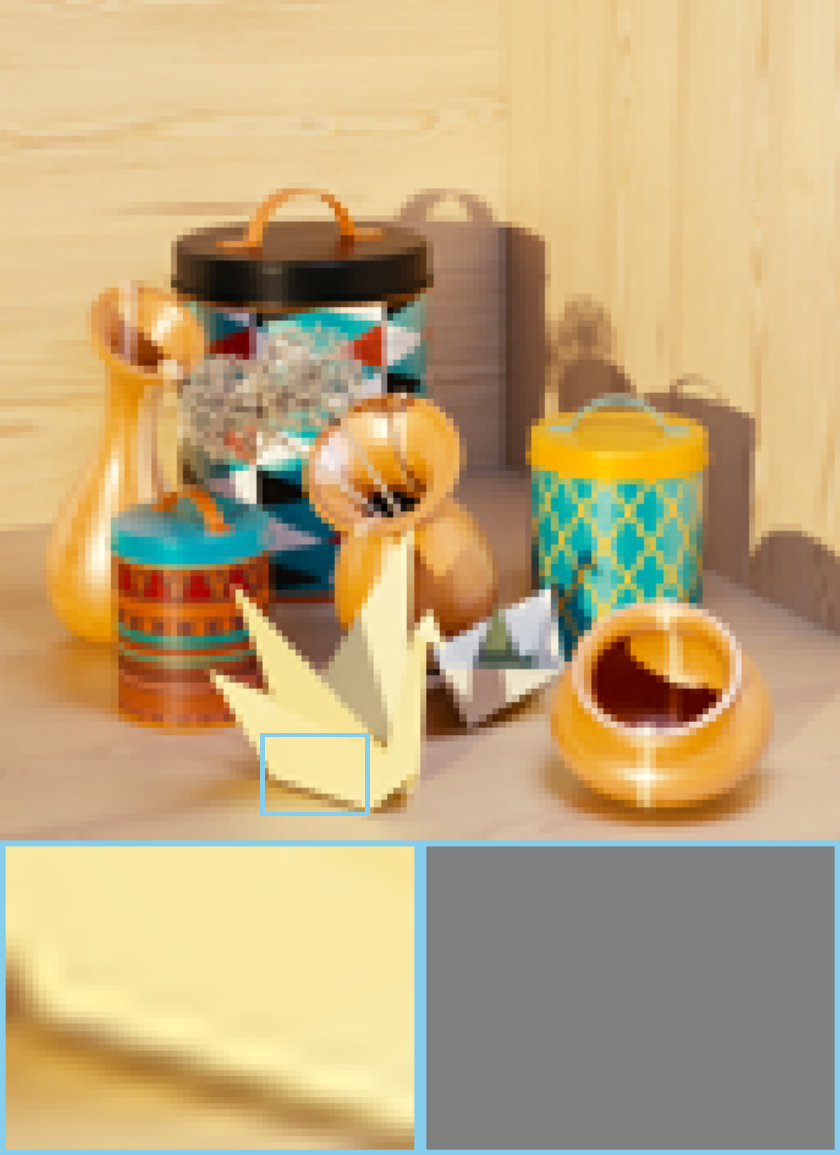}\\[1pt]
			Observed &TRLRF &TNN&FCT&FSA&LRTFR &NeuApprox & Original\\
		\end{tabular}
	\end{center}
	
	\caption{The results of multi-dimensional image inpainting with SR=0.15 by different methods on MSIs (i.e., {\it Toys} and {\it Painting}), videos (i.e., {\it Foreman} and {\it Carphone}),  and light field data (i.e., {\it Greek} and {\it Origami}).\label{fig_completion_image}}
\end{figure*}

\begin{table*}[h]
	\caption{The quantitative results by different methods for traffic data completion. The {\bf best} and \underline{second-best} values are highlighted. (RMSE $\downarrow$ and MAPE $\downarrow$)\label{tab_traffic}}\vspace{0.2cm}
	\begin{center}
		\tiny
		\setlength{\tabcolsep}{1pt}
		\begin{spacing}{1.2}
			\begin{tabular}{clcccccccccc}
				\toprule
				\multicolumn{2}{c}{Missing rate}&\multicolumn{2}{c}{0.1}&\multicolumn{2}{c}{0.15}&\multicolumn{2}{c}{0.2}&\multicolumn{2}{c}{0.25}&\multicolumn{2}{c}{0.3}\\
				\cmidrule{1-12}
				Data&Method&RMSE&MAPE&RMSE&MAPE&RMSE&MAPE&RMSE&MAPE&RMSE&MAPE\\
				\midrule
				\multirow{7}{*}{
					\textit{Guangzhou}}
				&HaLRTC & 40.09 & 1.000 & 40.97 & 1.000 & 39.70 & 1.000 & 39.55 & 1.000 & 39.23 & 1.000\\
				&LRTC-TNN&40.09 & 1.000 & 40.97 & 1.000 & 39.70 & 1.000 & 39.55 & 1.000 & 39.23 & 1.000\\
				&LATC&40.09 & 1.000 & 40.97 & 1.000 & 39.70 & 1.000 & 39.55 & 1.000 & 39.23 & 1.000\\

				&FSA & 10.25 & 0.241 & 9.71 & 0.227 & 9.48 & 0.220 & \underline{9.45} & 0.220 & 10.2111 & 0.237 \\
				&LRTFR & \underline{9.55} & \underline{0.215} & \underline{9.44} & \underline{0.213} & \underline{9.31} & \underline{0.215} & 9.69 & \underline{0.217} & \underline{9.8618} & \underline{0.225} \\
				&NeuApprox & \textbf{9.30} & \textbf{0.209} & \textbf{9.32} & \textbf{0.210} & \textbf{9.24}& \textbf{0.214} & \textbf{9.44} & \textbf{0.213} & \textbf{9.4125} & \textbf{0.208} \\
				\midrule
				\multirow{6}{*}{
					\textit{PEMS07}}
				&HaLRTC & 60.27 & 1.000 & 60.63 & 1.000 & 60.44 & 1.000 & 60.48 & \underline{1.000} & 60.53 & 1.000\\
				&LRTC-TNN&60.27 & 1.000 & 60.63 & 1.000 & 60.44 & 1.000 & 60.48 & \underline{1.000} & 60.53 & 1.000\\
				&LATC & 60.27 & 1.000 & 60.6392 & 1.000 & 60.44 & 1.000 & 60.48 & \underline{1.000} & 60.53 & 1.000\\

				&FSA & 11.56 & 0.152 & 11.56 & \underline{0.153} & 11.57 & \underline{0.152} & 11.56 & \textbf{0.155} & 11.6456 & 0.154\\
				&LRTFR & \underline{11.36} & \underline{0.150} & \underline{11.25} & \textbf{0.152} & \underline{11.25} & \textbf{0.151} & \underline{11.39} & \textbf{0.155} & \underline{11.0258} & \underline{0.148}\\
				&NeuApprox & \textbf{11.28} & \textbf{0.149} & \textbf{11.21} & \textbf{0.152} & \textbf{11.19} & \textbf{0.151} & \textbf{11.32} & \textbf{0.155} & \textbf{10.9685} & \textbf{0.142} \\
				\midrule
				\multirow{6}{*}{\textit{Seattle}}
				&HaLRTC & 58.02 & 1.000 & 57.96 & 1.000 & 57.92 & 1.000 & 58.03 & 1.000& 58.12 & 1.000 \\
				&LRTC-TNN & 58.02 & 1.000 & 57.96 & 1.000 & 57.92 & 1.000 & 58.03 & 1.000 & 58.12 & 1.000\\
				&LATC & 58.02 & 1.000 & 57.96 & 1.000 & 57.92 & 1.000 & 58.03 & 1.000 & 58.12 & 1.000\\

				&FSA & 8.01 & \underline{0.108} & 7.58 & \underline{0.101} & 7.63 & 0.104 & 7.51 & \textbf{0.098} & 7.4621 & 0.108 \\
				&LRTFR & \underline{6.62} & 0.121 & \underline{6.74} & 0.127 & \underline{6.83} & \underline{0.103} & \underline{6.84} & \underline{0.100} & \underline{6.9346} & \underline{0.100}\\
				&NeuApprox & \textbf{6.48} & \textbf{0.082} & \textbf{6.62} & \textbf{0.094} & \textbf{6.68} & \textbf{0.101} & \textbf{6.70} & \textbf{0.098} & \textbf{6.8787} & \textbf{0.097}\\
				
				\bottomrule
			\end{tabular}
		\end{spacing}
	\end{center}
\end{table*}
\begin{figure*}[h]
	\tiny
	\setlength{\tabcolsep}{2pt}
	\begin{center}
		\begin{tabular}{ccccccc} 
			HaLRTF & LRTC-TNN & LATC & FSA&LRTFR &NeuApprox & Original\\[2pt]
			\includegraphics[width=0.14\textwidth]{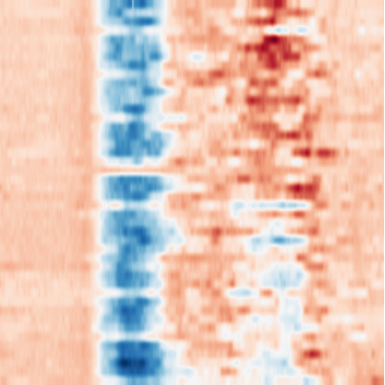}&
			\includegraphics[width=0.14\textwidth]{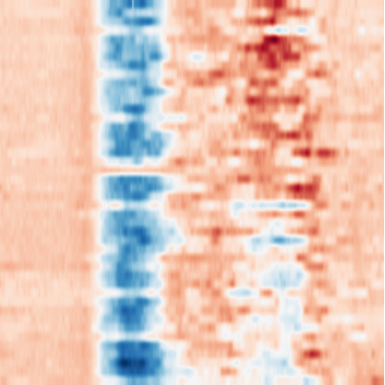}&
			\includegraphics[width=0.14\textwidth]{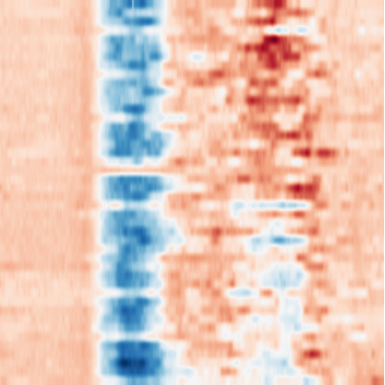}&
			\includegraphics[width=0.14\textwidth]{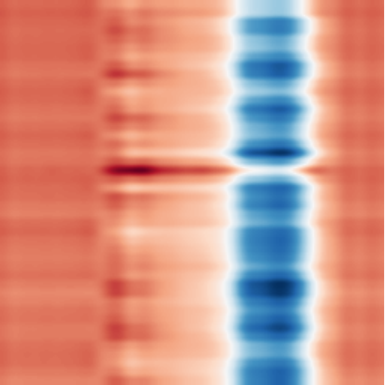}&
			\includegraphics[width=0.14\textwidth]{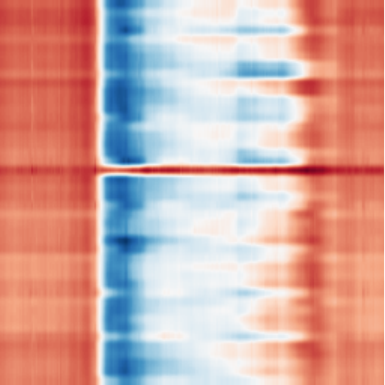}&
			\includegraphics[width=0.14\textwidth]{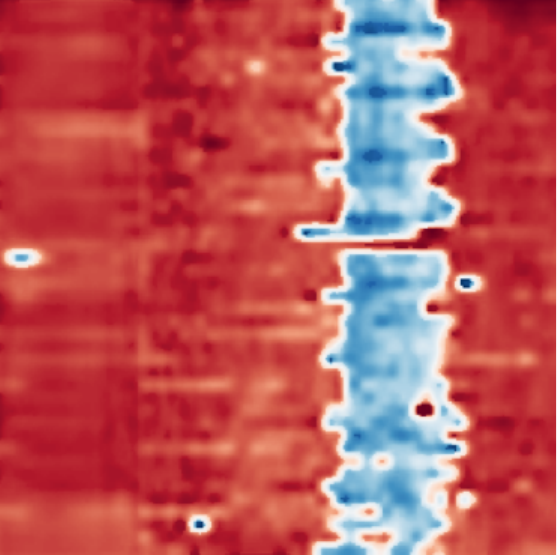}&
			\includegraphics[width=0.14\textwidth]{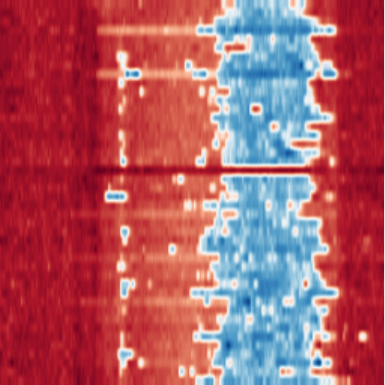}\\ [2pt]

			\includegraphics[width=0.14\textwidth]{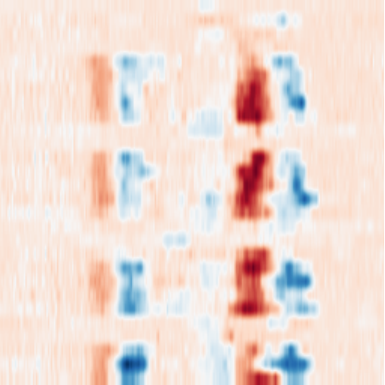}&
			\includegraphics[width=0.14\textwidth]{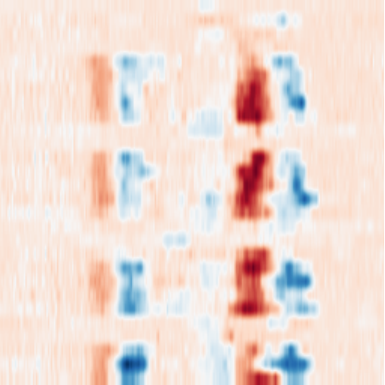}&
			\includegraphics[width=0.14\textwidth]{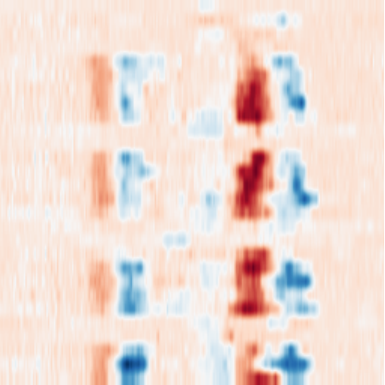}&
			\includegraphics[width=0.14\textwidth]{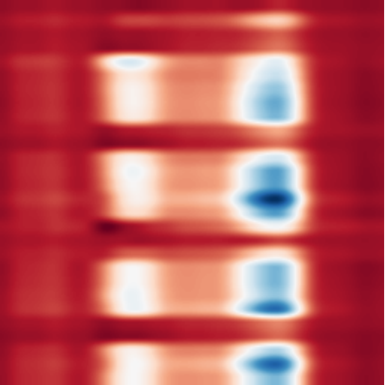}&
			\includegraphics[width=0.14\textwidth]{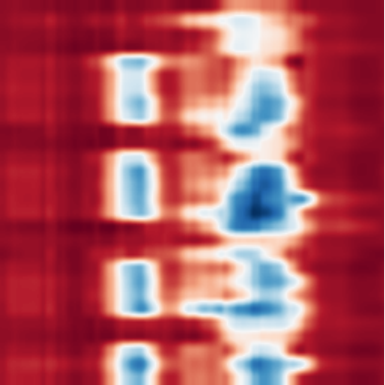}&
			\includegraphics[width=0.14\textwidth]{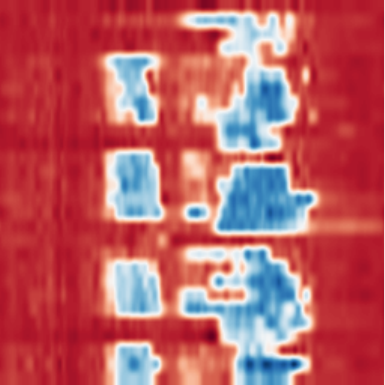}&
			\includegraphics[width=0.14\textwidth]{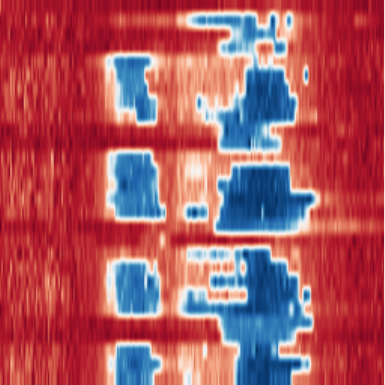}\\ [2pt]

			\includegraphics[width=0.14\textwidth]{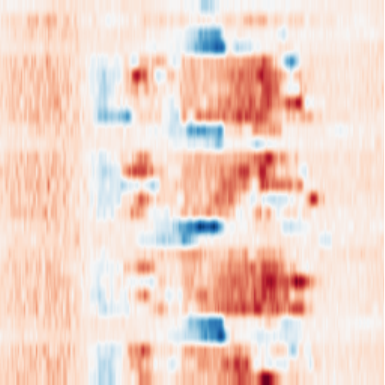}&
			\includegraphics[width=0.14\textwidth]{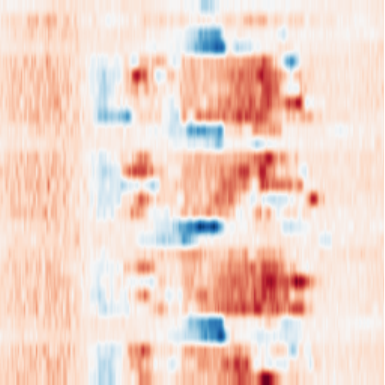}&
			\includegraphics[width=0.14\textwidth]{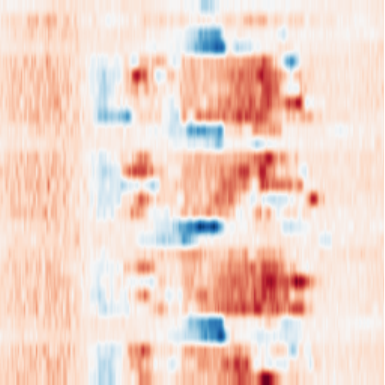}&
			\includegraphics[width=0.14\textwidth]{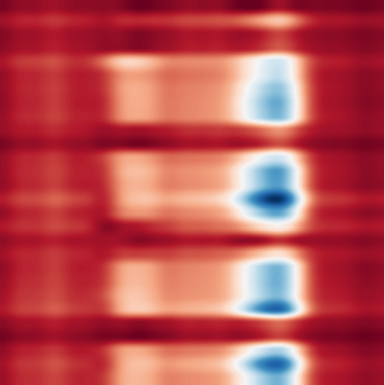}&
			\includegraphics[width=0.14\textwidth]{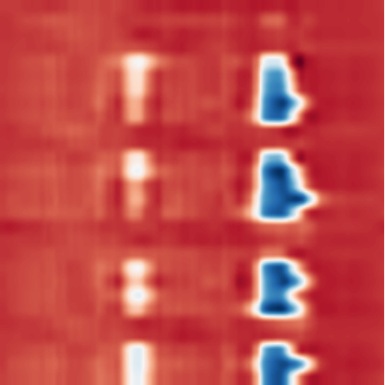}&
			\includegraphics[width=0.14\textwidth]{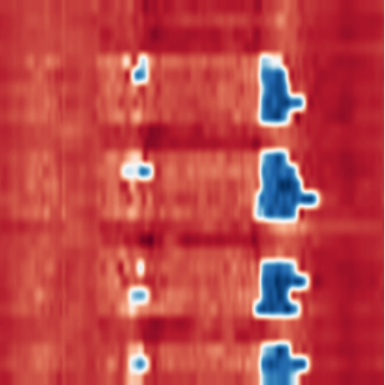}&
			\includegraphics[width=0.14\textwidth]{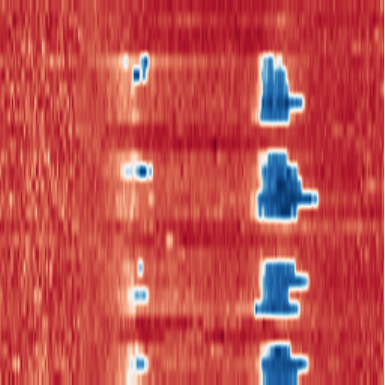}\\ [2pt]
			
		\end{tabular}
		
		\begin{tabular}{c}
			\includegraphics[width=0.99\textwidth]{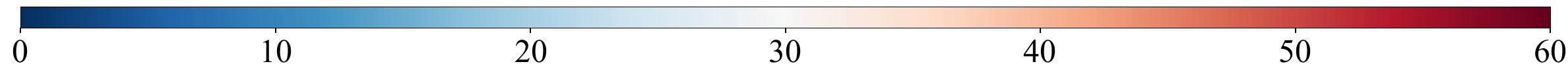}\\
		\end{tabular}
	\end{center}
	
	\caption{The results of traffic data completion by different methods. From top to bottom: the 58-th slice of \textit{PEMS07}, 40-th slice of \textit{Seattle}, and 148-th slice of \textit{Seattle}.\label{fig__traffic_completion_2}}
\end{figure*}

\begin{table*}[h]
	\caption{The quantitative results by different methods for point cloud data completion. The {\bf best} and \underline{second-best} values are highlighted. (NRMSE $\downarrow$ and SR-Square $\downarrow$)\label{tab_pcd}}\vspace{0.2cm}
	\begin{center}
		\tiny
		\setlength{\tabcolsep}{1pt}
		\begin{spacing}{1.2}
			\begin{tabular}{ccccccccccccc}
				\toprule
				Data & \multicolumn{2}{c}{\textit{Doll}} & \multicolumn{2}{c}{\textit{Duck}} & \multicolumn{2}{c}{\textit{Frog}} & \multicolumn{2}{c}{\textit{Mario}} & \multicolumn{2}{c}{\textit{Rabbit}} & \multicolumn{2}{c}{Average} \\
				\midrule
				Methods & NRMSE & R-Square & NRMSE & R-Square & NRMSE & R-Square & NRMSE & R-Square & NRMSE & R-Square & NRMSE & R-Square  \\
				\midrule
				DT & 0.2032 & 0.9074 &0.1793 & 0.6991 &0.1510 & 0.9140 &0.2658 & 0.7573 &0.2145 & 0.5996  & 0.2028 & 0.7747 \\
				KNR & \underline{0.1426} & 0.9544 &0.1355 & 0.8282 &0.1287 & 0.9375 &0.2112 & 0.8469 &\underline{0.1621} & 0.7715  & 0.1561 & 0.8677 \\
				RF & 0.1698 & 0.9353 &0.1375 & 0.8231 &\underline{0.1252} & 0.9199 &\underline{0.2022} & \underline{0.8596} &0.1649 & 0.7635 & 0.1651 & 0.7461 \\

				FSA & 0.2061 & 0.9048 &0.2085 & 0.5932 &0.2291 & 0.8021 &0.3296 & 0.6270 &0.2689 & 0.3710 & 0.2485 & 0.7849 \\
				LRTFR & 0.1433 & \underline{0.9588} & \underline{0.1224} & \underline{0.8387} & \underline{0.1252} & \underline{0.9214} & 0.2139 & 0.8515 & 0.1636 & \underline{0.7729}  & \underline{0.1536} & \underline{0.8687} \\
				NeuApprox & \textbf{0.1373} & \textbf{0.9632} & \textbf{0.1185} & \textbf{0.8449} & \textbf{0.1186} & \textbf{0.9302} & \textbf{0.1985} & \textbf{0.8621} & \textbf{0.1577} & \textbf{0.7809} & \textbf{0.1462} & \textbf{0.8763} \\
				\bottomrule
			\end{tabular}
		\end{spacing}
	\end{center}
\end{table*}
\begin{figure*}[h]
	\tiny
	\setlength{\tabcolsep}{2pt}
	\begin{center}
		\begin{tabular}{cccccccc} 
			\includegraphics[width=0.12\textwidth]{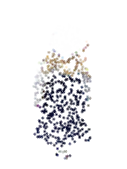}&
			\includegraphics[width=0.12\textwidth]{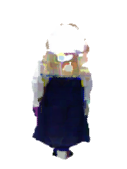}&
			\includegraphics[width=0.12\textwidth]{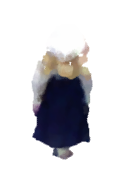}&
			\includegraphics[width=0.12\textwidth]{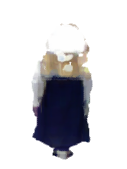}&
			\includegraphics[width=0.12\textwidth]{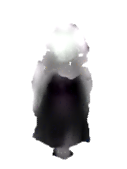}&
			\includegraphics[width=0.12\textwidth]{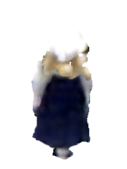}&
			\includegraphics[width=0.12\textwidth]{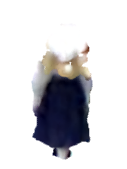}&
			\includegraphics[width=0.12\textwidth]{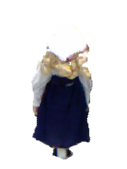}\\[2pt]

			\includegraphics[width=0.12\textwidth]{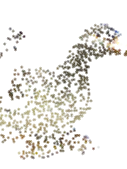}&
			\includegraphics[width=0.12\textwidth]{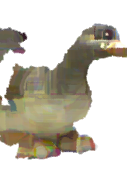}&
			\includegraphics[width=0.12\textwidth]{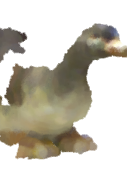}&
			\includegraphics[width=0.12\textwidth]{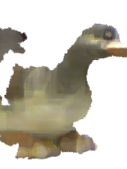}&
			\includegraphics[width=0.12\textwidth]{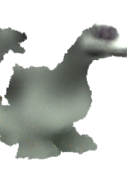}&
			\includegraphics[width=0.12\textwidth]{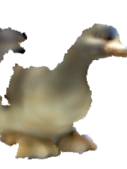}&
			\includegraphics[width=0.12\textwidth]{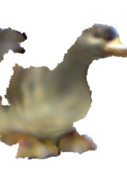}&
			\includegraphics[width=0.12\textwidth]{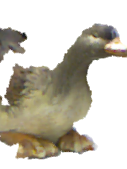}\\ [2pt]

			\includegraphics[width=0.12\textwidth]{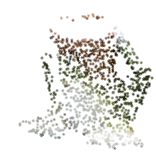}&
			\includegraphics[width=0.12\textwidth]{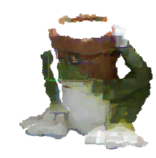}&
			\includegraphics[width=0.12\textwidth]{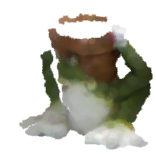}&
			\includegraphics[width=0.12\textwidth]{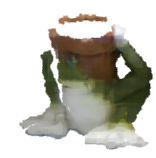}&
			\includegraphics[width=0.12\textwidth]{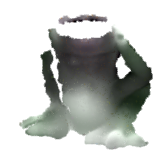}&
			\includegraphics[width=0.12\textwidth]{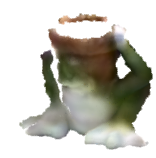}&
			\includegraphics[width=0.12\textwidth]{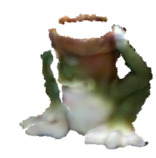}&
			\includegraphics[width=0.12\textwidth]{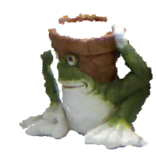}\\[2pt]
			
			\includegraphics[width=0.12\textwidth]{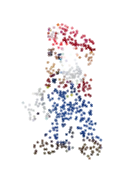}&
			\includegraphics[width=0.12\textwidth]{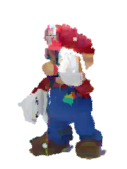}&
			\includegraphics[width=0.12\textwidth]{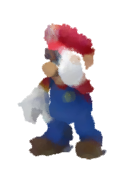}&
			\includegraphics[width=0.12\textwidth]{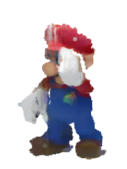}&
			\includegraphics[width=0.12\textwidth]{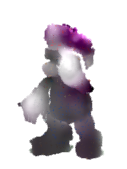}&
			\includegraphics[width=0.12\textwidth]{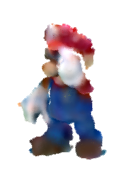}&
			\includegraphics[width=0.12\textwidth]{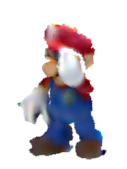}&
			\includegraphics[width=0.12\textwidth]{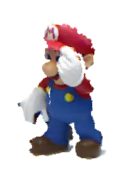}\\ [2pt]
			
			\includegraphics[width=0.12\textwidth]{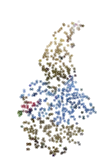}&
			\includegraphics[width=0.12\textwidth]{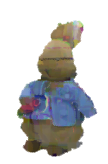}&
			\includegraphics[width=0.12\textwidth]{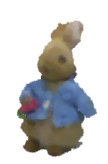}&
			\includegraphics[width=0.12\textwidth]{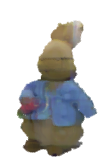}&
			\includegraphics[width=0.12\textwidth]{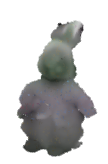}&
			\includegraphics[width=0.12\textwidth]{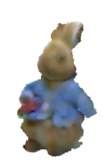}&
			\includegraphics[width=0.12\textwidth]{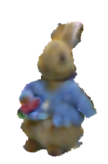}&
			\includegraphics[width=0.12\textwidth]{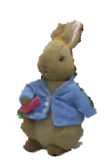}\\[2pt]
			Observed & DT & KNR & RF & FSA&LRTFR &NeuApprox & Original\\
		\end{tabular}
	\end{center}
	
	\caption{The results of point cloud data completion by different methods. From top to bottom: the visual results of \textit{Doll}, \textit{Duck}, \textit{Frog}, \textit{Mario}, and \textit{Rabbit} by different methods.\label{fig__pcd_completion_2}}
\end{figure*}

\subsection{Experimental Settings}
\subsubsection{Evaluation Metrics} 
To quantitatively evaluate the performance of different methods for multi-dimensional image inpainting, we employ three commonly used evaluation metrics: Peak Signal-to-Noise Ratio (PSNR)~\cite{HuynhThu2008}, Structural Similarity Index (SSIM)~\cite{1284395}, and Normalized Root Mean Square Error (NRMSE). PSNR is defined as: 

\begin{equation}
	\text{PSNR} = 10 \cdot \log_{10} \left( \frac{P^2}{\text{MSE}} \right),
\end{equation}
where \( P \) is the maximum possible pixel value of the image, and MSE is the Mean Squared Error between the original and inpainted images. Higher PSNR values generally indicate better image quality. 

SSIM is defined as:

\begin{equation}
	\text{SSIM}(x, y) = \frac{(2\mu_x \mu_y + C_1)(2\sigma_{xy} + C_2)}{(\mu_x^2 + \mu_y^2 + C_1)(\sigma_x^2 + \sigma_y^2 + C_2)},
\end{equation}
where \( \mu_x \) and \( \mu_y \) are the mean values of the images \( x \) and \( y \), \( \sigma_x \) and \( \sigma_y \) are the variances, and \( \sigma_{xy} \) is the covariance between \( x \) and \( y \). SSIM values range from -1 to 1, with higher values indicating greater structural similarity between the images.

NRMSE is calculated as:

\begin{equation}
	\text{NRMSE} = \frac{\sqrt{\sum_{i=1}^{n} (y_i - \hat{y}_i)^2}}{\sqrt{\sum_{i=1}^{n} y_i^2}},
\end{equation}
where \( y_i \) represents the original image pixels, and \( \hat{y}_i \) represents the inpainted image pixels. In general, higher PSNR and SSIM, and lower NRMSE values indicate better results.

In this paper, two quantitative evaluation metrics are used for evaluating the traffic data completion performance of different methods, i.e., the mean absolute percentage error (MAPE) and the root mean square error (RMSE), which characterize different aspects of the quality and accuracy of the imputation results. The definitions of MAPE and RMSE are as follows:

\begin{equation}
	\text{MAPE}(\mathcal{X}, \mathcal{X^*}) = \frac{1}{N(\Omega^C)} \sum_{i \in \Omega^C} \left| \frac{x_i - x_i^*}{x_i} \right| \times 100
\end{equation}

\noindent and

\begin{equation}
	\text{RMSE}(\mathcal{X}, \mathcal{X^*})  = \sqrt{ \frac{1}{N(\Omega^C)} \sum_{i \in \Omega^C} \left( x_i - x_i^* \right)^2 },
\end{equation}
where $\Omega^C$ is the complementary set of the observed set $\Omega$, $N(\Omega^C)$ denotes the number of elements in $\Omega^C$, and $x_i$ and $x_i^*$ are the actual and imputed values, respectively. Lower MAPE and RMSE values indicate better imputation performance.

For point cloud data completion, we select two commonly used evaluation metrics, i.e., Root Mean Square Error (RMSE) and \( R^2 \) (coefficient of determination), to assess the performance of the inpainting methods. RMSE, which measures the average magnitude of error between the predicted and original data, is defined as:

\begin{equation}
	\text{RMSE}(x, y) = \frac{\sqrt{\frac{1}{n} \sum_i (x(i) - y(i))^2}}{\max_i\{x(i)\} - \min_i\{x(i)\}},
\end{equation}
where \( x(i) \) represents the original image pixels and \( y(i) \) represents the inpainted image pixels. The RMSE is normalized to ensure independence from image size and dynamic range. A smaller RMSE value indicates better performance.

The \( R^2 \) metric measures the goodness of fit between the predicted and original data, and is defined as:

\begin{equation}
	R^2(x, y) = 1 - \frac{\sum_i (x(i) - y(i))^2}{\sum_i \left(x(i) - \frac{1}{n} \sum_i x(i)\right)^2},
\end{equation}
where \( x(i) \) is the original image pixel and \( y(i) \) is the inpainted image pixel. \( R^2 \) ranges from 0 to 1. A larger \( R^2 \) value suggests that the model's predictions better explain the variability in the original data.

\subsubsection{Hyperparameters Settings}

For all experiments, the depth of each neural basis function is selected from the set $\{2, 3, 4, 5\}$, and the width of each layer of the neural basis function is chosen from the set $\{64, 128, 192, 256\}$. The number of terms, denoted as \(T\), is selected from the set \(\{2, 3, 4, 5\}\). The size of each core tensor \(\mathcal{C}^{(j)}\) is identical, denoted as \((R_1, R_2, R_3)\), and are chosen from the set \(\{(\lfloor D_1/s_1 \rfloor, \lfloor D_2/s_2 \rfloor, \lfloor D_3/s_3 \rfloor) \mid s_1, s_2, s_3 \in \{1, 2, 4, 8, 16\}\}\), where \(D_i\) (for \(i = 1, 2, 3\)) denotes the sizes of the observed data. The learning rate (\(lr\)) is uniformly set to \(10^{-4}\) across all tasks. Furthermore, the weight decay parameter for the Adam optimizer is selected from the set $\{0, 0.001, 0.01, 0.1, 1\}$.

\subsection{Multi-Dimensional Image Inpainting Results}

To evaluate the effectiveness of our proposed method, we consider a typical data recovery problem, e.g., multi-dimensional image inpainting. We compare our method with three traditional discrete low-rank tensor decomposition-based methods (i.e., TRLRF \cite{TRLRF}, TNN \cite{TNN}, and FCTN \cite{FCTN}) and two low-rank tensor function-based methods (i.e., FSA \cite{fsa_tsp} and LRTFR \cite{lrtfr_tpami}). The testing data include MSIs in the CAVE dataset \cite{CAVE}, videos, and light field data \cite{10.1007/978-3-319-54187-7_2}. The random SRs are set to be 0.1, 0.15, 0.2, 0.25, and 0.3. The parameters of each method are manually adjusted to achieve its best performance.\par

Table \ref{tab_completion} reveals the inpainting comparison regarding PSNR, SSIM, and NRMSE on different data with different SRs. We can observe that the discrete low-rank tensor decomposition-based methods TRLRF, TNN, and FCTN achieve lower PSNR and SSIM values which indicates unsatisfactory inpainting performance, mainly due to their limited representation ability. Two low-rank tensor function-based methods FSA and LRTFR achieves higher PSNR and SSIM values than those of discrete low-rank tensor decomposition-based methods. The proposed NeuApprox obtains the best PSNR, SSIM, and RNMSE values in almost all inpainting cases. The superior inpainting performance of the proposed NeuApprox can be attributed to powerful representation ability of neural basis functions and structured decomposition.

To further visually compare the performance of all competing methods, we display the visual results of multi-dimensional image inpainting by different methods with SR $=0.15$ in Figs. \ref{fig_completion_image}. We can clearly observe that the proposed NeuApprox is markedly promising in preserving the local fine details and recovering visual qualities. For example, most results recovered by compared methods remain serious artifacts and blurry regions that are reflected in the enlarged areas. In contrast, the proposed achieves promising performance, especially in its ability to recover local fine details. These observations empirically illustrate the superiority of the proposed NeuApprox compared with traditional discrete low-rank tensor decomposition-based methods and low-rank tensor function-based methods.

\subsection{Traffic Data Completion}

To further test the effectiveness of our method, we apply the proposed method on traffic data completion. In our benchmark experiments, we use the three benchmark traffic datasets (i.e., \textit{Guangzhou} dataset, \textit{PEMS07} dataset \cite{Song_Lin_Guo_Wan_2020}, and \textit{Seattle} dataset \cite{8917706}) to evaluate all methods. We compare our method with three traditional traffic data completion methods (i.e., HaLRTC \cite{6138863}, LRTC-TNN \cite{10147850}, and LATC \cite{9548664}) and three low-rank tensor function-based traffic data completion methods (i.e., FSA \cite{fsa_tsp} and LRTFR \cite{lrtfr_tpami}). Due to communication network issues, restricted power supply conditions, or scheduled maintenance, the traffic data usually suffer from missing slice problem. We generate five different cases with randomly missing slices at rates of $0.1$, $0.15$, $0.2$, $0.25$, and $0.3$.

Table \ref{tab_traffic} presents the MAPE and RMSE values of the recovered traffic data using different methods. We can observe that HaLRTC, LRTC-TNN, and LATC can not handle the challenging slice missing scenarios. This is because these methods only leverage the low-rankness of traffic data, which is also inherent in slice-missing traffic data. FSA and LRTFR can impute slice-missing traffic data because they can simultaneously characterize low-rankness and local smoothness of traffic data. The suggested NeuApprox excels in the challenging slice missing cases since the suggested NeuApprox enjoys powerful approximation ability to capture underlying traffic data. The appealing performance on traffic data completion also reveals the strong data-adaptation ability of the proposed NeuApprox.

To intuitively show the advantages of the proposed method, we display the $57$-th of the imputed Guangzhou, PEM07, and Seattle by different methods with missing rate $0.1$, as depicted in Figure \ref{fig__traffic_completion_2}. From the visual comparison, we can observe that HaLRTC, LRTC-TNN, and LATC are worse than the other methods since these methods can not handle challenging slice missing scenarios. Although FSA and LRTFR can recover some missing traffic data, they perform unsactisfactory in some patterns. Contrastingly, it is evident that the traffic data imputed by the proposed NeuApprox is closest to the actual traffic data, with the local patterns of traffic data being comparatively well-preserved, comparing with other competing methods.

\subsection{Point Cloud Data Completion}

We further evaluate our method on off-meshgrid point cloud data processing. Specifically, we consider the point cloud data recovery task, which aims to estimate the color information of the given point cloud. The original point cloud data with $n$ points is represented by an $n\times 6$ matrix $\mathbf{P} \in \mathbb{R}^{n \times 6}$. Each row of $\mathbf{P}$ is an $(x, y, z, r, g, b)$-formed six-dimensional vector, where $(x,y,z)$ is the coordinate and $(r,g,b)$ is the color information. We split the point cloud data into training and testing datasets. The training dataset contains $n'$ pairs of $(x, y, z)$ and (R,G,B) information, where the model takes $(x, y, z)$ as input and is expected to output the color information (R,G,B). The testing dataset contains $n - n'$ pairs of $(x, y, z)$ and (R,G,B) data to test the trained model. We use five color point cloud datasets in the SHOT website, named \textit{Doll}, \textit{Duck}, \textit{Frog}, \textit{Mario}, and \textit{Rabbit}. We consider the training/testing data split ratio as 1/19 for all data and report the prediction results.

\begin{table*}[h]
	\caption{The average quantitative results of the proposed NeuApprox with different neural basis functions on MSI {\it Toy} and {\it Painting} with different SRs.}\label{tab:ablation}
	\begin{center}
		\tiny
		\setlength{\tabcolsep}{1pt}
		\begin{spacing}{1.2}
			\begin{tabular}{cccccccccccccccc}
				\toprule
				Sampling rate&\multicolumn{3}{c}{0.1}&\multicolumn{3}{c}{0.15}&\multicolumn{3}{c}{0.2}&\multicolumn{3}{c}{0.25}&\multicolumn{3}{c}{0.3}\\
				\midrule
				Basis function &PSNR &SSIM &NRMSE &PSNR &SSIM &NRMSE &PSNR &SSIM &NRMSE &PSNR &SSIM&NRMSE&PSNR &SSIM&NRMSE\\
				\midrule
				Poly & 22.15 & 0.510 & 0.078& 22.29 & 0.520 & 0.077& 23.87 & 0.520 & 0.064& 25.86 & 0.610 & 0.051& 29.99 & 0.875 & 0.032\\
				Fourier & 26.27 & 0.696 & 0.049& 27.77 & 0.769 & 0.041& 28.43 & 0.804 & 0.038& 29.66 & 0.825 & 0.033& 30.30 & 0.827 & 0.031\\
				Gaussian& 30.07 & 0.827 & 0.031& 30.43 & 0.836 & 0.030& 31.30 & 0.867 & 0.027& 31.66 & 0.876 & 0.026&     32.22 & 0.892 & 0.024\\
				Neural& \textbf{33.28} & \textbf{0.864} & \textbf{0.022}& \textbf{35.82} & \textbf{0.919} & \textbf{0.016}& \textbf{37.45} & \textbf{0.942} & \textbf{0.013}& \textbf{38.66} & \textbf{0.949} & \textbf{0.012}& \textbf{39.49} & \textbf{0.968} & \textbf{0.011}\\
				\bottomrule
			\end{tabular}
		\end{spacing}
	\end{center}
\end{table*}
\begin{table*}[h]
	\caption{Inpainting results of different training strategies on MSI {\it Toys} and {\it Painting} with different SRs. The neural basis functions are pretrained on out-of-distribution datasets.}\label{tab:generalization}
	\begin{center}
		\tiny
		\setlength{\tabcolsep}{1pt}
		\begin{spacing}{1.2}
			\begin{tabular}{ccccccccccccccccc}
				\toprule
				\multirow{2}{*}{Data}&
				Sampling rate & \multicolumn{3}{c}{0.1} & \multicolumn{3}{c}{0.15} &
				\multicolumn{3}{c}{0.2}&\multicolumn{3}{c}{0.25}&\multicolumn{3}{c}{0.3} \\\cmidrule{2-17}
				&Method &PSNR &SSIM & Time &PSNR &SSIM & Time  &PSNR &SSIM & Time &PSNR &SSIM & Time &PSNR &SSIM & Time\\\midrule
				\multirow{3}{*}{Painting}&Pretraining & 29.25 & 0.759 & \textbf{2.35}& 30.97 & 0.813 & \textbf{3.67}& 33.52 & 0.891 & \textbf{2.98}& 34.46 & 0.910 & \textbf{3.75} & 35.78 & 0.930 & \textbf{3.48}\\
				&Fine-Tuning & \underline{30.92} & \underline{0.802} & \underline{4.07}& \underline{32.65} & \underline{0.855} & \underline{7.15}& \underline{35.19} & \underline{0.915} & \underline{8.23}& \underline{36.69} & \underline{0.939} & \underline{9.74} & \textbf{37.88 }& \underline{0.952} & \underline{7.14}\\
				&Training from Scratch & \textbf{31.32} & \textbf{0.839} & 28.55& \textbf{33.96} & \textbf{0.903} & 30.62& \textbf{35.35} & \textbf{0.928} & 31.72 & \textbf{36.83} & \textbf{0.946} & 41.09 & \underline{37.78} & \textbf{0.980} & 40.53\\
				\midrule
				\multirow{3}{*}{Toys}&Pretraining & 29.59 & 0.758 & \textbf{2.14} & 30.83 & 0.799 & \textbf{2.55}& 33.32 & 0.889 & \textbf{2.77}& 34.82 & 0.913 & \textbf{3.49} & 35.70 & 0.932 & \textbf{3.74}\\
				&Fine-Tuning & \underline{32.04} & \underline{0.820} & \underline{4.18} & \underline{35.19} & \underline{0.905} & \underline{6.47} & \textbf{38.05} & \textbf{0.953} & \underline{7.62} & \textbf{39.50} & \textbf{0.967} & \underline{9.56} & \textbf{40.44} & \underline{0.974} & \underline{10.53}\\
				&Training from Scratch & \textbf{32.58} & \textbf{0.871} & 26.33& \textbf{35.50} & \textbf{0.922} & 27.40& \underline{37.37} & \underline{0.944} & 31.35& \underline{38.66} & \underline{0.959} & 40.76 & \underline{39.76} & \textbf{0.984} & 40.05 \\
				\bottomrule
			\end{tabular}
		\end{spacing}
	\end{center}
\end{table*}

\begin{figure*}[h]
	\scriptsize
	\setlength{\tabcolsep}{2pt}
	\begin{center}
		\begin{tabular}{cccccccc}
			\includegraphics[width=0.19\textwidth]{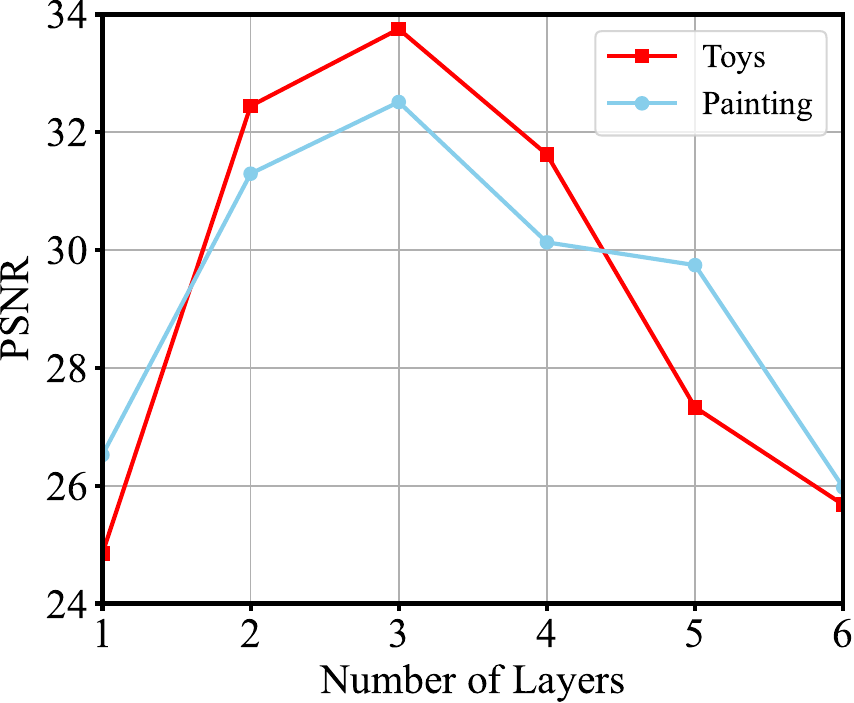}&
			\includegraphics[width=0.19\textwidth]{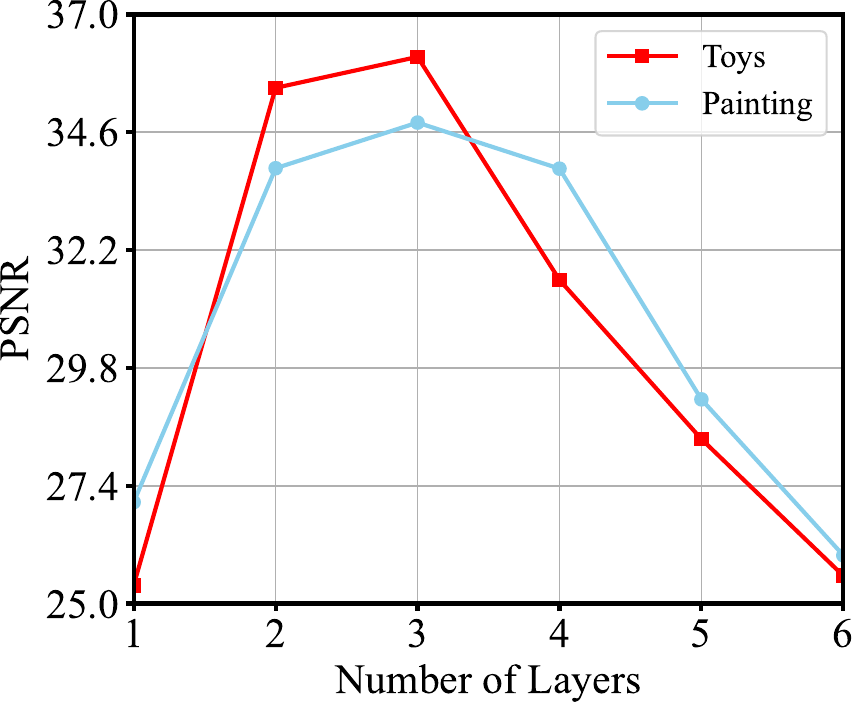}&
			\includegraphics[width=0.19\textwidth]{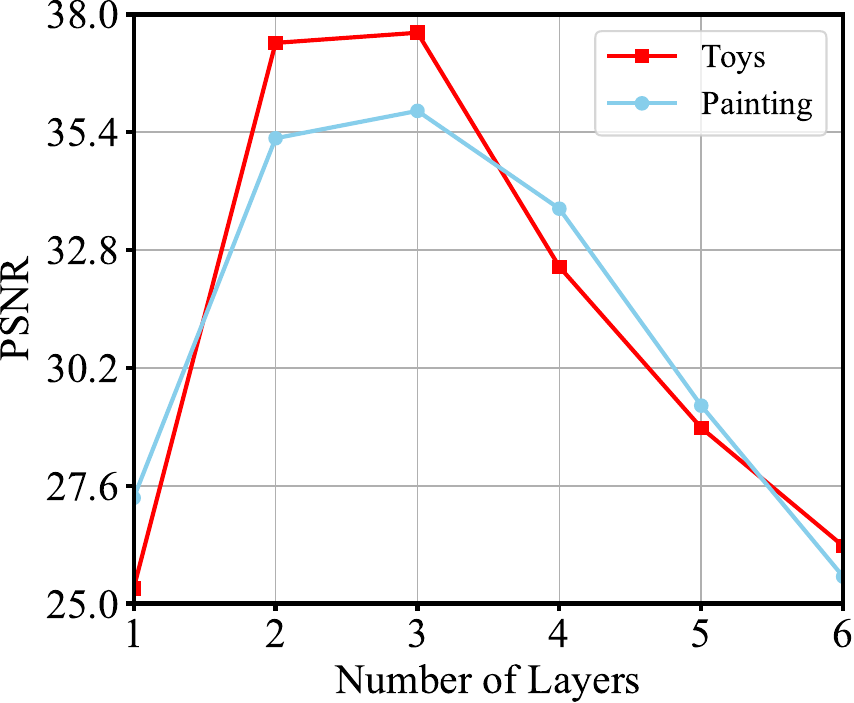}&
			\includegraphics[width=0.19\textwidth]{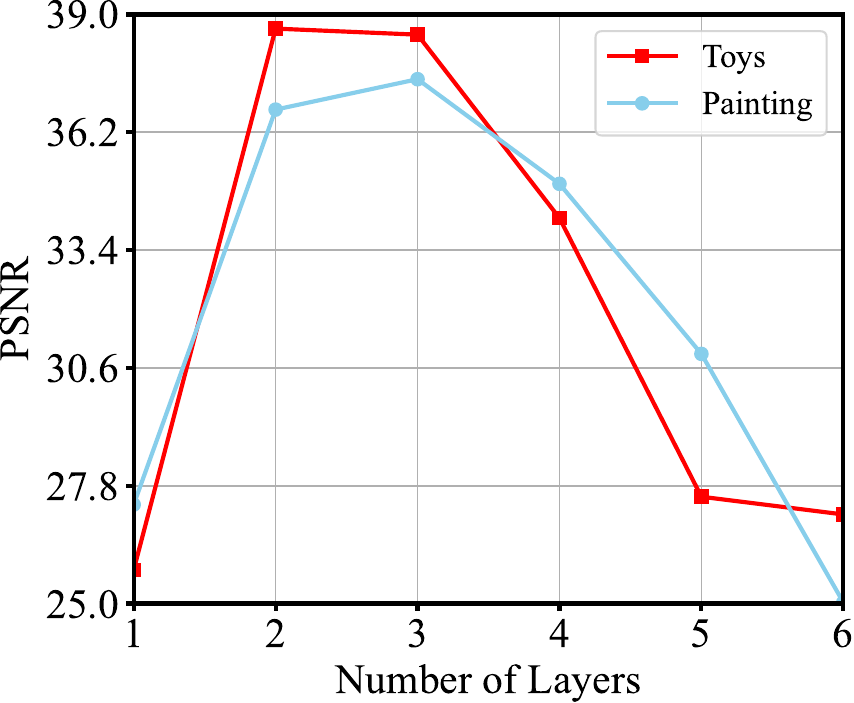}&
			\includegraphics[width=0.19\textwidth]{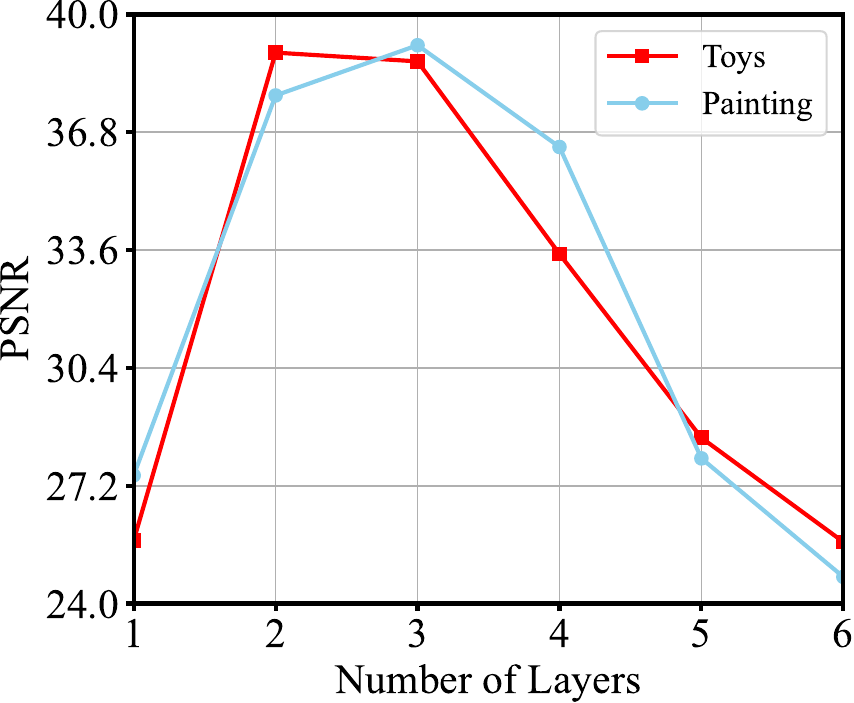}
		\end{tabular}
	\end{center}
	\caption{The PSNR curves of the recovered MSI {\it Toy} and {\it Painting} by NeuApprox with different numbers of layers. From left to right: The PSNR curves with respect to number of layers of neural networks used in basis functions on MSI {\it Toy} and {\it Painting} with SR=0.1, 0.15, 0.2, 0.25, and 0.3.  \label{fig_layers}}
\end{figure*}

\begin{figure*}[h]
	\tiny
	\setlength{\tabcolsep}{1pt}
	\centering
	\begin{tabular}{ccccc}
		\includegraphics[width=0.19\textwidth]{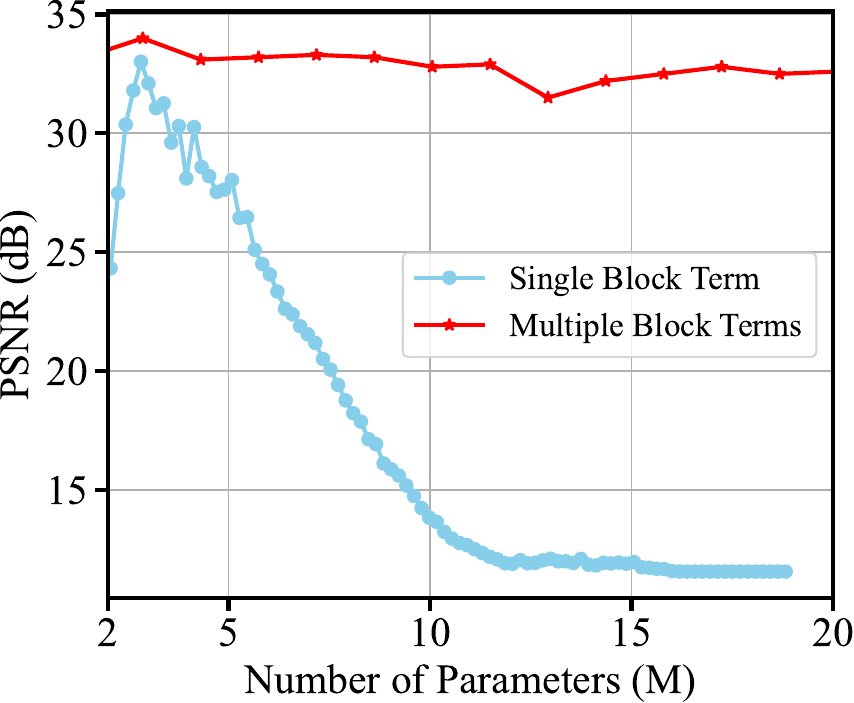}&
		\includegraphics[width=0.19\textwidth]{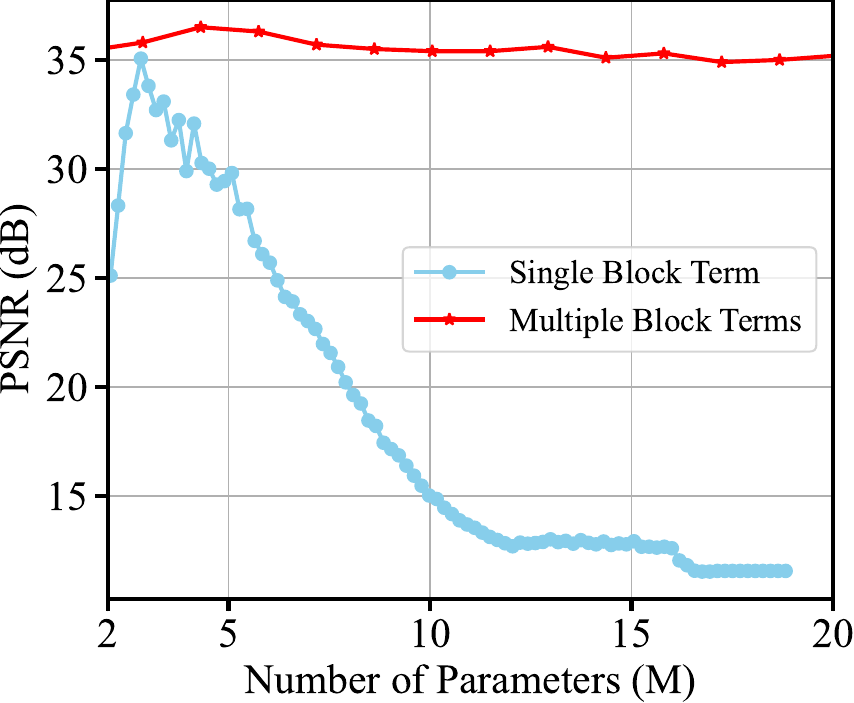}&
		\includegraphics[width=0.19\textwidth]{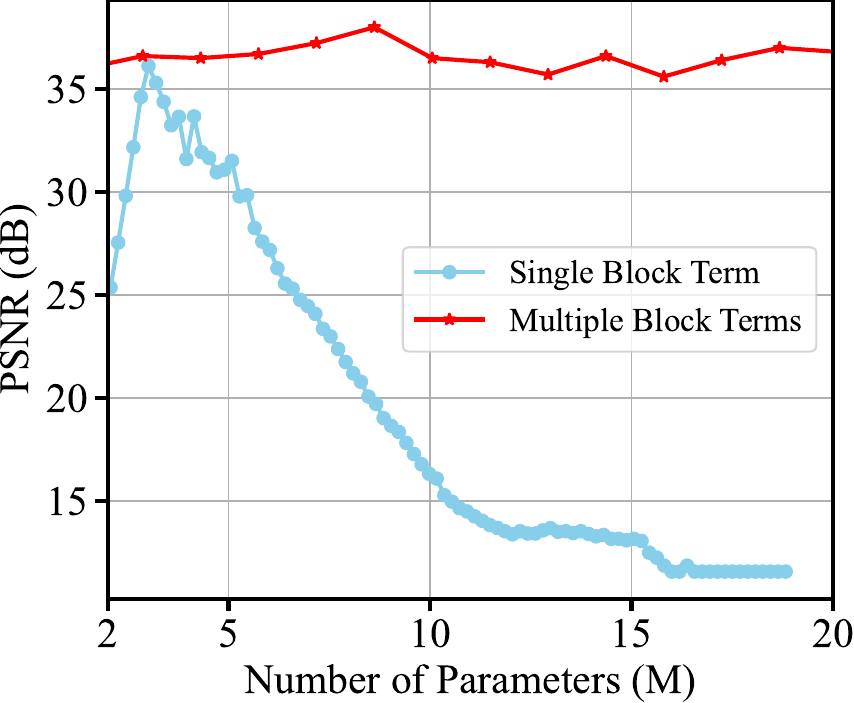}&
		\includegraphics[width=0.19\textwidth]{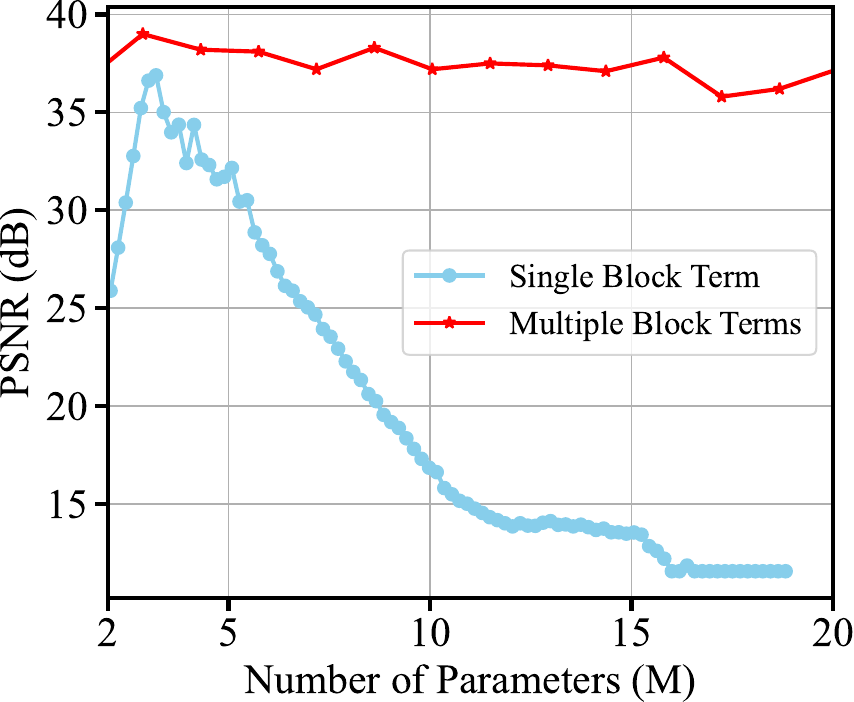}&
		\includegraphics[width=0.19\textwidth]{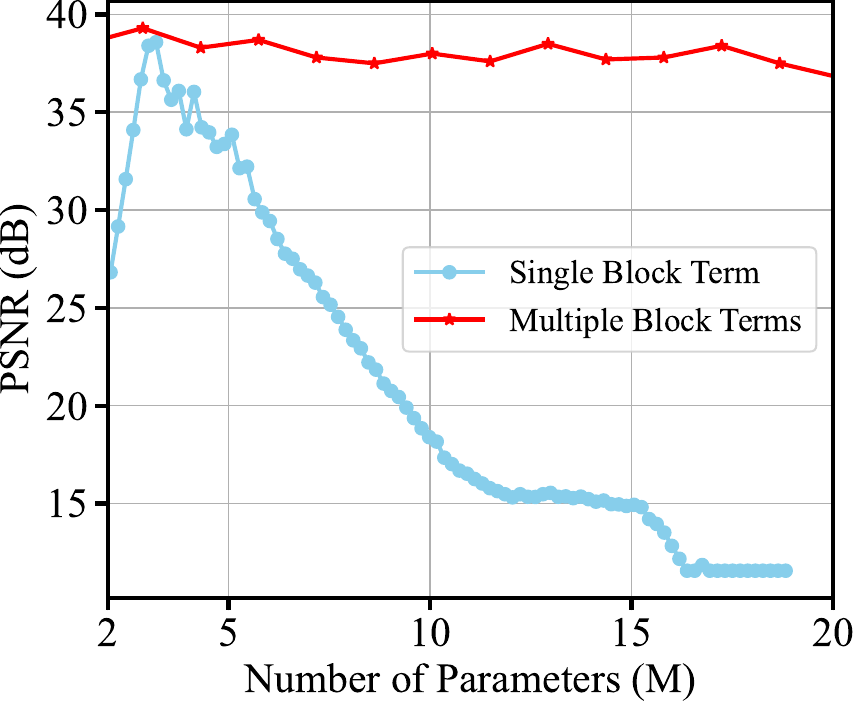}\\[2pt]
	\end{tabular}
	\caption{The PSNR values of the recovered by the proposed NeuApprox with single block term representation and multiple block terms on MSI \textit{Toy} with respect to the number of parameters. From left to right: The PSNR curves with respect to number of parameters in NeuApprox on MSI {\it Toy} and {\it Painting} with SR=0.1, 0.15, 0.2, 0.25, and 0.3. \label{oneterm}}
\end{figure*}

\begin{figure}[h]
	\scriptsize
	\setlength{\tabcolsep}{2pt}
	\begin{center}
		\begin{tabular}{ccc}
			(a) & (b) & (c)\\[2pt]
			\includegraphics[width=0.3\textwidth]{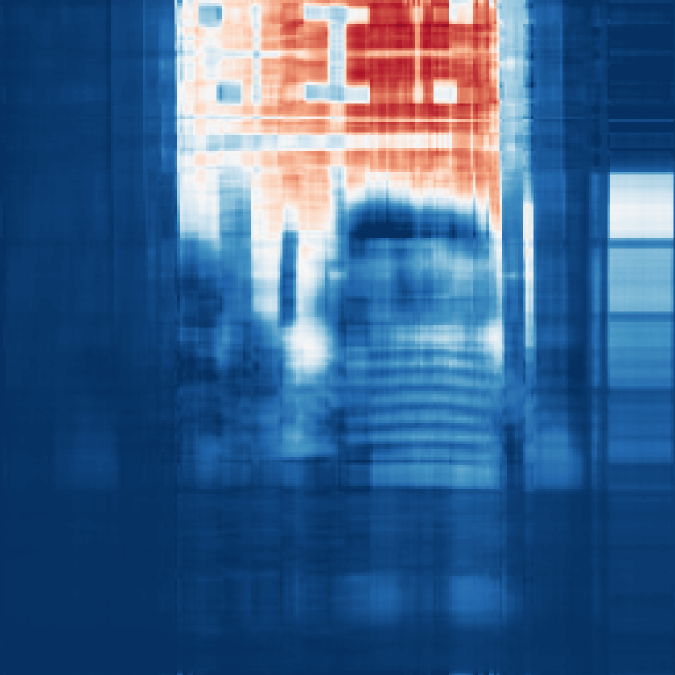}&
			\includegraphics[width=0.3\textwidth]{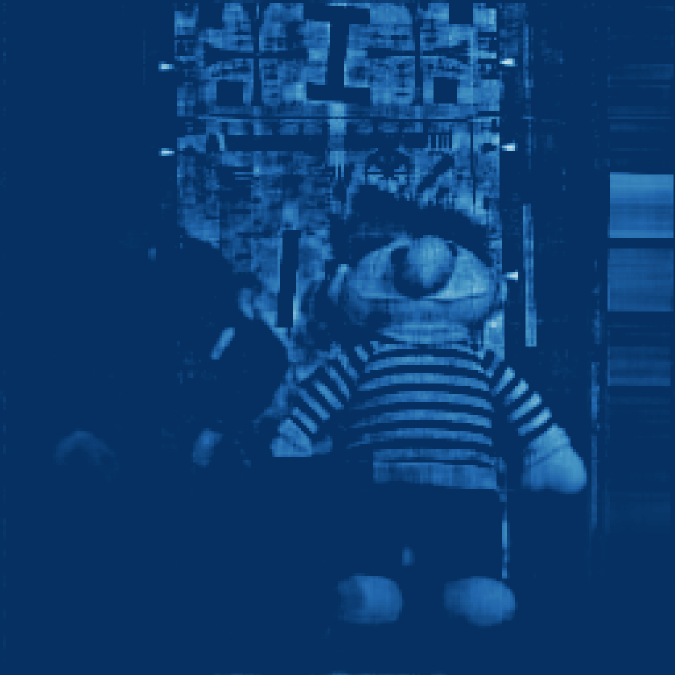}&
			\includegraphics[width=0.3\textwidth]{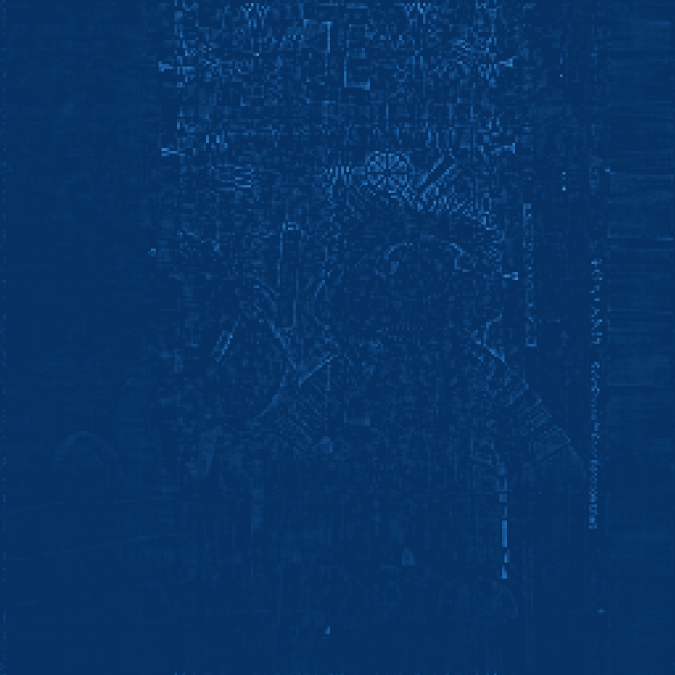}\\[2pt]
			\includegraphics[width=0.3\textwidth]{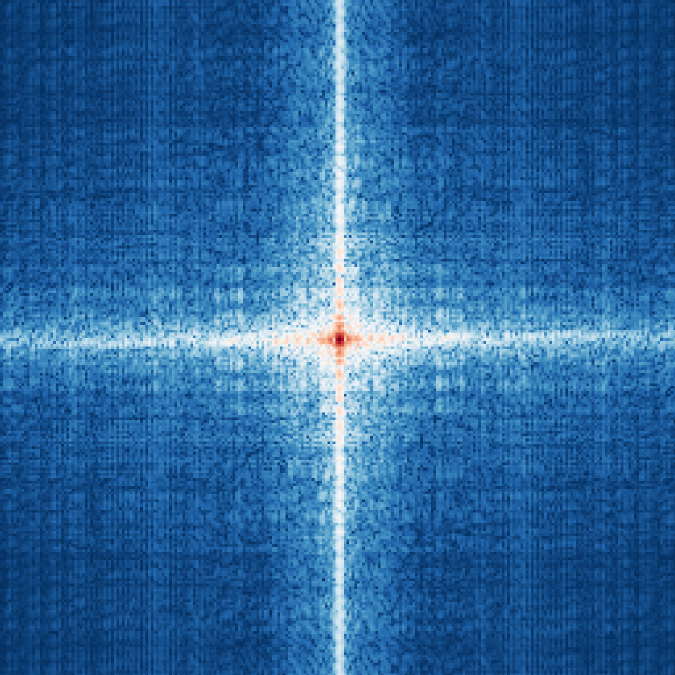}&
			\includegraphics[width=0.3\textwidth]{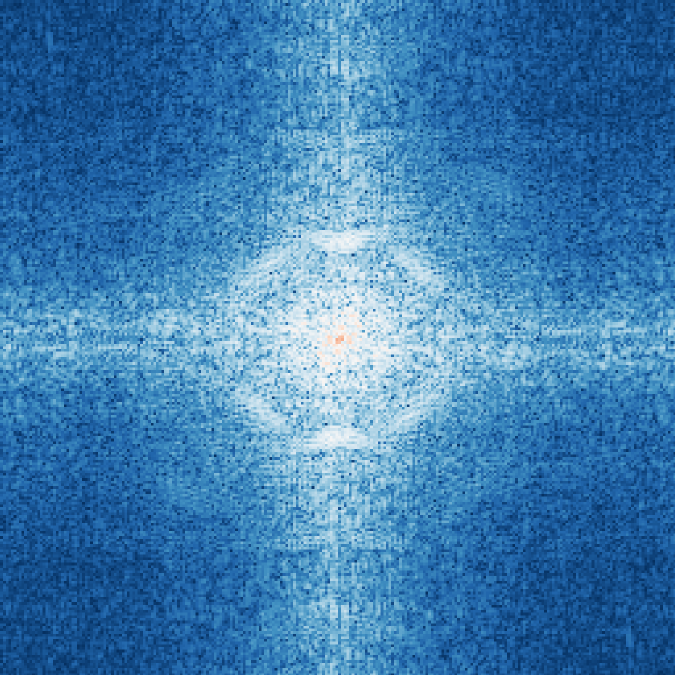}&
			\includegraphics[width=0.3\textwidth]{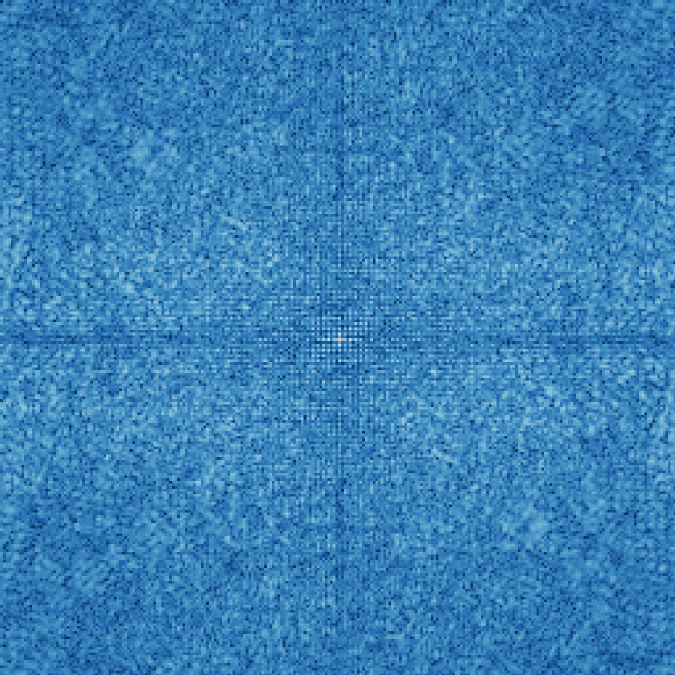}\\[2pt]
		\end{tabular}
		\begin{tabular}{c}
			\hspace{0.05cm}\includegraphics[width=0.9\textwidth]{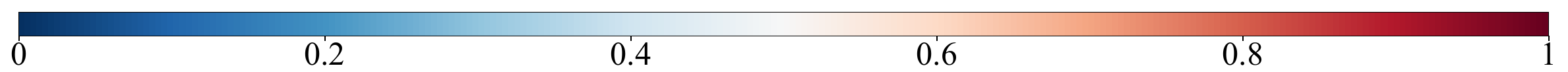}\\
		\end{tabular}
	\end{center}
	\caption{Different block terms (top) of the proposed NeuApprox and their Fourier spectrum images (bottom) on MSI \textit{Toy} with SR=0.1. (a)-(c) first, second, and third block terms of the proposed NeuApprox with corresponding Fourier spectrum images.\label{multiblock}}
\end{figure}

The quantitative and qualitative results of the point cloud data completion are shown in Table \ref{tab_pcd} and Fig. \ref{fig__pcd_completion_2}. We can observe that our NeuApprox achieves higher performance than other comparison methods in terms of NRMSE and R-Square. From the visual results, we can see that FSA and other regression methods have difficulty in accurately recovering intricate details. In contrast, the proposed NeuApprox can better represent the details due to its powerful approximation to capture the details of point cloud data and data-adaptation ability to different data, as evidenced by the patterns of the bottle in Fig. \ref{fig__pcd_completion_2}. These results confirm the effectiveness and superiority of our NeuApprox for complex point data completion tasks.

\section{Discussions}\label{sec:diss}

\subsection{Approximation Ability of NeuApprox}
One advantage of our proposed NeuApprox is the strong approximation ability. To understand the approximation ability of NeuApprox, we compare the performance of the proposed NeuApprox with different hand-crafted basis functions including polyomial basis function, Fourier basis function, and Gaussian basis function on MSI {\it Toys} and {\it Painting} with different SRs. Table \ref{tab:ablation} lists the numerical results by the proposed NeuApprox with different basis functions. From Table \ref{tab:ablation}, we can observe that the the proposed NeuApprox with neural basis function consistently outperforms the proposed NeuApprox with other hand-crafted basis functions. It can be attributed to the powerful approximation ability of neural basis function, which endows the proposed NeuApprox with powerful representation ability.

To further analyze the approximation ability of NeuApprox, we compare the performance of the proposed NeuApprox with different numbers of MLP layers, denoted as $L$. In the proposed NeuApprox, the number of neurons in each hidden layer is uniformly set to 256.  Fig. \ref{fig_layers} displays the PSNR cruves of the proposed NeuApprox with different numbers of layers. From Fig. \ref{fig_layers}, we can observe that the number of layers have a relatively obvious effect on the performance of our method. Notably, once the number of layers surpasses a certain threshold, the proposed NeuApprox exhibits overfitting, which consequently results in unsatisfactory performance. Thus, to achieve optimal performance, the proposed NeuApprox requires an appropriate number of layers.

\subsection{Data Adaptation Ability of NeuApprox}
One advantage of the proposed NeuApprox is its powerful data adaptation ability. To systematically investigate this capability, we conduct experiments under different scenarios. We explore three distinct strategies for obtaining neural basis functions: (1) the pretraining approach, where neural basis functions are trained on the training data (e.g., \textit{Flowers}) and fixed on out-of-distribution test data during testing; (2) the fine-tuning approach, where the parameters of pretrained neural basis functions are fine-tuned using LoRA~\cite{hu2022lora} with rank $10$ during testing; and (3) the training from scratch approach, where the parameters of neural basis functions are directly optimized on the test data. In all three strategies, the corresponding coefficients of the neural basis functions are directly optimized on the test data.

The experimental results are listed in Tab \ref{tab:generalization}. Compared with pretraining approach, the fine-tuning approach achieves significant improvement in terms of PSNR and SSIM values, as the neural basis functions on fine-tuning approach enjoys a good initialization. Compared with the training from scratch approach, the fine-tuning approach achieves comparable or even higher PSNR and SSIM values in some cases, while requiring significantly less training time. This observation demonstrates that the neural basis functions possess strong data adaptation capabilities for out-of-distribution data.  

\subsection{Contribution of Block Terms}\label{multiblockterms}

The innovative building blocks in the proposed NeuApprox are the block terms, each of which is the product of expressive neural basis functions and corresponding learnable coefficients. The block terms allow us to faithfully capture distinct components and feature clear physical interpretability. To evaluate the contribution of block terms, we compare the performance of NeuApprox with different numbers of block terms (i.e., different numbers of parameters). To further analyze the contribution of the block terms, we also compare the performance of single-block-term NeuApprox with different numbers of parameters, by adjusting the number of layers in neural basis functions. The curve of PSNR with respect to the number of parameters are displayed in Fig.\ref{oneterm}. We can observe that multiple block terms can achieve better results as compared to single block term even with increasing number of parameters, attributed to the ability of multiple block terms to capture distinct components.

\begin{figure*}[!h]
	\scriptsize
	\setlength{\tabcolsep}{2pt}
	\begin{center}
		\begin{tabular}{ccccc}
			\includegraphics[width=0.19\textwidth]{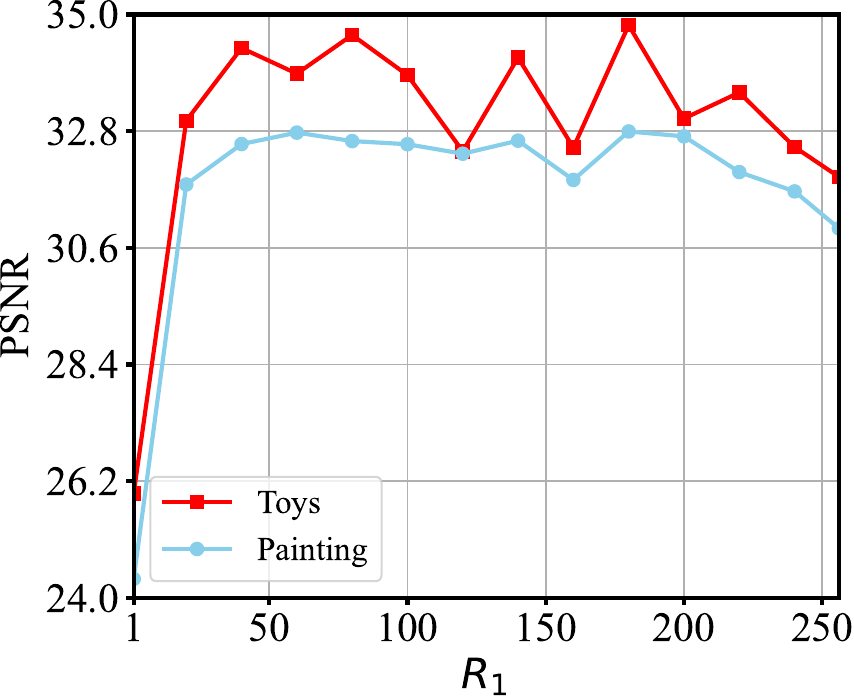}&
			\includegraphics[width=0.19\textwidth]{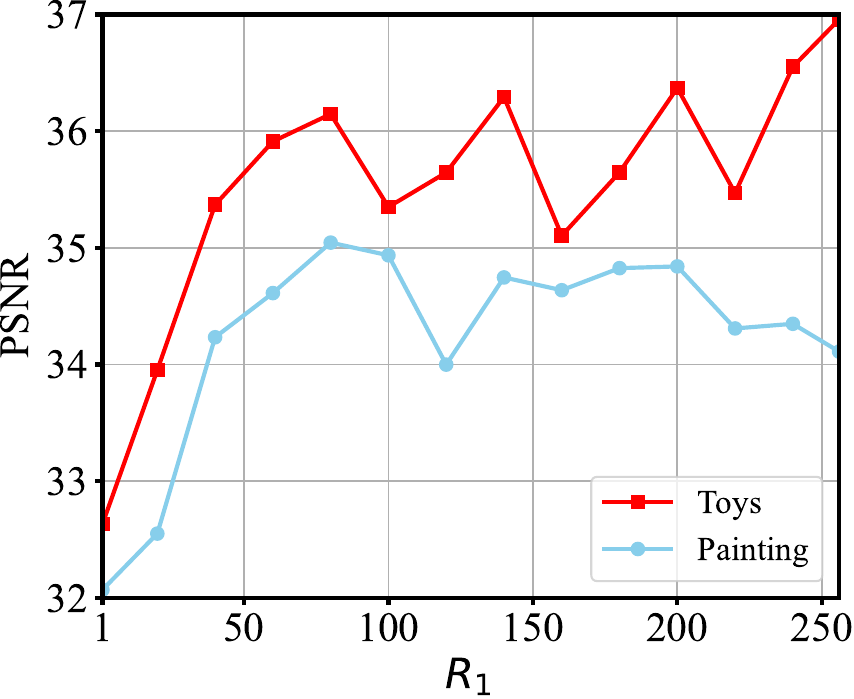}&
			\includegraphics[width=0.19\textwidth]{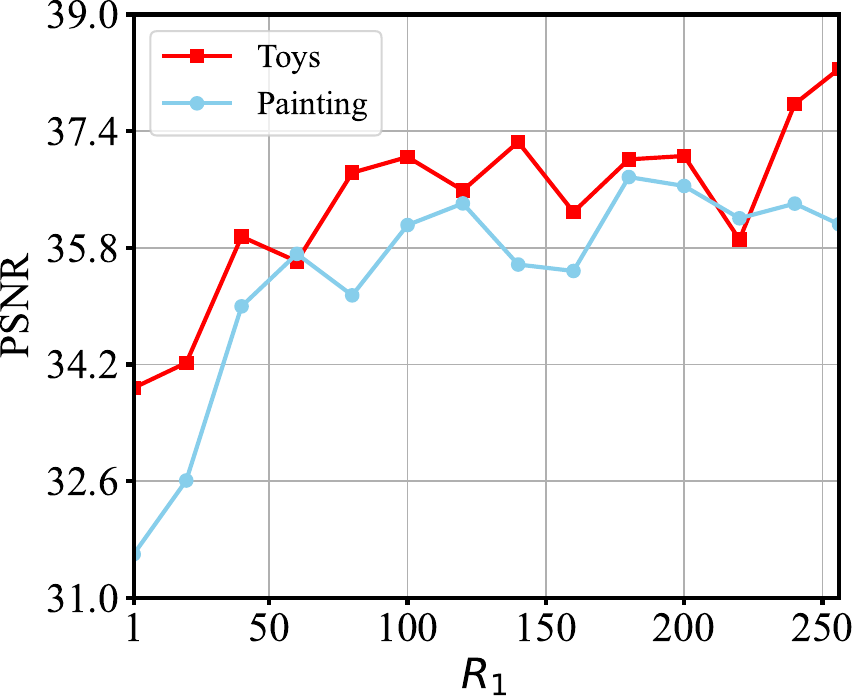}&
			\includegraphics[width=0.19\textwidth]{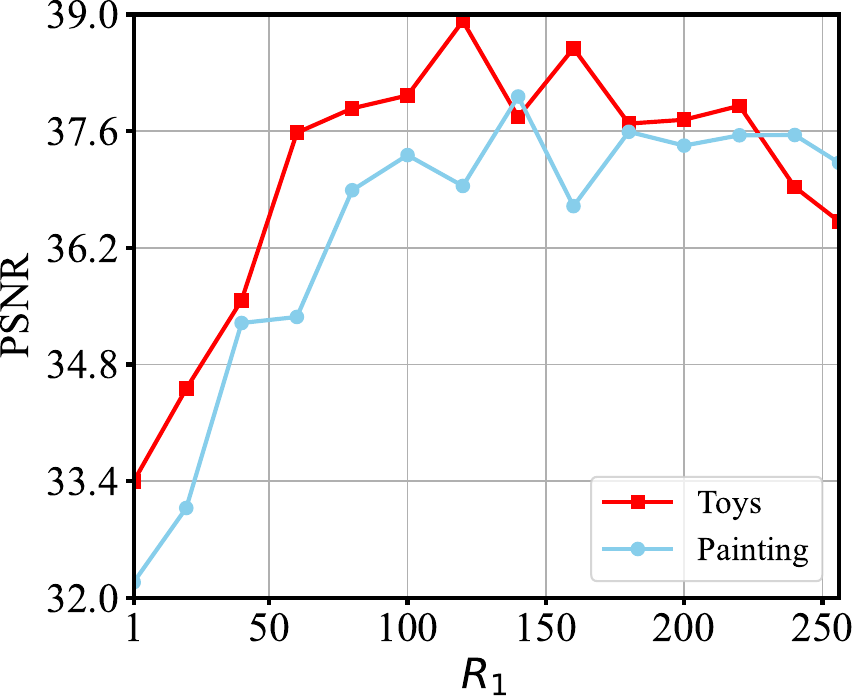}&
			\includegraphics[width=0.19\textwidth]{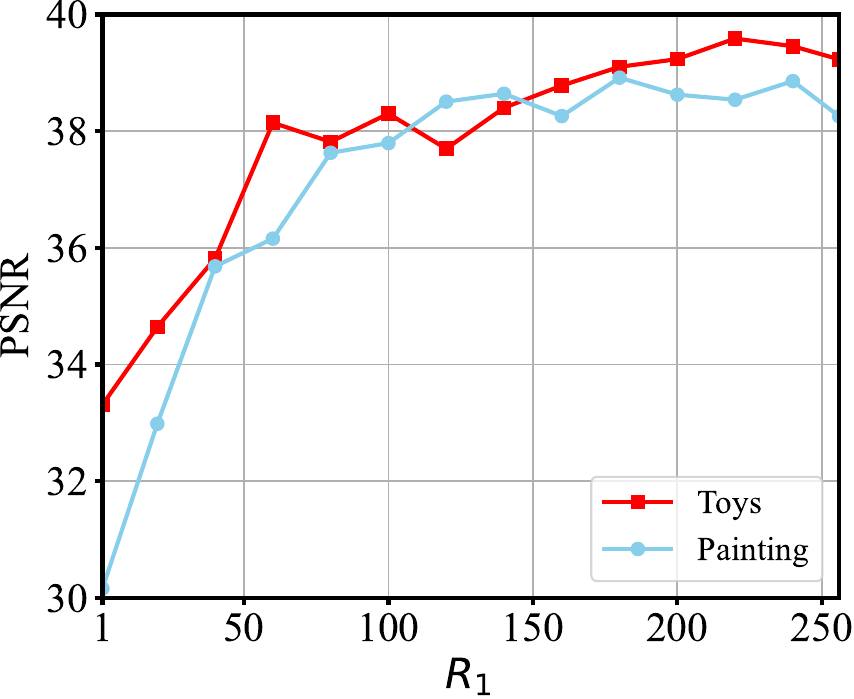}\\[2pt]
			\includegraphics[width=0.19\textwidth]{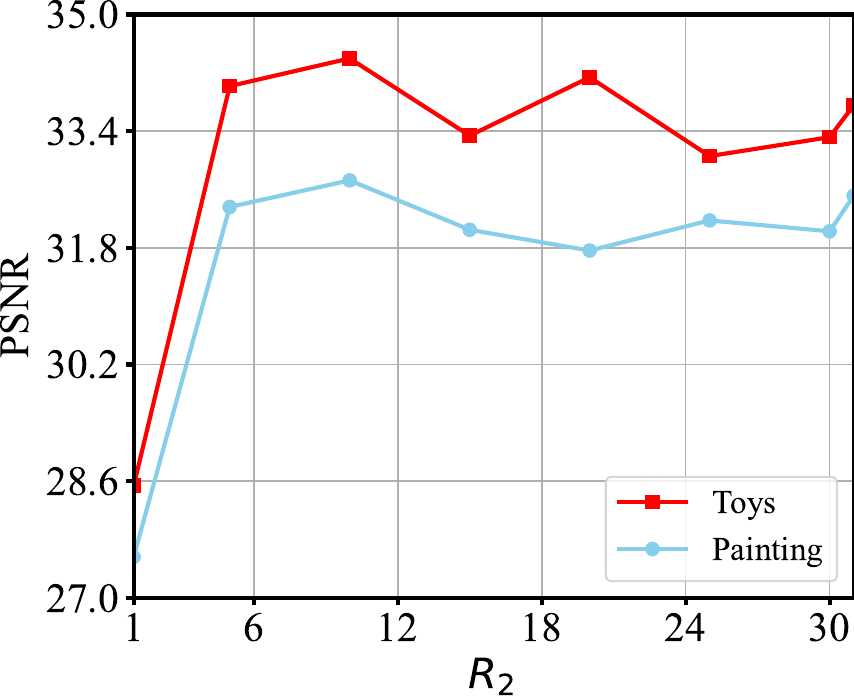}&
			\includegraphics[width=0.19\textwidth]{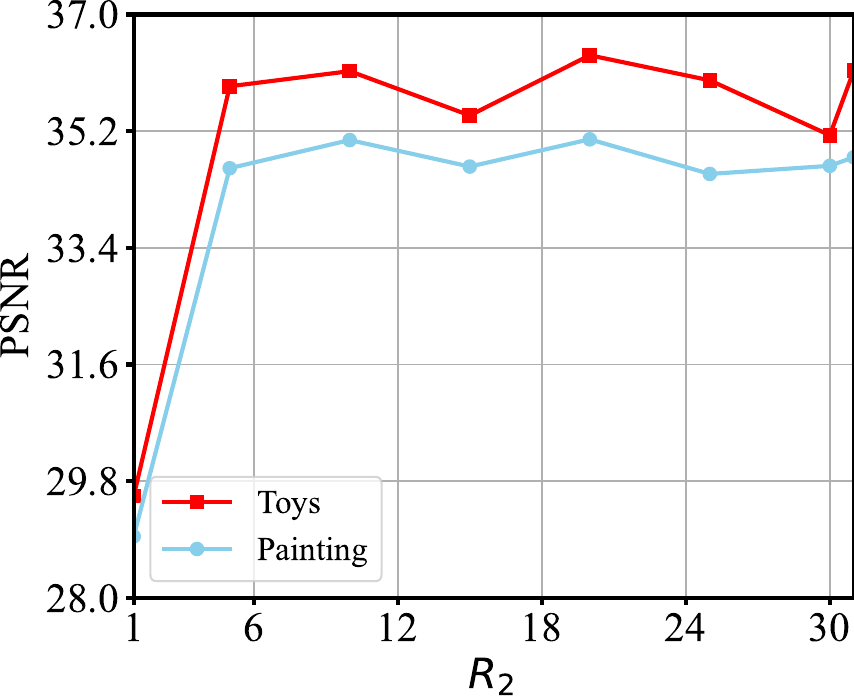}&
			\includegraphics[width=0.19\textwidth]{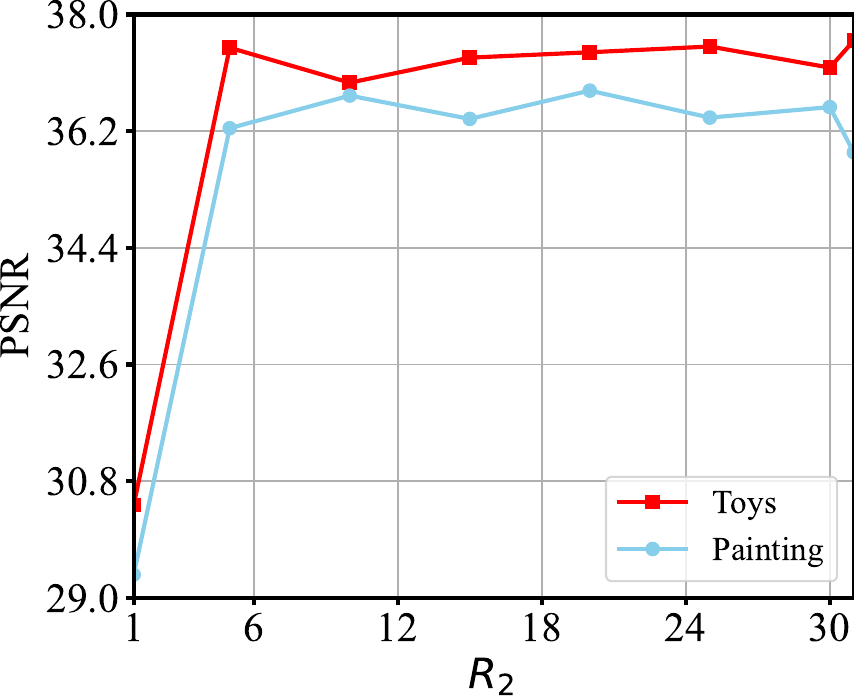}&
			\includegraphics[width=0.19\textwidth]{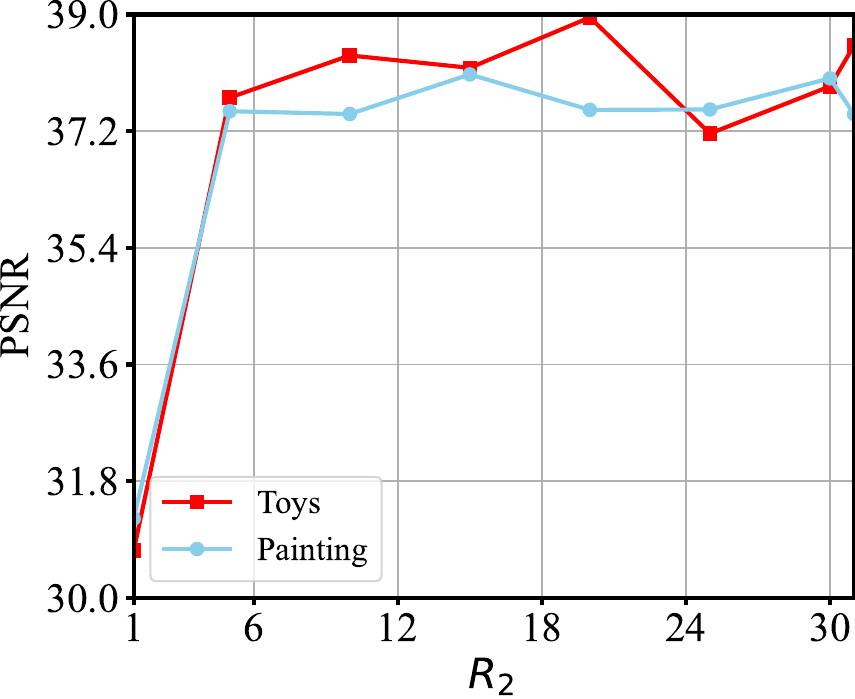}&
			\includegraphics[width=0.19\textwidth]{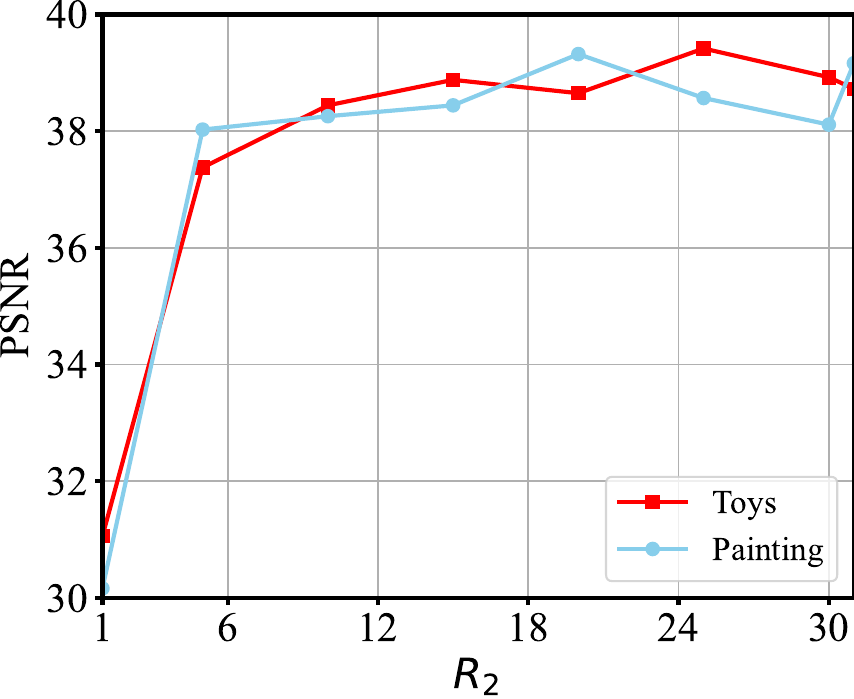}\\[2pt]
		\end{tabular}
	\end{center}
	\caption{The PSNR curves of the recovered MSI {\it Toy} and {\it Painting} by NeuApprox with different coefficient sizes. From top to bottom: The PSNR curves with respect to $R_1$ values and the PSNR curves with respect to $R_2$ values. From left to right: The PSNR curves of the recovered MSI {\it Toy} and {\it Painting} with SR=0.1, 0.15, 0.2, 0.25, and 0.3.  \label{fig_rank}}
\end{figure*}

Furthermore, we visualize each block term of the proposed NeuApprox, as well as their corresponding Fourier spectrum images, in Fig. \ref{multiblock}. As shown in the figure, each block term is specifically responsible for modeling different components of underlying data, thereby enabling NeuApprox to effectively capture both the global structures and the intricate fine details present in the data. This collaborative modeling strategy results in a more accurate and comprehensive reconstruction, clearly demonstrating the substantial advantage of employing multiple block terms for representing complex multi-dimensional data.

\subsection{Influence of Coefficients Tensor Size}

Since the coefficients tensor size (i.e., Tucker rank of each block term) is a key hyperparameter, we compare the performance of the proposed NeuApprox with different coefficients tensor size. For convenience, we set the coefficients tensor size as $(R_1, R_1, R_2)$, where $R_1$ and $R_2$ denote the spatial rank and the spectral rank, respectively. Fig.~\ref{fig_rank} reports the PSNR of the recovered MSI \textit{Toy} and \textit{Painting} with respect to the two key hyperparameters $R_1$ and $R_2$ under five SRs (i.e., SR = 0.1, 0.15, 0.2, 0.25, 0.3). From Fig. \ref{fig_rank}, we can observe that $R_1$ and $R_2$ have a significant effect on the performance of the proposed NeuApprox. To achieve the best performance, the proposed NeuApprox requires an appropriate coefficients tensor size.

\section{Conclusion}
\label{sec:conclusion}

In this work, we proposed NeuApprox, a neural basis function-based paradigm for multivariate function approximation, overcoming the limitations of classical hand-crafted basis function-based methods in approximation ability and data adaptation ability. NeuApprox decomposes the underlying multivariate function as a sum of interpretable block terms, each formed by the product of neural basis functions and corresponding learnable coefficients, allowing us to faithfully capture distinct components of the data and flexibly adapt to out-of-distribution data through fine-tuning. We theoretically proved that NeuApprox can approximate any continuous multivariate function to arbitrary accuracy. Extensive experiments on diverse multi-dimensional datasets including multispectral images, light field data, videos, traffic data, and point clouds demonstrated its superior performance in both approximation ability and data adaptation ability compared to classical hand-crafted basis function methods.

\section*{Data Availability}
The CAVE dataset is publicly available at \url{https://www.cs.columbia.edu/CAVE/databases/multispectral/}.\\
The videos are publicly available at \url{http://trace.eas.asu.edu/yuv/}.\\
The light field data are publicly available at \url{https://lightfield-analysis.uni-konstanz.de/}.\\
The Guangzhou dataset is publicly available at \url{https://zenodo.org/records/1205229}.\\
The PEMS07 dataset is publicly available at \url{https://pems.dot.ca.gov/?dnode=Clearinghouse&type=station_5min&district_id=7&submit=Submit}.\\
The Seattle dataset is publicly available at \url{https://github.com/zhiyongc/Seattle-Loop-Data}.\\
The color point cloud datasets are publicly available at \url{http://www.vision.deis.unibo.it/research/80-shot}.

\backmatter

\bibliographystyle{sn-mathphys-ay}
\bibliography{sn-bibliography}

\end{document}